\documentclass{article}

\usepackage[utf8]{inputenc}

\usepackage{hyperref}
\usepackage{amsmath, amsthm, amssymb}
\usepackage{float}
\usepackage{enumitem}

\usepackage{tikz} 
\usetikzlibrary{graphs}
\usetikzlibrary{arrows}
\usetikzlibrary{arrows,positioning,automata,calc}

\newtheorem{definition}{Definition}[section]
\newtheorem{lemma}[definition]{Lemma}
\newtheorem{corollary}[definition]{Corollary}
\newtheorem{theorem}[definition]{Theorem}
\newtheorem{example}[definition]{Example}

\newcommand{\pair}[1]{\left({#1}\right)}
\newcommand{\sqbra}[1]{\left[{#1}\right]}
\newcommand{\set}[1]{\left\{{#1}\right\}}
\newcommand{\ang}[1]{\left\langle{#1}\right\rangle}
\newcommand{\abs}[1]{\left\lvert{#1}\right\rvert}

\newcommand{\es}{\varnothing}
\newcommand{\pow}{\mathcal{P}}
\newcommand{\nat}{\mathbb{N}}
\newcommand{\integ}{\mathbb{Z}}

\numberwithin{equation}{section}
\numberwithin{figure}{section}
\numberwithin{table}{section}

\begin{document}

\title{Notes on Abstract Argumentation Theory}
\author{Anthony P. Young\footnote{Comments and suggestions to \href{mailto:peter.young@kcl.ac.uk}{peter.young@kcl.ac.uk}}}
\date{\today}
\maketitle

\begin{abstract}
\noindent This note reviews Section 2 of Dung's seminal 1995 paper on abstract argumentation theory. In particular, we clarify and make explicit all of the proofs mentioned therein, and provide more examples to illustrate the definitions, with the aim to help readers approaching abstract argumentation theory for the first time. However, we provide minimal commentary and will refer the reader to Dung's paper for the intuitions behind various concepts. The appropriate mathematical prerequisites are provided in the appendices.
\end{abstract}

\tableofcontents

\newpage

\listoffigures

\newpage

\listoftables

\newpage

\section{Introduction}

\textit{Abstract argumentation theory} \cite{Dung:95} is concerned with the formalisation and implementation of methods that resolve disagreements rationally, based on the pattern of disagreements alone. Such a need typically arises when reasoning with incomplete and contradictory information from multiple sources, whether human or machine. It provides a general approach to modelling conflict between arguments and the agents putting forward those arguments. This is based on the commonsensical idea that the ``winning'' arguments are those that are collectively consistent and adequately responds to all counterarguments. It is assumed that arguments that have no counterarguments will win \textit{by default}. Such ideas can be quite intuitive in its handling of conflict and justification \cite{Dung:95,CogSci:10}.

In this note, we review the mathematical background of abstract argumentation theory \cite[Section 2]{Dung:95}, making explicit all steps in the proofs, and occasionally providing lemmas that can make the longer proofs easier to comprehend. We also illustrate many of the concepts with examples. Further, we briefly recap the relevant aspects of directed graphs and lattice theory in the appendices.

This note will focus on definitions and technical results with minimal commentary. \textit{We do not claim originality as many of these results, especially those not explicitly stated in \cite{Dung:95}, should be folklore.} Our intention for writing this note is to collate all relevant results that may assist a reader coming to abstract argumentation theory, in particular \cite{Dung:95}, for the first time; this document can also serve as a reference for researchers. We will not cover further topics such as argument labellings (e.g. \cite{Caminada:06}), non-Dung semantics (e.g. \cite{Baroni:09,Baumann:15}), dialogical argumentation (e.g. \cite{Sanjay:09}) and structured argumentation (e.g. \cite{Bondarenko:97,Sanjay:13}).

\newpage


\section{Abstract Argumentation Frameworks}

\subsection{Definition and Basic Examples}

An \textit{abstract argumentation framework} is a directed graph (digraph) $\ang{A,R}$ where the set of nodes $A$ represent the set of arguments under consideration and the set of directed edges $R$ denote when a given argument is a counterargument to another argument or itself, usually due to logical inconsistency or conflicting values. This representation of arguments and how they disagree \textit{abstracts} away from the internal structure and content of arguments and the nature of such disagreements, hence the term ``abstract'' argumentation. This results in an \textit{external} theory of justification \cite{Dung:95}, as opposed to an \textit{internal} theory of justification concerned with whether individual arguments are valid or plausible. We assume the reader is familiar with graph theory, but have recapped the basic ideas, notation and definitions of graph theory in Appendix \ref{app:digraphs} (page \pageref{app:digraphs}).

\begin{definition}
\cite[Definition 2]{Dung:95} An \textbf{(abstract) argumentation framework} (AF) is a digraph $\ang{A,R}$ where $A$ is \textbf{the set of arguments} and $R\subseteq A^2$ is \textbf{the attack relation}. For $a,b\in A$, we say $a$ \textbf{attacks} / \textbf{is a counter-argument to} / \textbf{disagrees with} $b$ iff $(a,b)\in R$, denoted as $R(a,b)$. Further, we denote $(a,b)\notin R$ with $\neg R(a,b)$.
\end{definition}

\begin{example}\label{eg:nixon}
\cite[Example 9]{Dung:95} The \textbf{Nixon diamond} is the AF whose digraph is isomorphic to the directed cycle graph on two nodes, denoted $C_2$ (see Example \ref{eg:cycle}, page \pageref{eg:cycle}, for the notation), i.e. $A=\set{a,b}$ and $R=\set{(a,b),(b,a)}$. This is depicted in Figure \ref{fig:nixon}.

\begin{figure}[H]
\begin{center}
\begin{tikzpicture}[>=stealth',shorten >=1pt,node distance=2cm,on grid,initial/.style    ={}]
\tikzset{mystyle/.style={->,relative=false,in=0,out=0}};
\node[state] (a) at (0,0) {$ a $};
\node[state] (b) at (2,0) {$ b $};
\draw [->] (a) to [out = 45, in = 135] (b);
\draw [->] (b) to [out = 225, in = -45] (a);
\end{tikzpicture}
\caption{The AF depicting the Nixon diamond, from Example \ref{eg:nixon}.}\label{fig:nixon}
\end{center}
\end{figure}
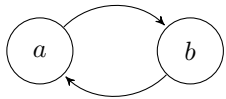
\end{example}

\begin{example}\label{eg:simple_reinstatement}
\textbf{Simple reinstatement} (e.g. \cite{Caminada:08,CogSci:10}) is the AF whose digraph is isomorphic to the directed path graph on three nodes, denoted $P_3$ (see Example \ref{eg:path}, page \pageref{eg:path}, for the notation),\footnote{This is also called a \textit{chain} of three arguments \cite{BenchCapon:14}.} i.e. $A=\set{a,b,c}$ and $R=\set{(b,a),(c,b)}$. This is depicted in Figure \ref{fig:simple_reinstatement}.

\begin{figure}[H]
\begin{center}
\begin{tikzpicture}[>=stealth',shorten >=1pt,node distance=2cm,on grid,initial/.style    ={}]
\tikzset{mystyle/.style={->,relative=false,in=0,out=0}};
\node[state] (a) at (0,0) {$ a $};
\node[state] (b) at (2,0) {$ b $};
\node[state] (c) at (4,0) {$ c $};
\draw [->] (c) to (b);
\draw [->] (b) to (a);
\end{tikzpicture}
\caption{The AF depicting simple reinstatement, from Example \ref{eg:simple_reinstatement}.}\label{fig:simple_reinstatement}
\end{center}
\end{figure}
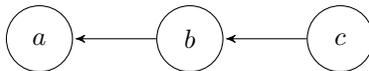
\end{example}

\begin{example}\label{eg:double_reinstatement}
\textbf{Double reinstatement}\footnote{This name may not be standard in the literature.} is the AF where $A=\set{a,b,c,e}$ and $R=\set{(b,a),(c,b),(e,b)}$. This is depicted in Figure \ref{fig:double_reinstatement}.\footnote{We will not use the letter ``$d$'' to denote arguments - see Section \ref{sec:defence_function}, page \pageref{sec:defence_function}.}

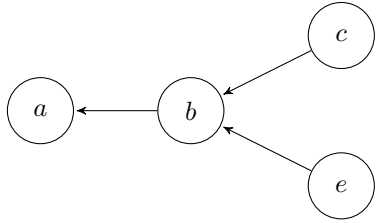
\begin{figure}[H]
\begin{center}
\begin{tikzpicture}[>=stealth',shorten >=1pt,node distance=2cm,on grid,initial/.style    ={}]
\tikzset{mystyle/.style={->,relative=false,in=0,out=0}};
\node[state] (a) at (0,0) {$ a $};
\node[state] (b) at (2,0) {$ b $};
\node[state] (c) at (4,1) {$ c $};
\node[state] (e) at (4,-1) {$ e $};
\draw [->] (c) to (b);
\draw [->] (b) to (a);
\draw [->] (e) to (b);
\end{tikzpicture}
\caption{The AF depicting double reinstatement, from Example \ref{eg:double_reinstatement}.}\label{fig:double_reinstatement}
\end{center}
\end{figure}
\end{example}

\begin{example}\label{eg:Israel_Arab}
\cite[Examples 1 and 3]{Dung:95} Consider the AF with $A=\set{a,b,c}$ and $R=\set{(a,b),(b,a),(c,b)}$. This is depicted in Figure \ref{fig:Israel_Arab}.

\begin{figure}[H]
\begin{center}
\begin{tikzpicture}[>=stealth',shorten >=1pt,node distance=2cm,on grid,initial/.style    ={}]
\tikzset{mystyle/.style={->,relative=false,in=0,out=0}};
\node[state] (a) at (0,0) {$ a $};
\node[state] (b) at (-2,0) {$ b $};
\node[state] (c) at (-4,0) {$ c $};
\draw [->] (a) to [out = 225, in = -45] (b);
\draw [->] (b) to [out = 45, in = 135] (a);
\draw [->] (c) to (b);
\end{tikzpicture}
\caption{The AF depicting \cite[Example 3]{Dung:95}, from Example \ref{eg:Israel_Arab}.}\label{fig:Israel_Arab}
\end{center}
\end{figure}
\end{example}

\begin{example}\label{eg:P4}
\cite[Example 2.3.5]{EoA} We can also have an AF whose underlying digraph is isomorphic to $P_4$, the directed path graph on four nodes (Example \ref{eg:path}, page \pageref{eg:path}), i.e. $A=\set{a,b,c,e}$, $R=\set{(e,c),(c,b),(b,a)}$. This is depicted in Figure \ref{fig:P4}.

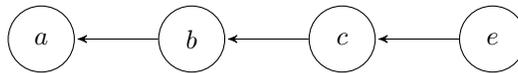
\begin{figure}[H]
\begin{center}
\begin{tikzpicture}[>=stealth',shorten >=1pt,node distance=2cm,on grid,initial/.style    ={}]
\tikzset{mystyle/.style={->,relative=false,in=0,out=0}};
\node[state] (a) at (0,0) {$ a $};
\node[state] (b) at (2,0) {$ b $};
\node[state] (c) at (4,0) {$ c $};
\node[state] (e) at (6,0) {$ e $};
\draw [->] (e) to (c);
\draw [->] (c) to (b);
\draw [->] (b) to (a);
\end{tikzpicture}
\caption{The AF depicting simple reinstatement, from Example \ref{eg:P4}.}\label{fig:P4}
\end{center}
\end{figure}
\end{example}

\begin{example}\label{eg:fri}
(See \cite[Figure 1]{Caminada_IE:07} and \cite[Figure 2]{CogSci:10}) \textbf{Floating reinstatement} is the AF where $A=\set{a,b,c,e}$, $R=\set{(a,b),(b,a),(a,c),(b,c),(c,e)}$. This AF is depicted in Figure \ref{fig:floating}.

\begin{figure}[H]
\begin{center}
\begin{tikzpicture}[>=stealth',shorten >=1pt,node distance=2cm,on grid,initial/.style    ={}]
\node[state] (a) at (0,0) {$ a $};
\node[state] (b) at (0,-2) {$ b $};
\node[state] (c) at (2,-1) {$ c $};
\node[state] (e) at (4,-1) {$ e $};
\tikzset{mystyle/.style={->,relative=false,in=0,out=0}};
\draw [->] (a) to [out=225,in=135] (b);
\draw [->] (b) to [out=45,in=-45] (a);
\draw [->] (a) to [out=0,in=135] (c);
\draw [->] (b) to [out=0,in=225] (c);
\draw [->] (c) to [out=0,in=180] (e);
\end{tikzpicture}
\end{center}
\caption{The argument framework for floating reinstatement, from Example \ref{eg:fri}}\label{fig:floating}
\end{figure}
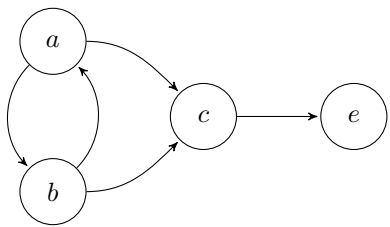
\end{example}

\noindent From now on we assume an arbitrary AF $\ang{A,R}$. Note that $R$ is not assumed to be a symmetric relation. This is to include cases where two arguments disagree in only one direction.\footnote{For recent experiments that investigate whether we can infer the direction of attacks from natural language, see \cite{Cramer:18}.}

\begin{example}
(From \cite{CogSci:10}) Let the argument $a$ represent ``Mary does not limit her phone usage. Therefore, Mary has a large phone bill.''. Let the argument $b$ represent ``Mary has a speech disorder. Therefore, Mary limits her phone usage.'' In real-life dialogues, arguments support a conclusion, and it is the conclusion (instead of any other part of the argument) that is used to agree or disagree with other arguments that support various other conclusions \cite[Section 2]{Sanjay:13}. In this example, the argument $b$ attacks $a$ because the conclusion of $b$ attacks an assumption of $a$, and the attack is not symmetric because the conclusion of $a$ does not disagree with anything $b$ has concluded or assumed.
\end{example}

\subsection{Types of Attacks and Examples}

\begin{definition}\label{def:set_attacks}
\cite[Remark 4]{Dung:95} We say $S\subseteq A$ \textbf{attacks} $a\in A$ iff $a\in S^+$.\footnote{$S^+$ denotes the set of arguments attacked by some argument in $S$ (see Definition \ref{def:S_pm}, page \pageref{def:S_pm}). This is also called the \textit{forward} or \textit{successor} set of $S$ in a digraph.} We say $a$ \textbf{attacks} $S$ iff $a\in S^-$.\footnote{$S^-$ denotes the set of arguments attacking some argument in $S$ (see Definition \ref{def:S_pm}, page \pageref{def:S_pm}). This is also called the \textit{backward} or \textit{predecessor} set of $S$ in a digraph.} We say $S$ \textbf{attacks} $T\subseteq A$ iff $S^+\cap T\neq\es$.
\end{definition}

\noindent Definition \ref{def:set_attacks} generalises attacks between individual arguments to \textit{sets} of arguments. By Corollary \ref{cor:es_pm} (page \pageref{cor:es_pm}), the empty set $\es$ can never attack any argument, nor can it be attacked by any argument.

\begin{example}
(Example \ref{eg:nixon} continued) In the Nixon diamond, $b\in\set{a}^+=a^+$, hence the set $\set{a}$ attacks $b$. By symmetry, $a\in b^+$.
\end{example}

\begin{example}
(Example \ref{eg:Israel_Arab} continued) Clearly, $c$ attacks the set $\set{a,b}$, i.e. $c\in\set{a,b}^-$, because $c$ attacks $b$.
\end{example}

\begin{example}
(Example \ref{eg:double_reinstatement} continued) In double reinstatement, $\set{c,e}$ attacks $\set{a,b}$, because both $c$ and $e$ attack $b$.
\end{example}

\begin{definition}
We say $a\in A$ is an \textbf{unattacked argument} iff $a^-=\es$.
\end{definition}

\noindent When arguments are represented by nodes of a digraph, unattacked arguments correspond to source nodes (Definition \ref{def:source_node}, page \pageref{def:source_node}). Unattacked arguments are important because we will see that they always \textit{win}. This formalises the idea that the person with the last word always wins the argument, because such arguments have no counter-argument represented in the AF \cite[Section 1]{Dung:95}. Another way of understanding this is that the claim of such an argument is seen as provisionally true until it is explicitly rebutted.

\begin{definition}\label{def:unattacked}
$U:=\set{a\in A\:\vline\:a^-=\es}$ is the \textbf{set of all unattacked arguments}.\footnote{We have a slightly less general definition here compared to \cite[Definition 2.9]{Baroni:09}, where a set $S$ of non-empty arguments is \textit{unattacked} iff $S^-=\es$, i.e. there are no arguments outside of $S$ that is attacking $S$. In our case, $U$ is the $\subseteq$-greatest such set.}
\end{definition}

\begin{example}
(Example \ref{eg:P4} continued) We have $U=\set{e}$.
\end{example}

\begin{example}
(Example \ref{eg:fri} continued) In floating reinstatement, we have $U=\es$ because every argument is being attacked.
\end{example}

\begin{corollary}\label{cor:unattacked_characterisation}
We have that $a\in U\Leftrightarrow a\notin A^+$.
\end{corollary}
\begin{proof}
$a\in U$ iff $a^-=\es$ iff $\pair{\forall b\in A}b\notin a^-$ iff $a\notin A^+$.
\end{proof}

\begin{definition}
An argument $a\in A$ is \textbf{self-attacking} iff $a\in a^+$, equivalently $a\in a^-$ by Corollary \ref{cor:plus_minus} (page \pageref{cor:plus_minus}).
\end{definition}

\begin{example}\label{eg:ab_self_attack}
Consider $A=\set{a,b}$ and $R=\set{\pair{a,a},\pair{a,b}}$. This AF is depicted in Figure \ref{fig:ab_self_attack}.

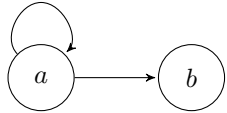
\begin{figure}[H]
\begin{center}
\begin{tikzpicture}[>=stealth',shorten >=1pt,node distance=2cm,on grid,initial/.style    ={}]
\node[state] (a) at (0,0) {$ a $};
\node[state] (b) at (2,0) {$ b $};
\tikzset{mystyle/.style={->,relative=false,in=0,out=0}};
\draw [->] (a) to [out=135,in=180] ($(a) + (0,1)$) to [out = 0, in = 45] (a);
\draw [->] (a) to (b);
\end{tikzpicture}
\end{center}
\caption{An AF with a self-attacking argument, from Example \ref{eg:ab_self_attack}.}\label{fig:ab_self_attack}
\end{figure}

\noindent In this AF, $a$ is our self-attacking argument. Further, the set of unattacked arguments $U=\es$.
\end{example}

\subsection{Basic Types of Abstract Argumentation Frameworks}

\begin{definition}
An AF is \textbf{empty} iff $A=\es$.
\end{definition}

\begin{definition}
An AF is \textbf{trivial} iff $R=\es$.
\end{definition}

\begin{definition}\label{def:symmetric}
An AF is \textbf{symmetric} iff $R$ is a non-empty symmetric relation.\footnote{Like \cite{Coste:05}, we exclude the empty relation as it is vacuously symmetric. Unlike \cite{Coste:05}, we do not restrict our attention to finite argumentation frameworks, i.e. when $A$ of $\ang{A,R}$ is a finite set. Further, we do not exclude the possibility of having self-attacking arguments, as $(a,a)\in R$ does not violate that $R$ is a symmetric relation.}
\end{definition}

\begin{definition}
An AF is \textbf{finite} iff $A$ is a finite set. Else, the AF is \textbf{infinite}.
\end{definition}

\begin{example}
All of Examples \ref{eg:nixon} to \ref{eg:ab_self_attack} above are finite, non-trivial and non-empty AFs. Of these examples, only Example \ref{eg:nixon} is a symmetric AF.
\end{example}

\noindent In this note, we do not assume that the AFs we deal with are finite; they can be finite or infinite \cite{Baumann:15}.

\begin{example}\label{eg:semi_inf}
The following AF is infinite. Let $A=\set{a_i,b_i}_{i\in\nat}$ and $R=\set{\pair{b_i, a_i},\pair{a_{i+1},b_i}}_{i\in\nat}$. This AF is depicted in Figure \ref{fig:semi_inf}.

\begin{figure}[H]
\begin{center}
\begin{tikzpicture}[>=stealth',shorten >=1pt,node distance=2cm,on grid,initial/.style    ={}]
\tikzset{mystyle/.style={->,relative=false,in=0,out=0}};
\node[state] (a0) at (0,0) {$ a_0 $};
\node[state] (b0) at (2,0) {$ b_0 $};
\node[state] (a1) at (4,0) {$ a_1 $};
\node[state] (b1) at (6,0) {$ b_1 $};
\node[state] (a2) at (8,0) {$ a_2 $};
\node (dots) at (10,0) {\Huge $ \cdots $};
\draw [->] (b0) to (a0);
\draw [->] (a1) to (b0);
\draw [->] (b1) to (a1);
\draw [->] (a2) to (b1);
\draw [->] (dots) to (a2);
\end{tikzpicture}
\end{center}
\caption{An infinite AF, from Example \ref{eg:semi_inf}}\label{fig:semi_inf}
\end{figure}
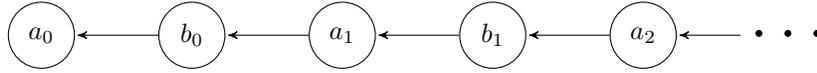
\end{example}

\begin{example}\label{eg:bi_inf}
The following AF is also infinite. Let $A=\set{a_i,b_i}_{i\in\integ}$ and $R=\set{\pair{b_i,a_i},\pair{a_{i+1},b_i}}_{i\in\integ}$. This AF is depicted in Figure \ref{fig:bi_inf}.

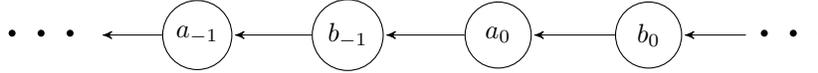
\begin{figure}[H]
\begin{center}
\begin{tikzpicture}[>=stealth',shorten >=1pt,node distance=2cm,on grid,initial/.style    ={}]
\tikzset{mystyle/.style={->,relative=false,in=0,out=0}};
\node (dotsn) at (-2,0) {\Huge $ \cdots $};
\node[state] (a-1) at (0,0) {$ a_{-1} $};
\node[state] (b-1) at (2,0) {$ b_{-1} $};
\node[state] (a0) at (4,0) {$ a_0 $};
\node[state] (b0) at (6,0) {$ b_0 $};
\node (dotsp) at (8,0) {\Huge $ \cdots $};
\draw [->] (a-1) to (dotsn);
\draw [->] (b-1) to (a-1);
\draw [->] (a0) to (b-1);
\draw [->] (b0) to (a0);
\draw [->] (dotsp) to (b0);
\end{tikzpicture}
\end{center}
\caption{An infinite AF, from Example \ref{eg:bi_inf}}\label{fig:bi_inf}
\end{figure}
\end{example}

The infinite\footnote{Clearly the AFs from Examples \ref{eg:semi_inf} and \ref{eg:bi_inf} have \textit{countably} infinitely many arguments. It is also possible for AFs to have \textit{uncountably} infinitely many arguments. See \cite[Section 3.1]{Dung:95} or \cite{young2020continuum,CGT0}} AFs from Examples \ref{eg:semi_inf} and \ref{eg:bi_inf} are \textit{locally finite} in that each argument only has one other argument attacking it. This motivates the following definition:

\begin{definition}
\cite[Definition 27]{Dung:95} An AF is \textbf{finitary} iff $\pair{\forall a\in A}\abs{a^-}<\aleph_0$.
\end{definition}

\begin{corollary}\label{cor:finite_finitary}
Finite AFs are finitary. The converse is not true.
\end{corollary}

\begin{proof}
If $\ang{A,R}$ is finite, then for all $a\in A$, the set $a^-\subseteq A$ is also finite.

An example of an infinite finitary AF is depicted in Example \ref{eg:semi_inf}, where every argument has exactly one attacker (finitary) but there are countably infinitely many arguments. Therefore, the converse is not true. 
\end{proof}



\begin{example}\label{eg:non-finitary_AF}
The following is an example of a non-finitary AF. By the contrapositive of Corollary \ref{cor:finite_finitary}, the AF cannot be finite. Let $A=\set{a}\cup\set{b_i}_{i\in\nat}$ and $R=\set{\pair{b_i,a}}_{i\in\nat}$. This AF is depicted in Figure \ref{fig:non-finitary_AF}.

\begin{figure}[H]
\begin{center}
\begin{tikzpicture}[>=stealth',shorten >=1pt,node distance=2cm,on grid,initial/.style    ={}]
\tikzset{mystyle/.style={->,relative=false,in=0,out=0}};
\node[state] (0) at (4,-2) {$ a $};
\node[state] (1) at (2,0) {$ b_0 $};
\node[state] (2) at (4,0) {$ b_1 $};
\node[state] (3) at (6,0) {$ b_2 $};
\node[state] (4) at (8,0) {$ b_3 $};
\node (dottop) at (9.5,0) {\Huge $ \cdots $};
\node (dotbot) at (7.5,-1) {\Huge $ \cdots $};
\draw [->] (1) to (0);
\draw [->] (2) to (0);
\draw [->] (3) to (0);
\draw [->] (4) to (0);
\end{tikzpicture}
\end{center}
\caption{An infinite non-finitary AF, from Example \ref{eg:non-finitary_AF}.}\label{fig:non-finitary_AF}
\end{figure}
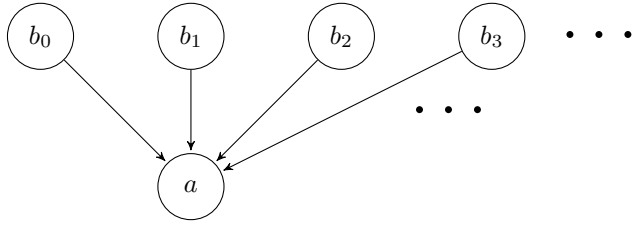
\noindent Then $a^-=\set{b_i}_{i\in\nat}$, which means it has infinitely many attackers. This $\ang{A,R}$ is therefore not finitary.
\end{example}

\begin{example}\label{eg:non_finitary2}
The following is another example of a non-finitary AF. Let $A=\set{a}\cup\set{b_i,c_i}_{i\in\nat}$ and $R=\set{\pair{b_i,a},\pair{c_i,b_i}}_{i\in\nat}$. We have $a^-=\set{b_i}_{i\in\nat}$ hence $\abs{a^-}=\aleph_0$. This AF is depicted in Figure \ref{fig:non_finitary2}.

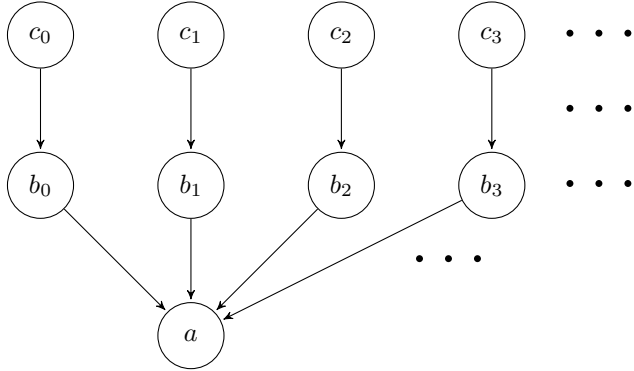
\begin{figure}[H]
\begin{center}
\begin{tikzpicture}[>=stealth',shorten >=1pt,node distance=2cm,on grid,initial/.style    ={}]
\tikzset{mystyle/.style={->,relative=false,in=0,out=0}};
\node[state] (0) at (4,-2) {$ a $};
\node[state] (1) at (2,0) {$ b_0 $};
\node[state] (2) at (4,0) {$ b_1 $};
\node[state] (3) at (6,0) {$ b_2 $};
\node[state] (4) at (8,0) {$ b_3 $};
\node (dottop) at (9.5,0) {\Huge $ \cdots $};
\node (dotbot) at (7.5,-1) {\Huge $ \cdots $};
\draw [->] (1) to (0);
\draw [->] (2) to (0);
\draw [->] (3) to (0);
\draw [->] (4) to (0);
\node[state] (5) at (2,2) {$ c_0 $};
\node[state] (6) at (4,2) {$ c_1 $};
\node[state] (7) at (6,2) {$ c_2 $};
\node[state] (8) at (8,2) {$ c_3 $};
\node (dot2) at (9.5,1) {\Huge $ \cdots $};
\node (dot1) at (9.5,2) {\Huge $ \cdots $};
\draw [->] (5) to (1);
\draw [->] (6) to (2);
\draw [->] (7) to (3);
\draw [->] (8) to (4);
\end{tikzpicture}
\end{center}
\caption{An infinite non-finitary AF, from Example \ref{eg:non_finitary2}.}\label{fig:non_finitary2}
\end{figure}
\end{example}

\begin{definition}
Let $B\subseteq A$. The \textbf{(induced) sub-framework w.r.t. $B$} is the AF $\ang{B,R\cap B^2}$.
\end{definition}

\begin{example}
(Example \ref{eg:bi_inf} continued) Clearly, the AF from Example \ref{eg:semi_inf} is a sub-framework of this AF.
\end{example}

\begin{example}
(Example \ref{eg:non_finitary2} continued) Clearly, the AF in Figure \ref{fig:non-finitary_AF} is a sub-framework of the AF in Figure \ref{fig:non_finitary2}.
\end{example}

\begin{example}
(Example \ref{eg:simple_reinstatement} continued, page \pageref{eg:simple_reinstatement}) The idea of induced sub-frameworks allow us to model the course of a dialogue (e.g. such as online debates \cite{boschi2021has,young2020ranking,young2018approx}). The set of arguments could be the arguments that have so far been mentioned during the dialogue. Consider a dialogue based on simple reinstatement (Example \ref{eg:simple_reinstatement}). We can imagine Agent 1 claiming $a$ and we have the induced sub-framework w.r.t. $\set{a}$. We then imagine Agent 2 claiming $b$, so the set of arguments mentioned so far is $\set{a,b}$, and the corresponding induced sub-framework also has the attack $R(b,a)$. Finally, Agent 1 responds by claiming $c$ and the dialogue ends, so the set of arguments mentioned so far is $\set{a,b,c}$ and we recover the full framework of simple reinstatement.
\end{example}

\subsection{Cycles in Argumentation Frameworks}

One could then imagine cycles in AFs which can represent ``never-ending'' courses of dialogue where agents can repeat the same arguments over and over again. We will see that this will make determining the winning arguments problematic \cite{BenchCapon:14}.

\begin{definition}
We say that an AF is \textbf{cyclic} iff it contains a (directed) cycle, else it is \textbf{acyclic}.
\end{definition}

\noindent Necessarily, the number of arguments in a cycle is finite.

\begin{definition}
Given a cycle in an AF, its \textbf{parity} is whether the number of arguments in the cycle is even or odd.
\end{definition}

\begin{example}
(Example \ref{eg:ab_self_attack} continued, page \pageref{eg:ab_self_attack}) This is a cyclic AF with an odd cycle, specifically the 1-cycle where the argument $a$ is self-attacking.
\end{example}

\begin{example}
(Example \ref{eg:nixon} continued, page \pageref{eg:nixon}) The Nixon diamond is a cyclic AF with a 2-cycle.
\end{example}

\begin{example}\label{eg:3cycle}
\cite[Example 2.3.1]{EoA} The following is a cyclic AF with a 3-cycle: $A=\set{a,b,c,e}$ and $R=\set{(b,a),(c,b),(e,c),(b,e)}$. This AF is depicted in Figure \ref{fig:3cycle}.

\begin{figure}[H]
\begin{center}
\begin{tikzpicture}[>=stealth',shorten >=1pt,node distance=2cm,on grid,initial/.style    ={}]
\tikzset{mystyle/.style={->,relative=false,in=0,out=0}};
\node[state] (a) at (0,0) {$ a $};
\node[state] (b) at (-2,0) {$ b $};
\node[state] (c) at (-4,1) {$ c $};
\node[state] (e) at (-4,-1) {$ e $};
\draw [->] (b) to (a);
\draw [->] (c) to (b);
\draw [->] (e) to (c);
\draw [->] (b) to (e);
\end{tikzpicture}
\end{center}
\caption{An AF containing a 3-cycle, from Example \ref{eg:3cycle}.}\label{fig:3cycle}
\end{figure}
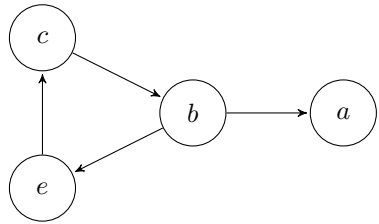
\end{example}

\begin{definition}\label{def:well_founded}
\cite[Definition 29]{Dung:95} An AF is \textbf{well-founded} iff there is no $A$-sequence $\set{a_i}_{i\in\nat}$ such that $\pair{\forall i\in\nat}R\pair{a_{i+1},a_i}$.
\end{definition}


\begin{example}
(Example \ref{eg:P4} continued, page \pageref{eg:P4}) This AF is well-founded because its attack sequence is finite, involving only four arguments.
\end{example}

\begin{example}
(Example \ref{eg:nixon} continued, page \pageref{eg:nixon}) This AF is not well-founded due to the $A$-sequence $\set{a,b,a,b,a,b,a,b,\ldots}$, because $R(a,b)$ and $R(b,a)$.
\end{example}

\begin{corollary}\label{cor:wf_unattacked}
If $\ang{A,R}$ is well-founded then $U\neq\es$. The converse is not true in general.
\end{corollary}
\begin{proof}
(Contrapositive) If $U=\es$ then $\pair{\forall a\in A}a^-\neq\es$. Let $a_0\in A$, then there is some $a_1\in a_0^-$. Similarly, there is some $a_2\in a_1^-$. So for any $a_i\in A$ there exists some $a_{i+1}\in a_i^-$. This gives an infinite sequence $\set{a_i}_{i\in\nat}$ such that $\pair{\forall i\in\nat}R\pair{a_{i+1},a_i}$. Therefore, $\ang{A,R}$ is not well-founded.

For the converse, consider $A=\set{a_i}_{i\in\nat}\cup\set{b}$ such that $\pair{\forall i\in\nat}R\pair{a_{i+1},a_i}$. This AF is depicted in Figure \ref{fig:point_off_path}.

\begin{figure}[H]
\begin{center}
\begin{tikzpicture}[>=stealth',shorten >=1pt,node distance=2cm,on grid,initial/.style    ={}]
\tikzset{mystyle/.style={->,relative=false,in=0,out=0}};
\node[state] (0) at (-2,0) {$ b $};
\node[state] (0) at (0,0) {$ a_0 $};
\node[state] (1) at (2,0) {$ a_1 $};
\node[state] (2) at (4,0) {$ a_2 $};
\node[state] (3) at (6,0) {$ a_3 $};
\node (4) at (8,0) {\Huge $ \cdots $};
\draw [->] (1) to (0);
\draw [->] (2) to (1);
\draw [->] (3) to (2);
\draw [->] (4) to (3);
\end{tikzpicture}
\end{center}
\caption{An example of an AF that satisfies $U\neq\es$ and is not well-founded, from Corollary \ref{cor:wf_unattacked}.}\label{fig:point_off_path}
\end{figure}
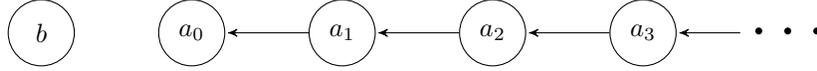

\noindent This AF has an infinite sequence of arguments with an isolated argument $b$. Notice that $U=\set{b}\neq\es$, but this AF is not well-founded.
\end{proof}


\begin{corollary}\label{cor:cyclic_not_wf}
\cite[Proposition 2]{Coste:05} Cyclic AFs are not well-founded. The converse is not true in general.
\end{corollary}
\begin{proof}
If $\ang{A,R}$ has a cycle then denote that cycle as $\set{a_i}_{i=1}^k$ such that $$\pair{\forall 1\leq i\leq k-1}R\pair{a_{i+1},a_i}\text{ and }R(a_1,a_k).$$ This gives an infinite, periodic $A$-sequence $$a_1,a_k,a_{k-1},a_{k-2},\ldots,a_2,a_1,a_k,a_{k-1},\ldots$$ such that $R(a_1,a_k)$, $R(a_k,a_{k-1})$, ... etc. Therefore, $\ang{A,R}$ cannot be well-founded.

For the converse, Figure \ref{fig:point_off_path} from Corollary \ref{cor:wf_unattacked} is an example of a non-well-founded AF that is acyclic.
\end{proof}

\begin{corollary}
\cite[Proposition 3]{Coste:05} If an AF is symmetric, then it is not well-founded. The converse is not true.
\end{corollary}
\begin{proof}
If an AF $\ang{A,R}$ is symmetric, then $R\neq\es$ is a symmetric relation. We have some $(a,b)\in R$ and hence $(b,a)\in R$. This is a 2-cycle, hence by Corollary \ref{cor:cyclic_not_wf}, this AF is not well-founded.

For the converse, Example \ref{eg:semi_inf} (page \pageref{eg:semi_inf}) is a non-well-founded AF that is not symmetric.
\end{proof}

\noindent The following result gives an equivalent characterisation of non-well-foundedness for an AF.

\begin{corollary}\label{cor:non_wf_core}
$\ang{A,R}$ is not well-founded iff there exists some $\es\neq S\subseteq A$ such that $S\subseteq S^+$.
\end{corollary}
\begin{proof}
($\Leftarrow$) To demonstrate that the underlying $\ang{A,R}$ is not well-founded, we use induction to construct the desired sequence $\set{a_i}_{i\in\nat}$.

\begin{enumerate}
\item (Base) As $S\neq\es$, let $a_0\in S$.
\item (Inductive) Let $a_i\in S$, then $a_i\in S^+$ by the hypothesis $S\subseteq S^+$ and hence there is some $a_{i+1}\in S$ such that $R\pair{a_{i+1},a_i}$.
\end{enumerate}

\noindent By induction, $\set{a_i}_{i\in\nat}$ is the sequence satisfying $\pair{\forall i\in\nat}R\pair{a_{i+1},a_i}$. This shows that the underlying $\ang{A,R}$ is not well-founded.

($\Rightarrow$) If $\ang{A,R}$ is not well-founded, then there is a sequence $\set{a_i}_{i\in\nat}$ such that $\pair{\forall i\in\nat}R\pair{a_{i+1},a_i}$. Let $S:=\set{a_i}_{i\in\nat}$. Clearly, $S\neq\es$. For any $a_i\in S$, there is some $a_{i+1}\in S$ such that $R\pair{a_{i+1},a_i}$, therefore $a_i\in S^+$ and hence $S\subseteq S^+$.
\end{proof}

\begin{corollary}\label{cor:fin_acyc_wf}
A finite acyclic AF is well-founded. The converse is not true.
\end{corollary}
\begin{proof}
(Contrapositive) Let $\ang{A,R}$ be finite and not well-founded. We seek to construct an attack cycle. By non-well-foundedness, there exists an $A$-sequence $\set{a_i}_{i\in\nat}$ such that $\pair{\forall i\in\nat}R\pair{a_{i+1},a_i}$. As $A$ is finite, WLOG $\abs{A}=N\in\nat$, then the first $N+1$ terms of the sequence $\set{a_i}_{i\in\nat}$ must have some repeating argument by the pigeonhole principle. Let $b$ be such an argument, then we can construct a cycle starting and ending with $b$ through the property $\pair{\forall i\in\nat}R\pair{a_{i+1},a_i}$. Therefore, $AF$ is cyclic.\footnote{We cannot prove the contrapositive by assuming that the AF is cyclic and not well-founded because this contradicts Corollary \ref{cor:cyclic_not_wf}.}

As for the converse, the AF where $A=\set{a_i,b_i}_{i\in\nat}$ and $R=\set{\pair{a_i,b_i}}_{i\in\nat}$ is well-founded, acyclic but infinite. This AF is depicted in Figure \ref{fig:fin_acyc_wf}.

\begin{figure}[H]
\begin{center}
\begin{tikzpicture}[>=stealth',shorten >=1pt,node distance=2cm,on grid,initial/.style    ={}]
\tikzset{mystyle/.style={->,relative=false,in=0,out=0}};
\node[state] (1) at (2,0) {$ b_0 $};
\node[state] (2) at (4,0) {$ b_1 $};
\node[state] (3) at (6,0) {$ b_2 $};
\node[state] (4) at (8,0) {$ b_3 $};
\node (dottop) at (9.5,0) {\Huge $ \cdots $};
\node[state] (5) at (2,2) {$ a_0 $};
\node[state] (6) at (4,2) {$ a_1 $};
\node[state] (7) at (6,2) {$ a_2 $};
\node[state] (8) at (8,2) {$ a_3 $};
\node (dot2) at (9.5,1) {\Huge $ \cdots $};
\node (dot1) at (9.5,2) {\Huge $ \cdots $};
\draw [->] (5) to (1);
\draw [->] (6) to (2);
\draw [->] (7) to (3);
\draw [->] (8) to (4);
\end{tikzpicture}
\end{center}
\caption{An infinite AF that is well-founded, from Corollary \ref{cor:fin_acyc_wf}.}\label{fig:fin_acyc_wf}
\end{figure}
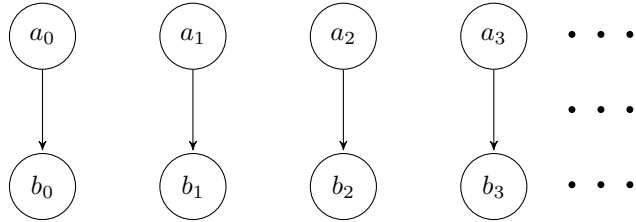
\noindent Therefore, the converse is not true. The result follows.
\end{proof}

\subsection{Controversy in Argumentation Frameworks}

\begin{definition}
We say $a\in A$ \textbf{indirectly attacks} $b\in A$ iff there is an odd-length path from $a$ to $b$ in $\ang{A,R}$.
\end{definition}

\noindent Notice that direct attacks, i.e. paths of length 1, are special cases of indirect attacks.

\begin{example}
(Example \ref{eg:P4} continued, page \pageref{eg:P4}) The argument $e$ indirectly attacks the argument $a$, as there is a path of length 3 from $e$ to $a$.
\end{example}

\begin{example}
Self-attacking arguments both directly and indirectly attack themselves.
\end{example}

\begin{definition}\label{def:indirectly_defends}
We say $a\in A$ \textbf{indirectly defends} $b\in A$ iff there is an even-length path from $a$ to $b$ in $\ang{A,R}$.
\end{definition}

\begin{example}
(Example \ref{eg:nixon} continued, page \pageref{eg:nixon}) The argument $b$ indirectly defends itself as there is a path of length 2 from itself to itself.
\end{example}

\begin{example}
Self-attacking arguments also indirectly defend themselves, by going through their loop twice to obtain a path of length 2.
\end{example}

\begin{corollary}\label{cor:no_loop_indir_self_def}
Every non-self-attacking argument in an AF indirectly defends itself.
\end{corollary}
\begin{proof}
Every non self-attacking argument has a path length of 0 to itself, which is an even path.\footnote{We would not call such ``paths'' cycles though.}
\end{proof}

\begin{definition}
We say \textbf{$a\in A$ is controversial w.r.t. $b\in A$} iff $a\in A$ indirectly attacks and indirectly defends $b\in A$.
\end{definition}

\begin{example}\label{eg:controversial_triangle}
\cite[Example 2.1.2]{EoA} Consider $A=\set{a,b,c}$ such that $R(a,b)$, $R(b,c)$ and $R(a,c)$. It is clear that $a$ is controversial w.r.t. $c$. This AF is depicted in Figure \ref{fig:controversial_triangle}.

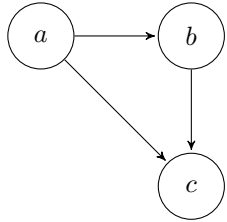
\begin{figure}[H]
\begin{center}
\begin{tikzpicture}[>=stealth',shorten >=1pt,node distance=2cm,on grid,initial/.style    ={}]
\tikzset{mystyle/.style={->,relative=false,in=0,out=0}};
\node[state] (a) at (0,0) {$ a $};
\node[state] (b) at (2,0) {$ b $};
\node[state] (c) at (2,-2) {$ c $};
\draw [->] (a) to (b);
\draw [->] (b) to (c);
\draw [->] (a) to (c);
\end{tikzpicture}
\end{center}
\caption{An example of controversy, from Example \ref{eg:controversial_triangle}.}\label{fig:controversial_triangle}
\end{figure}
\end{example}

\begin{example}
Each self-attacking argument is controversial with respect to itself, because it both indirectly attacks and indirectly defends itself.
\end{example}

\begin{definition}
We say $a\in A$ is \textbf{controversial} iff there is some $b\in A$ such that $a$ is controversial w.r.t. $b$.
\end{definition}

\begin{definition}
\cite[Definition 32(1)]{Dung:95} An AF $\ang{A,R}$ is \textbf{controversial} iff there is some controversial argument in $A$. Else, the AF is \textbf{uncontroversial}.
\end{definition}

\begin{example}
(Example \ref{eg:controversial_triangle} continued) This AF is controversial, as $a$ is a controversial argument.
\end{example}

\begin{example}
(Example \ref{eg:ab_self_attack}, page \pageref{eg:ab_self_attack} continued) This AF is controversial because it contains the self-attacking argument $a$.
\end{example}

\begin{definition}
\cite[Definition 32(2)]{Dung:95} An AF $\ang{A,R}$ is \textbf{limited controversial} iff there is \textit{no} $A$-sequence $\set{a_i}_{i\in\nat}$ such that $a_{i+1}$ is controversial with respect to $a_i$.
\end{definition}

\begin{corollary}\label{cor:LC_no_self_attack}
If an AF is limited controversial then it cannot have self-attacking arguments.
\end{corollary}
\begin{proof}
(Contrapositive) If an AF has a self-attacking argument $a$, then it is controversial w.r.t. itself so the constant sequence $\set{a}_{i\in\nat}$ renders the AF not limited controversial.
\end{proof}

\noindent We will answer the converse of Corollary \ref{cor:LC_no_self_attack} in Corollary \ref{cor:odd_cycle_not_lc}.

\begin{corollary}\label{cor:uncont_LC}
Uncontroversial AFs are limited controversial. The converse is not true in general.
\end{corollary}
\begin{proof}
An uncontroversial AF has no controversial arguments and hence there is no infinite sequence of arguments controversial with respect to its predecessor. Therefore, such AFs are also limited controversial.

For the converse, there is no infinite sequence of arguments in the AF of Example \ref{eg:controversial_triangle} where each is controversial with respect to its predecessors. Therefore, in this example, $\ang{A,R}$ is limited controversial and controversial.
\end{proof}

The property of being limited controversial is downwards inheritable.

\begin{corollary}\label{cor:LC_subAF}
If $\ang{A',R'}\subseteq_g\ang{A,R}$ and $\ang{A,R}$ is limited controversial, then $\ang{A',R'}$ is also limited controversial. The converse is not true.
\end{corollary}
\begin{proof}
If $\ang{A,R}$ is limited controversial, then there is no infinite $A$-sequence of arguments such that each is controversial w.r.t. to its predecessor. Therefore, any sub-framework of $\ang{A,R}$ is also limited controversial.

For the converse, consider any AF $\ang{A,R}$ and take its disjoint union with the AF consisting of a single self-attacking argument. More precisely, for $x\notin A$, consider the AF $\ang{A\cup\set{x}, R\cup\set{(x,x)}}$. The first AF is an induced subgraph of the second, but the second is not limited controversial by the contrapositive of Corollary \ref{cor:LC_no_self_attack}.
\end{proof}

As self-attacking arguments are cycles of length 1 (hence odd), the following result generalises Corollary \ref{cor:LC_no_self_attack} and also shows that its converse is not true.

\begin{corollary}\label{cor:odd_cycle_not_lc}
If an AF is limited controversial, then it has no odd cycle. The converse is not true.
\end{corollary}
\begin{proof}
(Contrapositive, from \cite{BenchCapon:14}) Assume the AF has an odd cycle with arguments with $a_1$, $a_2$, ... , $a_n$, where $\pair{\forall 1\leq i\leq n}R\pair{a_i,a_{i+1}}$, $R\pair{a_n,a_1}$, and $n$ is odd. For $1\leq i \leq n$, there is a path of length $i+nk$ from $a_{i+1}$ to $a_i$, where $k\in\nat$ is the number of times around the cycle. Depending on $k$, $j-i+nk$ is both even and odd. Therefore, $a_{i+1}$ is controversial w.r.t. $a_n$, down to $a_1$ is controversial w.r.t. $a_0$. We can do this infinitely many times by continuing that $a_1$ is controversial w.r.t. $a_n$... etc. This generates our infinite $A$-sequence such that each $a_{i+1}$ is controversial w.r.t. $a_i$.

For the converse, we construct an AF that is not limited controversial, but has no odd cycle. Consider an AF with arguments $A=\set{a_i,b_i}_{i\in\nat}$ and attacks $R=\set{\pair{a_{i+1},a_i}}_{i\in\nat}\cup\set{\pair{b_i,a_i}}_{i\in\nat}\cup\set{\pair{a_{i+1},b_i}}_{i\in\nat}$. This is depicted in Figure \ref{fig:chain_triangle}.

\begin{figure}[H]
\begin{center}
\begin{tikzpicture}[>=stealth',shorten >=1pt,node distance=2cm,on grid,initial/.style    ={}]
\tikzset{mystyle/.style={->,relative=false,in=0,out=0}};
\node[state] (0) at (0,0) {$ a_0 $};
\node[state] (1) at (2,0) {$ a_1 $};
\node[state] (2) at (4,0) {$ a_2 $};
\node[state] (3) at (6,0) {$ a_3 $};
\draw [->] (1) to (0);
\draw [->] (2) to (1);
\draw [->] (3) to (2);
\draw [->] (4) to (3);
\node (4) at (8.5,0) {\Huge $ \cdots $};
\node[state] (b0) at (1,1) {$ b_0 $};
\node[state] (b1) at (3,1) {$ b_1 $};
\node[state] (b2) at (5,1) {$ b_2 $};
\draw [->] (b0) to (0);
\draw [->] (1) to (b0);
\draw [->] (b1) to (1);
\draw [->] (2) to (b1);
\draw [->] (b2) to (2);
\draw [->] (3) to (b2);
\draw [->] (7,1) to (3);
\node (5) at (8,1) {\Huge $ \cdots $};
\end{tikzpicture}
\end{center}
\caption{An example of an AF that is not limited controversial and has no odd cycle, from Corollary \ref{cor:odd_cycle_not_lc}.}\label{fig:chain_triangle}
\end{figure}
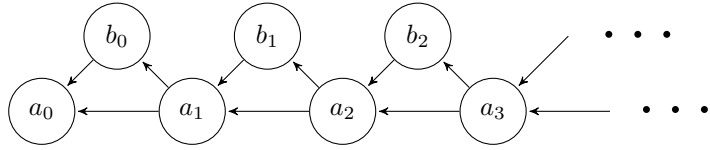

Clearly, $a_{i+1}$ is controversial w.r.t. $a_i$ for $i\in\nat$, therefore this AF is not limited controversial. However, there is no odd cycle. Therefore, the converse to this result is not true in general.
\end{proof}

\begin{corollary}\label{cor:finite_no_odd_cycles_LC}
A finite AF without any odd cycles is limited controversial.
\end{corollary}
\begin{proof}
(Contrapositive) Assume that the finite AF is not limited controversial.\footnote{We cannot assume that the AF is infinite because an infinite AF without any odd cycles does not have to be limited controversial by the converse of Corollary \ref{cor:odd_cycle_not_lc}.} Then there exists an infinite sequence of arguments $\set{a_i}_{i\in\nat}$ such that $a_{i+1}$ is controversial w.r.t. $a_i$. But as the AF is finite, this sequence must be periodic, so we have $a_0, a_1, \ldots, a_k, a_0$ for some $k\in\nat$ such that $a_{i+1}$ has both an even and an odd path to $a_i$. We construct our odd cycle as follows: we take all even paths from $a_0$ to $a_k$, $a_k$ to $a_{k-1}$, ... and $a_2$ to $a_1$, but take an odd path from $a_1$ to $a_0$. The result is an odd cycle.
\end{proof}

\subsection{Summary}

\begin{itemize}
\item An (abstract) argumentation framework (AF) is a digraph $\ang{A,R}$ where $A$ is the set of arguments under consideration and $R$ is the binary attack relation.
\item For $S\subseteq A$, $S^+\subseteq A$ is the set of arguments attacked by $S$, and $S^-\subseteq A$ is the set of arguments attacking $S$. When $S=\set{a}$ we write $a^+$ and $a^-$ respectively.
\item $U\subseteq A$ is the set of unattacked arguments, i.e. $a\in U\Leftrightarrow a^-=\es$.
\item A self-attacking argument $a$ satisfies $a\in a^+$.
\item An AF $\ang{A,R}$ is empty iff $A=\es$, trivial iff $R=\es$, finite iff $A$ is a finite set, and infinite iff $A$ is an infinite set.
\item An AF $\ang{A,R}$ is symmetric if $R$ is a non-empty symmetric relation. This does not exclude the possibility of there being self-attacking arguments.
\item An AF is finitary iff all arguments has finitely many attackers. All finite AFs are finitary, but finite AFs can have infinitely many arguments.
\item An induced argumentation sub-framework of $\ang{A,R}$ is the induced digraph on a set $B\subseteq A$.
\item An AF is cyclic iff it contains a directed cycle, else the AF is acyclic. A cycle can be odd or even depending on how many arguments it contains.
\item An AF is well-founded iff there exists no countably infinite backward chain of arguments. Cyclic AFs are not well-founded. Finite acyclic AFs are well-founded. Well-founded AFs satisfy $U\neq\es$.
\item An argument $a$ indirectly attacks an argument $b$ iff there is an odd-length path in $R$ from $a$ to $b$. An argument $a$ indirectly defends an argument $b$ iff there is an even-length path in $R$ from $a$ to $b$. We say $a$ is controversial w.r.t. $b$ iff $a$ both indirectly attacks and indirectly defends $b$. An argument $a$ is controversial iff there exists some argument $b$ w.r.t. which it is controversial. An AF is controversial iff there is some controversial argument, else it is uncontroversial.
\item An AF is limited controversial iff there is no countably infinite backward chain of controversial arguments. Uncontroversial AFs are trivially limited controversial. Limited controversial AFs have no odd cycles. A finite AF without any odd cycles is limited controversial.
\end{itemize}

\newpage

\section{Neutrality and Conflict-Freeness}

\subsection{The Neutrality Function}

\subsubsection{Definition}

\begin{definition}\label{def:neutrality_function}
Given an AF, its \textbf{neutrality function} is
\begin{align}
n:\pow\pair{A}\to&\pow\pair{A}\nonumber\\
S\mapsto& n(S):= A-S^+.
\end{align}
\end{definition}

\noindent The neutrality function of $S\subseteq A$ selects all arguments not attacked by $S$, i.e. $S$ is \textit{neutral} towards these arguments. If the underlying AF needs to be explicitly specified, we can write $n_{\mathcal{A}}$, for $\mathcal{A}=\ang{A,R}$.\footnote{This is denoted as $Pl_{AF}$ in \cite[Section 4.2]{Dung:95}.} From now, we will reserve the letter ``$n$'' to denote the neutrality function, and nothing else.\footnote{e.g. we will not use $n$ to denote arguments, or even indices in generalised operators such as unions and meets.}

\begin{corollary}
$n$ is well-defined as a function.
\end{corollary}
\begin{proof}
Given $S\in\pow\pair{A}$, $n(S)=A-S^+\in\pow\pair{A}$ is well-defined. Further, for $S=T$, $n(S)=n(T)$ by Corollary \ref{cor:pm_function} (page \pageref{cor:pm_function}). Therefore, $n$ is a well-defined function.
\end{proof}

\begin{example}\label{eg:nixon_neutrality}
(Example \ref{eg:nixon} continued, page \pageref{eg:nixon}) For the Nixon diamond, the values of $n$ are depicted in Table \ref{tab:nixon_neutrality}. Recall that in this case $A:=\set{a,b}$ and $R=\set{(a,b),(b,a)}$.

\begin{table}[H]
\begin{center}
\begin{tabular}{ | c | c | c | c | c | }
\hline
$S$ & $\es$ & $\set{a}$ & $\set{b}$ & $A$\\\hline
$n(S)$ & $A$ & $\set{a}$ & $\set{b}$ & $\es$\\\hline
\end{tabular}
\caption{The values of the neutrality function $n$, for Example \ref{eg:nixon}.}\label{tab:nixon_neutrality}
\end{center}
\end{table}
\end{example}

\begin{example}\label{eg:simple_neutrality}
(Example \ref{eg:simple_reinstatement} continued, page \pageref{eg:simple_reinstatement}) For simple reinstatement, the values of $n$ are depicted in Table \ref{tab:simple_neutrality}. Recall that in this case $A:=\set{a,b,c}$ and $R=\set{(c,b),(b,a)}$.

\begin{table}[H]
\begin{center}
\begin{tabular}{ | c | c | c | c | c | c | c | c | c | }
\hline
$S$ & $\es$ & $\set{a}$ & $\set{b}$ & $\set{c}$ & $\set{a,b}$ & $\set{b,c}$ & $\set{a,c}$ & $A$\\\hline
$n(S)$ & $A$ & $A$ & $\set{b,c}$ & $\set{a,c}$ & $\set{b,c}$ & $\set{c}$ & $\set{a,c}$ & $\set{c}$\\\hline
\end{tabular}
\caption{The values of the neutrality function $n$, for Example \ref{eg:simple_reinstatement}.}\label{tab:simple_neutrality}
\end{center}
\end{table}
\end{example}

\subsubsection{Properties}

For any AF, the following results hold.

\begin{corollary}\label{cor:neutrality_es}
We have that $n\pair{\es}=A$.
\end{corollary}
\begin{proof}
By Corollary \ref{cor:es_pm} (page \pageref{cor:es_pm}), $\es^+=\es$ so $n\pair{\es}=A-\es=A$.
\end{proof}

\begin{corollary}
We have that $n(A) = U$.
\end{corollary}
\begin{proof}
Let $a\in n(A)$ iff $a\in A-A^+$ iff $a\notin A^+$ iff $\pair{\forall b\in A}\neg R(b,a)$ iff $a^-=\es$. Therefore, $a\in U$ by Definition \ref{def:unattacked} (page \pageref{def:unattacked}).
\end{proof}

\begin{corollary}\label{cor:neutrality_antitone}
$n$ is $\subseteq$-antitone.
\end{corollary}
\begin{proof}
If $S\subseteq T\subseteq A$, then $S^+\subseteq T^+\subseteq A$ by Corollary \ref{cor:pm_monotonicity} (page \pageref{cor:pm_monotonicity}) and hence $A-T^+\subseteq A-S^+\subseteq A$. Therefore, $n(T)\subseteq n(S)$, so $n$ is $\subseteq$-antitone.
\end{proof}

\begin{corollary}
The square of the neutrality function, $n^2(S):=n\pair{n\pair{S}}$, is $\subseteq$-monotone.
\end{corollary}
\begin{proof}
This follows from the fact that the composition of an antitone function with itself results in a monotone function.
\end{proof}

\begin{corollary}\label{cor:neutrality_cup_cap}
Let $I$ be an index set and $\set{S_i}_{i\in I}$ be a family of subsets of $A$. We have that
\begin{align}
n\pair{\bigcup_{i\in I}S_i}=&\bigcap_{i\in I}n\pair{S_i}\text{ and }\label{eq:neut_of_cup}\\
n\pair{\bigcap_{i\in I}S_i}\supseteq& \bigcup_{i\in I}n\pair{S_i}.\label{eq:neut_of_cap}
\end{align}
The reverse containment of Equation \ref{eq:neut_of_cap} is not true in general.
\end{corollary}

\noindent Notice that for $I=\es$, the first equation reduces to Corollary \ref{cor:neutrality_es} and the second equation reduces to $n(A)\supseteq\es$, which is trivially true.

\begin{proof}
For the first result we apply Equation \ref{eq:plus_cup} (page \pageref{eq:plus_cap}) and De Morgan's laws.
\begin{align*}
n\pair{\bigcup_{i\in I}S_i}=A-\pair{\bigcup_{i\in I}S_i}^+=A-\bigcup_{i\in I}S_i^+=\bigcap_{i\in I}\pair{A-S_i^+}=\bigcap_{i\in I}n\pair{S_i}.
\end{align*}
For the second result we apply Equation \ref{eq:plus_cap} (page \pageref{eq:plus_cap}) and De Morgan's laws.
\begin{align*}
n\pair{\bigcap_{i\in I}S_i}=A-\pair{\bigcap_{i\in I}S_i}^+\supseteq A-\bigcap_{i\in I}S_i^+=\bigcup_{i\in I}\pair{A-S_i^+}=\bigcup_{i\in I}n\pair{S_i}.
\end{align*}
Now consider the AF whose underlying digraph is the same as that of Corollary \ref{cor:cup_cap_plus_minus}, depicted in Figure \ref{fig:cap_plus} (page \pageref{fig:cap_plus}). We have $n\pair{S_1\cap S_2}=n\pair{\set{b}}=\set{a,b,c,x}$. However, $n(S_1)=n(S_2)=\set{a,b,c}$ and hence $n\pair{S_1}\cup n\pair{S_2}=\set{a,b,c}$. Clearly, $n\pair{S_1\cap S_2}=\set{a,b,c,x}\not\subseteq n\pair{S_1}\cup n\pair{S_2}=\set{a,b,c}$. Therefore, the converse of the second result does not hold in general.
\end{proof}

\noindent Note that Equation \ref{eq:neut_of_cup} means that $n$ is a \textit{join antimorphism} \cite[Definition 3.3.21]{subrahmanyam2018elementary} on the complete lattice $\ang{\pow\pair{A},\cap,\cup}$ (see Appendix \ref{app:order}, page \pageref{app:order}).


\subsection{Conflict-Free Sets}\label{sec:CF}

We now use the neutrality function to define what it means for a set of arguments in an AF to be collectively consistent.

\subsubsection{Definition}

\begin{theorem}\label{thm:cf_char}
For $S\subseteq A$, TFAE:\footnote{\textit{TFAE} stands for ``the following (statements) are (logically) equivalent''.}
\begin{enumerate}
\item $S\subseteq n(S)$, i.e. $S$ is a postfixed point of $n$ (Definition \ref{def:prefix_postfix_fp}, page \pageref{def:prefix_postfix_fp}),
\item $S\cap S^+=\es$ and
\item $S^2\cap R=\es$ and
\item $S\cap S^-=\es$.
\end{enumerate}
\end{theorem}
\begin{proof}
(1) and (2) are equivalent because
\begin{align*}
S\cap S^+=\es\Leftrightarrow S\subseteq A-S^+\Leftrightarrow S\subseteq n(S).
\end{align*}

\noindent (1) and (3) are equivalent because
\begin{align*}
S\not\subseteq n(S)\Leftrightarrow& \pair{\exists a\in S}a\notin n(S)\Leftrightarrow \pair{\exists a\in S}a\in S^+\Leftrightarrow \pair{\exists a\in S}\pair{\exists b\in S}R(b,a)\\
\Leftrightarrow&\pair{\exists a,b\in S}R(b,a)\Leftrightarrow \pair{\exists (b,a)\in S^2} R(b,a)\Leftrightarrow S^2\cap R\neq\es.
\end{align*}

\noindent (3) and (4) are equivalent because

\begin{align*}
a\in S\cap S^-\Leftrightarrow \pair{\exists a,b\in S}R(a,b)\Leftrightarrow S^2\cap R\neq\es.
\end{align*}

\noindent This shows the result.
\end{proof}

\begin{definition}\label{definition:cf}
\cite[Definition 5]{Dung:95} A set $S\subseteq A$ is \textbf{conflict-free} (cf) iff $S$ satisfies any of the four equivalent conditions in Theorem \ref{thm:cf_char}.
\end{definition}

\noindent Intuitively, a cf set of arguments consists of arguments that do not disagree with (i.e. attack) each other. This denotes that the arguments are mutually consistent. Graph-theoretically, cf sets correspond to independent sets of the AF as a digraph.

\begin{example}
(Example \ref{eg:simple_neutrality} continued) From Table \ref{tab:simple_neutrality} (page \pageref{tab:simple_neutrality}), we can see that $\es$, $\set{a}$, $\set{b}$, $\set{c}$ and $\set{a,c}$ are all cf sets.
\end{example}

\begin{example}
(Example \ref{eg:fri} continued, page \pageref{eg:fri}) For floating reinstatement, we have
\begin{align}
CF = \set{\set{a}, \set{b}, \set{c}, \set{e}, \set{a,e}, \set{b,e}}.
\end{align}
\end{example}

\begin{example}
(Example \ref{eg:bi_inf} continued, page \pageref{eg:bi_inf} continued) For this AF, the cf sets are all subsets of $A$ that do not have $a_i$ and $b_i$ together for $i\in\integ$, because $R(b_i,a_i)$, and also all subsets that do not have $a_{i+1}$ and $b_i$ together, because $R(a_{i+1},b_i)$. This would include $\es$, all singleton sets (as no argument is self-attacking), all sets of two non-adjacent arguments, e.g. $\set{a_1,a_2}$ or $\set{a_5,b_6}$... etc.
\end{example}

\subsubsection{Existence}

\begin{definition}
Given an underlying AF, let $CF\subseteq\pow\pair{A}$ denote its set of cf sets.
\end{definition}

\noindent If the underlying AF $\mathcal{A}:=\ang{A,R}$ needs to be explicitly specified, then we write $CF\pair{\mathcal{A}}$, or $CF\pair{\ang{A,R}}$.

\begin{corollary}\label{cor:es_cf}
$\es\in CF$.
\end{corollary}
\begin{proof}
Trivially, $\es\subseteq n\pair{\es}=A$ by Corollary \ref{cor:neutrality_es} (page \pageref{cor:neutrality_es}).
\end{proof}

\noindent Therefore, for any AF, cf sets always exist; as $\es$ is cf, so $CF\neq\es$. Further:

\begin{corollary}\label{cor:U_cf}
$U\in CF$.
\end{corollary}
\begin{proof}
As $U^-=\es$, we have $U\cap U^-=\es$ and hence $U\in CF$ by Theorem \ref{thm:cf_char}.
\end{proof}

\subsubsection{Lattice-Theoretic Structure}

\begin{corollary}\label{cor:CF_down_closed}
$CF$ is $\subseteq$-downward closed.
\end{corollary}
\begin{proof}
Assume $S\in CF$ and $T\subseteq S$. As $S\subseteq n(S)$, then $T\subseteq S\subseteq n(S)\subseteq n(T)$ by Corollary \ref{cor:neutrality_antitone}. Therefore, $T\subseteq n(T)$ and hence $T\in CF$.
\end{proof}

\begin{corollary}
If $S\subseteq U$ then $S\in CF$. The converse is not true.
\end{corollary}
\begin{proof}
This follows from Corollaries \ref{cor:U_cf} and \ref{cor:CF_down_closed}.

The converse is not true, e.g. Example \ref{eg:simple_reinstatement} (page \pageref{eg:simple_reinstatement}) where $U=\set{c}$ and $\set{a,c}\in CF$, $U=\set{c}$ and $\set{a,c}\not\subseteq\set{c}$.
\end{proof}

\begin{corollary}\label{cor:non_CF_up_closed}
If $S\notin CF$ and $S\subseteq T$, then $T\notin CF$.
\end{corollary}
\begin{proof}
Immediate by taking the contrapositive of Corollary \ref{cor:CF_down_closed}.
\end{proof}

\begin{corollary}\label{cor:cf_cap_closed}
$CF$ is closed under arbitrary non-empty intersections.
\end{corollary}
\begin{proof}
Let $I\neq\es$ be an index set. Let $\set{S_i}_{i\in I}$ be a family of cf sets, then for all $i\in I$, $S_i\subseteq n\pair{S_i}$. Applying Corollaries \ref{cor:neutrality_antitone} and \ref{cor:neutrality_cup_cap} (page \pageref{cor:neutrality_cup_cap}), starting from $\pair{\forall i\in I}S_i\subseteq n\pair{S_i}$,
\begin{align}
\bigcap_{i\in I}S_i\subseteq\bigcap_{i\in I}n\pair{S_i}=n\pair{\bigcup_{i\in I}S_i}\subseteq n\pair{\bigcap_{i\in I}S_i}.
\end{align}
This shows that $\bigcap_{i\in I}S_i$ is also a cf set. The result follows.
\end{proof}

\noindent Note that if the intersection is over the empty family of cf sets, we get $A$, which is not in general cf unless the AF is trivial.

\begin{corollary}\label{cor:CF_no_cup}
$CF$ is not in general closed under unions.
\end{corollary}
\begin{proof}
Consider the AF $\ang{\set{a,b},\set{\pair{a,b}}}$. We depict this in Figure \ref{fig:dyad}.

\begin{figure}[H]
\begin{center}
\begin{tikzpicture}[>=stealth',shorten >=1pt,node distance=2cm,on grid,initial/.style    ={}]
\tikzset{mystyle/.style={->,relative=false,in=0,out=0}};
\node[state] (a) at (0,0) {$ a $};
\node[state] (b) at (2,0) {$ b $};
\draw [->] (a) to (b);
\end{tikzpicture}
\caption{The AF from Corollary \ref{cor:CF_no_cup}.}\label{fig:dyad}
\end{center}
\end{figure}
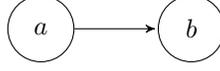

\noindent Clearly $n\pair{\set{a}}=\set{a}$, $n\pair{\set{b}}=\set{a,b}$ and $n\pair{\set{a,b}}=\set{a}$. Therefore, $\set{a}$ and $\set{b}$ are cf sets, but $\set{a}\cup\set{b}=\set{a,b}\not\subseteq n\pair{\set{a,b}}=\set{a}$ is not a cf set.
\end{proof}

\noindent We now give increasingly stronger completeness results for $CF$. Refer to Appendix \ref{app:order} (page \pageref{app:order}) for the definitions.

\begin{theorem}\label{thm:cf_omega_complete}
$\ang{CF,\subseteq}$ is $\omega$-complete.
\end{theorem}
\begin{proof}
Let $\set{S_i}_{i\in\nat}$ be an ascending $\omega$-chain in $CF$. Let $S:=\bigcup_{i\in\nat} S_i$. Assume for contradiction that $S^2\cap R\neq\es$. Then there are $a,b\in S$, $R(a,b)$. Therefore, $a\in S_i$ and $b\in S_j$ by definition of $S$, for some $i,j\in\nat$. As $\set{S_i}_{i\in\nat}$ is an ascending chain, WLOG assume $j\geq i$ hence $a,b\in S_j$ and hence $S_j^2\cap R\neq\es$, so $S_j\notin CF$ -- contradiction as $\set{S_i}_{i\in\nat}$ is an ascending chain in $CF$. Therefore, $\bigcup_{i\in\nat}S_i\in CF$ and hence $\ang{CF,\subseteq}$ is $\omega$-complete.
\end{proof}

\begin{corollary}\label{cor:cf_chain_complete}
$\ang{CF,\subseteq}$ is chain complete.
\end{corollary}
\begin{proof}
Let $\mathcal{C}$ be an ascending chain in $CF$ of arbitrary cardinality. Let $C:=\bigcup\mathcal{C}$. Assume for contradiction that there are $a,b\in C$ such that $R(a,b)$. There exists $A,B\in\mathcal{C}$ such that $a\in A$ and $b\in B$. As $\mathcal{C}$ is a chain, WLOG let $A\subseteq B$ so $a,b\in B$. Therefore, $B^2\cap R\neq\es$, meaning that $B\notin CF$ -- contradiction. Therefore, $\bigcup\mathcal{C}\in CF$ for all chains $\mathcal{C}$. Therefore $CF$ is chain complete.
\end{proof}

\begin{corollary}\label{cor:CF_di_comp}
$\ang{CF,\subseteq}$ is directed complete.
\end{corollary}
\begin{proof}
Let $\mathcal{D}$ be a directed set in $CF$. Let $D:=\bigcup\mathcal{D}$. Assume for contradiction that there are $a,b\in D$ such that $R(a,b)$. There exists $A,B\in\mathcal{D}$ such that $a\in A$ and $b\in B$. As $\mathcal{D}$ is a directed set, WLOG let $A,B\subseteq C$ for some $C\in\mathcal{D}$, so $a,b\in C$. Therefore, $C^2\cap R\neq\es$, meaning that $C\notin CF$ -- contradiction. Therefore, $\bigcup\mathcal{D}\in CF$ for all directed sets $\mathcal{D}$. Therefore $CF$ is directed complete.
\end{proof}

\begin{theorem}
$\ang{CF,\subseteq}$ is a complete semilattice (Definition \ref{def:complete_SL}, page \pageref{def:complete_SL}).
\end{theorem}
\begin{proof}
Every non-empty subset of $\ang{CF,\subseteq}$ has an infimum by Corollary \ref{cor:cf_cap_closed}, which is calculated by set-theoretic intersection. Further, $\ang{CF,\subseteq}$ is chain complete by Corollary \ref{cor:cf_chain_complete}.
\end{proof}

The neutrality function is not closed on $CF$.

\begin{corollary}\label{cor:n_not_closed_on_CF}
If $S\in CF$ then it is not generally true that $n(S)\in CF$.
\end{corollary}
\begin{proof}
Consider the AF $A=\set{a,b,c,e}$ with $R=\set{(c,e)}$ only, and $S=\set{a,b}$. This AF is depicted in Figure \ref{fig:n_not_closed_on_CF}.

\begin{figure}[H]
\begin{center}
\begin{tikzpicture}[>=stealth',shorten >=1pt,node distance=2cm,on grid,initial/.style    ={}]
\tikzset{mystyle/.style={->,relative=false,in=0,out=0}};
\node[state] (a) at (0,0) {$ a $};
\node[state] (b) at (2,0) {$ b $};
\node[state] (c) at (4,0) {$ c $};
\node[state] (e) at (6,0) {$ e $};
\draw [->] (c) to (e);
\end{tikzpicture}
\end{center}
\caption{The AF from Corollary \ref{cor:n_not_closed_on_CF}.}\label{fig:n_not_closed_on_CF}
\end{figure}
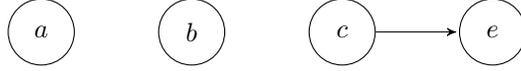

\noindent Clearly, $n(S)=A$ and is not cf.\footnote{Another example would be $\es\in CF$ and $n\pair{\es}=A$. Unless the underlying AF is trivial, $A\notin CF$.}
\end{proof}

\subsection{Naive Extensions}\label{sec:NAI}

We now begin to consider what it means for a set of arguments to be justified or winning. Let $\ang{A,R}$ denote an AF. One natural choice of justified arguments would be the $\subseteq$-maximal $CF$ subsets of this AF. This is a natural choice because as $R$ denotes inconsistency, consistent sets are by analogy cf sets, and $\subseteq$-maximal such sets are akin to maximal consistent subsets, which in logic is one way of drawing sensible inferences from an inconsistent set of formulae (e.g. \cite{Croitoru:13}).


\subsubsection{Definition}

\begin{definition}
The set $S\subseteq A$ is a \textbf{naive extension} iff $S\in\max_{\subseteq}CF$.
\end{definition}

\begin{example}
(Example \ref{eg:fri} continued, page \pageref{eg:fri}) As $$CF=\set{\set{a},\set{b},\set{c},\set{e},\set{a,e},\set{b,e}},$$ the naive extensions are  $\set{c}$, $\set{a,e}$ and $\set{b,e}$.
\end{example}

\begin{example}
(Example \ref{eg:bi_inf} continued, page \pageref{eg:bi_inf}) Among the $\subseteq$-maximal cf sets are $S_a:=\set{a_i}_{i\in\integ}$ and $S_b:=\set{b_i}_{i\in\integ}$.
\end{example}

\subsubsection{Existence and Lattice-Theoretic Structure}

\begin{definition}
We denote the set of all naive extensions of an AF $\ang{A,R}$ as $NAI\subseteq\pow\pair{A}$, or $NAI(\ang{A,R})$ if we need to make the underlying AF explicit.
\end{definition}

\noindent Clearly, $NAI=\max_{\subseteq}CF$. We show that every AF has naive extensions.

\begin{theorem}\label{thm:existence_of_naive}
$NAI\neq\es$.
\end{theorem}
\begin{proof}
By Corollary \ref{cor:cf_chain_complete} (page \pageref{cor:cf_chain_complete}), every chain $\mathcal{C}$ in $CF$ has an upper bound $\bigcup\mathcal{C}\in CF$. Therefore, by Zorn's lemma, $CF$ has at least one $\subseteq$-maximal element. The result follows.
\end{proof}

\noindent Notice in the proof of Theorem \ref{thm:existence_of_naive}, we have used Zorn's lemma, which is equivalent to the axiom of choice. This means that the axiom of choice is \textit{sufficient} to demonstrate that naive extensions exist for all AFs. It is also reasonable to ask whether it is \textit{necessary}. We address this in Appendix \ref{app:AC} (page \pageref{app:AC}).

Clearly, $NAI$ is not a singleton set, i.e. the naive extension does not have to be unique.

\begin{example}\label{eg:simp_re_nai}
(Example \ref{eg:simple_reinstatement} continued, page \pageref{eg:simple_reinstatement}) Clearly, $NAI=\set{\set{a,c},\set{b}}$.
\end{example}

\begin{example}\label{eg:nixon_nai}
(Example \ref{eg:nixon} continued, page \pageref{eg:nixon}) As $CF=\set{\es,\set{a},\set{b}}$, the naive extensions are $\set{a}$ and $\set{b}$.
\end{example}


\noindent Unlike $CF$, the lattice-theoretic structure of $NAI$ is trivial.

\begin{corollary}
$\ang{NAI,\subseteq}\subseteq\ang{\pow\pair{A},\subseteq}$ is a $\subseteq$-antichain.
\end{corollary}
\begin{proof}
If $NAI$ is singleton, then it is trivially an antichain. Else, as $NAI\subseteq CF$, let $S_1,S_2\in NAI$ be distinct. Then $S_1\not\subseteq S_2$ and $S_2\not\subseteq S_1$, by maximality of $NAI$ in $CF$. Therefore, $NAI$ is a $\subseteq$-antichain.
\end{proof}

\subsubsection{Criticism of Naive Extensions}\label{sec:nai_bad}

The naive extensions provide one suggestion of what a set of winning arguments should be by using the analogy from logic of drawing inferences from maximal consistent subsets of an otherwise inconsistent set of propositions. Graph-theoretically, these are the $\subseteq$-maximal independent sets. However, this does not seem like a sensible suggestion for the sets of winning arguments. While in examples such as Example \ref{eg:nixon_nai}, $NAI$ seems sensible in giving $\set{a}$ and $\set{b}$ as sets of winning arguments, examples such as Example \ref{eg:simp_re_nai} gives $\set{a,c}$ and $\set{b}$ as naive extensions. In this case, although $\set{a,c}$ seems sensible as a set of winning arguments because $c$ is unattacked, so it defeats $b$, which means that $a$ is no longer defeated by $b$ and hence $a$ should be justified, the naive semantics also suggest that $\set{b}$ should be winning, which does not seem to make sense as $b$ is defeated by $c$, which is undefeated. It is counter-intuitive examples such as this that discourage people from using the naive semantics as a way of defining winning arguments.

Despite this, the naive semantics are simple to understand, and motivate questions such as existence, uniqueness and lattice-theoretic structure that we will consider when investigating other notions of winning arguments. Further, the naive extensions serve as a useful intermediate concept when investing other extensions, which we will use in Sections \ref{sec:naive_and_preferred} (page \pageref{sec:naive_and_preferred}) and \ref{sec:stab_nai_pref} (page \pageref{sec:stab_nai_pref}).

\subsection{Summary}

\begin{itemize}
\item Given an AF, its neutrality function is $n:\pow\pair{A}\to\pow\pair{A}$, $n(S)=A-S^+$.
\item $n$ satisfies: $n\pair{\es}=A$, $n\pair{A}=U$, $n$ is $\subseteq$-antitone, and for any index set $I$ and family of subsets $\set{S_i}_{i\in I}$ of $A$,
\begin{align}
n\pair{\bigcup_{i\in I}S_i}=\bigcap_{i\in I}n\pair{S_i}\text{ and }n\pair{\bigcap_{i\in I}S_i}\supseteq \bigcup_{i\in I}n\pair{S_i}.
\end{align}
\item We say $S\subseteq A$ is conflict-free iff $S\subseteq n\pair{S}$. The set of all conflict-free sets of a given AF $\mathcal{A}$ is $CF(\mathcal{A})$ or just $CF$.
\item For all AFs, $\es, U\in CF$. Furthermore, $\ang{CF,\subseteq}$ is $\subseteq$-downward closed, is not closed under unions, and is a complete semilattice that is also directed complete. Furthermore, the neutrality function $n$ is not closed on $CF$.
\item We may consider the naive extensions, where the set of all naive extensions is $NAI=\max_{\subseteq}CF$, which is a $\subseteq$-antichain. Further, $NAI\neq\es$, and generally not singleton. Unfortunately, $NAI$ should not define when arguments win.
\end{itemize}

\newpage

\section{Defence}

\subsection{The Defence Function}\label{sec:defence_function}

\subsubsection{Definition}

The defence function formalises how a set of arguments can defend another argument.

\begin{definition}\label{def:acceptable}
\cite[Definition 6(1)]{Dung:95} Given $S\subseteq A$ and $a\in A$, $a^-\subseteq S^+$ iff
\begin{itemize}
\item $a$ is \textbf{acceptable} w.r.t. $S$,
\item $S$ \textbf{defends} $a$,
\item $S$ \textbf{reinstates} $a$ \cite{Caminada:08}.
\end{itemize}
All three terms are equivalent.
\end{definition}

\begin{example}
(Example \ref{eg:simple_reinstatement}, page \pageref{eg:simple_reinstatement} continued) In simple reinstatement, we have $\set{c}$ reinstating $a$.
\end{example}

\begin{example}
(Example \ref{eg:double_reinstatement}, page \pageref{eg:double_reinstatement} continued) In double reinstatement, we have $\set{c}$, $\set{e}$ and $\set{c,e}$ reinstating $a$.
\end{example}

\begin{example}\label{eg:Caminada_ex2}
\cite[Exercise 2]{Caminada:08} Consider the following AF \cite[Figure 4]{Caminada:08}, depicted in Figure \ref{fig:Caminada_ex2}.

\begin{figure}[H]
\begin{center}
\begin{tikzpicture}[>=stealth',shorten >=1pt,node distance=2cm,on grid,initial/.style    ={}]
\tikzset{mystyle/.style={->,relative=false,in=0,out=0}};
\node[state] (a) at (0,0) {$ c $};
\node[state] (b) at (-2,0) {$ b $};
\node[state] (c) at (-4,0) {$ a $};
\node[state] (d) at (2,0) {$ e $};
\draw [->] (a) to [out = 225, in = -45] (b);
\draw [->] (b) to [out = 45, in = 135] (a);
\draw [->] (c) to (b);
\draw [->] (a) to (d);
\end{tikzpicture}
\caption{The AF depicting \cite[Figure 4]{Caminada:08}, from Example \ref{eg:Caminada_ex2}.}\label{fig:Caminada_ex2}
\end{center}
\end{figure}
\begin{enumerate}
\item Does $\set{a}$ defend $c$? Yes, because $c^-=\set{b}$ and $\set{a}^+=a^+=\set{b}$, therefore $c^-\subseteq\set{a}^+$.
\item Does $\set{c}$ defend $c$? Yes, because $c^-=\set{b}$ and $\set{c}^+=c^+=\set{b,e}$, therefore $c^-\subseteq\set{c}^+$.
\item Does $\set{b}$ defend $c$? No, because $c^-=\set{b}$ and $\set{b}^+=b^+=\set{c}$, and $c^-\not\subseteq\set{b}^+$.
\end{enumerate}
\end{example}

In order for a set of arguments $S$ to defend $a$, $S$ attacks all of the attackers of $a$. This motivates the following function:

\begin{definition}\label{def:defence_function}
\cite[Definition 16]{Dung:95} Given an AF, its \textbf{defence function}\footnote{This is called the \textbf{characteristic function} in \cite{Dung:95} and denoted $F$.} is
\begin{align}
d:\pow\pair{A}\to&\pow\pair{A}\nonumber\\
S\mapsto& d(S):=\set{a\in A\:\vline\:a^-\subseteq S^+}.
\end{align}
\end{definition}

\noindent If the underlying AF $\mathcal{A}=\ang{A,R}$ needs to be explicitly specified, then we can write $d_{\mathcal{A}}$ \cite[Remark 17]{Dung:95}. From now, we will reserve the letter $d$ to denote the defence function only.

\begin{corollary}
$d:\pow\pair{A}\to\pow\pair{A}$ is a well-defined function.
\end{corollary}
\begin{proof}
The set $d(S):=\set{a\in A\:\vline\:a^-\subseteq S^+}$ is well-defined. Further, for $S=T$, $d(S)=d(T)$ by Corollary \ref{cor:pm_function} (page \pageref{cor:pm_function}). Therefore, $d$ is a well-defined function.
\end{proof}

\begin{corollary}
$a\in A$ is acceptable w.r.t. $S\subseteq A$ iff $a\in d(S)$.
\end{corollary}
\begin{proof}
This follows immediately from Definitions \ref{def:acceptable} and \ref{def:defence_function}.
\end{proof}

\begin{example}\label{eg:nixon_defence}
(Example \ref{eg:nixon} continued, page \pageref{eg:nixon}) For the Nixon diamond, the values of $d$ are depicted in Table \ref{tab:nixon_defence}. Recall that in this case $A:=\set{a,b}$ and $R=\set{(a,b),(b,a)}$.

\begin{table}[H]
\begin{center}
\begin{tabular}{ | c | c | c | c | c | }
\hline
$S$ & $\es$ & $\set{a}$ & $\set{b}$ & $A$\\\hline
$d(S)$ & $\es$ & $\set{a}$ & $\set{b}$ & $A$\\\hline
\end{tabular}
\caption{The values of the neutrality function $n$, for Example \ref{eg:nixon}.}\label{tab:nixon_defence}
\end{center}
\end{table}
\end{example}

\begin{example}\label{eg:simple_defence}
(Example \ref{eg:simple_reinstatement} continued, page \pageref{eg:simple_reinstatement}) For simple reinstatement, the values of $d$ are depicted in Table \ref{tab:simple_defence}. Recall that in this case $A:=\set{a,b,c}$.

\begin{table}[H]
\begin{center}
\begin{tabular}{ | c | c | c | c | c | c | c | c | c | }
\hline
$S$ & $\es$ & $\set{a}$ & $\set{b}$ & $\set{c}$ & $\set{a,b}$ & $\set{b,c}$ & $\set{a,c}$ & $A$\\\hline
$d(S)$ & $\set{c}$ & $\set{c}$ & $\set{c}$ & $\set{a,c}$ & $\set{c}$ & $\set{a,c}$ & $\set{a,c}$ & $\set{a,c}$\\\hline
\end{tabular}
\caption{The values of the defence function $d$ for Example \ref{eg:simple_defence}.}\label{tab:simple_defence}
\end{center}
\end{table}
\end{example}

\begin{example}\label{eg:loop_defence}
Consider the AF $A=\set{a,b}$ and $R=\set{(a,a),(a,b),(b,a)}$. Notice that this AF is symmetric (Definition \ref{def:symmetric}, page \pageref{def:symmetric}). We depict this AF in Figure \ref{fig:loop_defence}.

\begin{figure}[H]
\begin{center}
\begin{tikzpicture}[>=stealth',shorten >=1pt,node distance=2cm,on grid,initial/.style    ={}]
\tikzset{mystyle/.style={->,relative=false,in=0,out=0}};
\node[state] (a) at (0,0) {$ a $};
\node[state] (b) at (2,0) {$ b $};
\draw [->] (a) to [out = 45, in = 135] (b);
\draw [->] (b) to [out = 225, in = -45] (a);
\draw [->] (a) to [out=135,in=180] ($(a) + (0,1)$) to [out = 0, in = 45] (a);
\end{tikzpicture}
\end{center}
\caption{The AF from Example \ref{eg:loop_defence}.}\label{fig:loop_defence}
\end{figure}

\noindent The values of $d$ are depicted in Table \ref{tab:loop_defence}. Recall that in this case $A:=\set{a,b}$.

\begin{table} [H]
\begin{center}
\begin{tabular}{ | c | c | c | c | c | }
\hline
$S$ & $\es$ & $\set{a}$ & $\set{b}$ & $\set{a,b}$\\\hline
$d(S)$ & $\es$ & $\set{a,b}$ & $\set{b}$ & $\set{a,b}$\\\hline
\end{tabular}
\caption{The values of the defence function $d$ for Example \ref{eg:loop_defence}.}\label{tab:loop_defence}
\end{center}
\end{table}
\end{example}

\begin{example}\label{eg:non-finitary_AF2}
(Example \ref{eg:non-finitary_AF}, page \pageref{eg:non-finitary_AF}) Consider the non-finitary AF $A=\set{a}\cup\set{b_i}_{i\in\nat}$ and $R=\set{\pair{b_i,a}}_{i\in\nat}$. Let $S\subseteq A$. If $S=\es$ or $S=\set{a}$, then in both cases $S^+=\es$, hence $d(S)=\set{x\in A\:\vline\: x^-\subseteq\es}=\set{b_i}_{i\in\nat}=U$, the set of unattacked arguments.

If $S$ is neither empty nor only $\set{a}$, then $S^+=\set{a}$. Therefore, $d(S)=\set{x\in A\:\vline\:x^-\subseteq \set{a}}=\set{x\in A\:\vline\:x^-=\es\text{ or }x^-=\set{a}}$. However, for all $x\in A$, $x^-\neq\set{a}$. Therefore, $d(S)=U$ as well.

In summary, for the AF depicted in Example \ref{eg:non-finitary_AF}, $d(S)$ is a constant function, equal to $U$, the set of all unattacked arguments.




\end{example}

\begin{example}\label{eg:Caminada_ex3}
\cite[Exercise 3]{Caminada:08} Consider the following AF \cite[Figure 7]{Caminada:08}, depicted in Figure \ref{fig:Caminada_ex3}.

\begin{figure}[H]
\begin{center}
\begin{tikzpicture}[>=stealth',shorten >=1pt,node distance=2cm,on grid,initial/.style    ={}]
\tikzset{mystyle/.style={->,relative=false,in=0,out=0}};
\node[state] (b) at (0,0) {$ b $};
\node[state] (a) at (-2,0) {$ a $};
\node[state] (c) at (2,0) {$ c $};
\node[state] (e) at (4,1.5) {$ e $};
\node[state] (f) at (4,-1.5) {$ f $};
\draw [->] (b) to [out = 225, in = -45] (a);
\draw [->] (a) to [out = 45, in = 135] (b);
\draw [->] (b) to (c);
\draw [->] (c) to (e);
\draw [->] (e) to (f);
\draw [->] (f) to (c);
\end{tikzpicture}
\caption{The AF depicting \cite[Figure 7]{Caminada:08}, from Example \ref{eg:Caminada_ex3}.}\label{fig:Caminada_ex3}
\end{center}
\end{figure}
\noindent We can calculate:
\begin{enumerate}
\item $d(\set{a})=\set{x\in A\:\vline\:x^-\subseteq\set{a}^+}=\set{x\in A\:\vline\:x^-\subseteq\set{b}}=\set{a}$, because only $a^-=\set{b}$, while $c^-=\set{b,f}\not\subseteq\set{b}$. Therefore, $d\pair{\set{a}}=\set{a}$.
\item $d\pair{\set{b}}=\set{x\in A\:\vline\:x^-\subseteq\set{b}^+}=\set{x\in A\:\vline\:x^-\subseteq\set{a,c}}=\set{b,e}$. Therefore, $d\pair{\set{b}}=\set{b,e}$.
\item $d\pair{\set{b,e}}=\set{x\in A\:\vline\:x^-\subseteq\set{b,e}^+}=\set{x\in A\:\vline\:x^-\subseteq\set{a,c,f}}=\set{b,e}$, as $c^-=\set{b,f}\not\subseteq\set{a,c,f}$. Therefore, $d\pair{\set{b,e}}=\set{b,e}$.
\end{enumerate}
\end{example}

\subsubsection{Properties}

\begin{corollary}\label{cor:d_monotone}
\cite[Lemma 19]{Dung:95} The defence function is $\subseteq$-monotone
\end{corollary}
\begin{proof}
Assume $S\subseteq T\subseteq A$. If $a\in d(S)$, then $a^-\subseteq S^+\subseteq T^+$ by Definition \ref{def:defence_function} and Corollary \ref{cor:pm_monotonicity} (page \pageref{cor:pm_monotonicity}), and hence $a^-\subseteq T^+$. Therefore, $a\in d(T)$. As $a$ is arbitrary, $d(S)\subseteq d(T)$.
\end{proof}

\begin{definition}\label{def:d_fp}
Let $F_d:=\set{S\subseteq A\:\vline\:d(S)=S}$ be \textbf{the set of fixed points of $d$}.
\end{definition}

\begin{corollary}
$\ang{F_d,\subseteq}$ is a complete lattice.
\end{corollary}
\begin{proof}
Clearly, $\ang{\pow\pair{A},\subseteq}$ is a complete lattice and $d:\pow\pair{A}\to\pow\pair{A}$ is a $\subseteq$-monotone by Corollary \ref{cor:d_monotone}. By the Knaster-Tarski theorem (Theorem \ref{thm:KT}, page \pageref{thm:KT}), the result follows.
\end{proof}

\begin{corollary}
There exists a fixed point of $d$.
\end{corollary}
\begin{proof}
Immediate as $F_d$ is a complete lattice, so $F_d\neq\es$. 
\end{proof}

\noindent Further, as complete lattices are bounded, $d$ will have a least fixed point and a greatest fixed point.

\begin{corollary}\label{cor:source_not_def}
For all $a\in A$, $a^-=\es$ iff $a\in d\pair{\es}$.
\end{corollary}
\begin{proof}
Recalling Corollary \ref{cor:es_pm} (page \pageref{cor:es_pm}), we have that $a\in d\pair{\es}\Leftrightarrow a^-\subseteq\es^+\Leftrightarrow a^-\subseteq\es\Leftrightarrow a^-=\es$.
\end{proof}

\begin{corollary}\label{cor:unattacked_d_es}
$U=d\pair{\es}$.
\end{corollary}
\begin{proof}
Immediate from the Corollary \ref{cor:source_not_def} and Definition \ref{def:unattacked} (page \pageref{def:unattacked}).
\end{proof}

\noindent This means that the unattacked arguments do not have to be defended by anything.

\begin{corollary}\label{cor:iterate_d_get_U}
For any $S\subseteq A$, $U\subseteq d\pair{S}$.
\end{corollary}
\begin{proof}
Clearly $\es\subseteq S$ and hence $d\pair{\es}=U\subseteq d\pair{S}$ by Corollary \ref{cor:d_monotone}.
\end{proof}

\noindent Intuitively, the unattacked arguments are amongst all defended arguments, because they do not need to be defended by anything.

The set of arguments defended by the set of all arguments $A$ are exactly those arguments who are not attacked by an undefeated argument.

\begin{corollary}
We have that $d(A)=\set{a\in A\:\vline\:a^-\cap U=\es}$.
\end{corollary}
\begin{proof}
We have that $a\in d(A)$ iff $a^-\subseteq A^+$ iff $\pair{\forall b\in a^-}b\in A^+$ iff $\pair{\forall b\in a^-}b\notin U$ (by Corollary \ref{cor:unattacked_characterisation}, page \pageref{cor:unattacked_characterisation}), iff $a^-\cap U=\es$.
\end{proof}

\begin{corollary}\label{cor:d_cup_cap}
We have that
\begin{align}
\bigcup_{i\in I}d\pair{S_i}\subseteq d\pair{\bigcup_{i\in I}S_i}\text{ and }\bigcap_{i\in I}d\pair{S_i}\supseteq d\pair{\bigcap_{i\in I}S_i},
\end{align}
where in both cases the reverse containment may not be true.
\end{corollary}

\noindent Notice if $I=\es$, Equation \ref{cor:d_cup_cap} reduces to $\es\subseteq d\pair{\es}$ and $A\supseteq d\pair{A}$, both of which are trivially true.

\begin{proof}
For the first result:
\begin{align*}
a\in \bigcup_{i\in I} d\pair{S_i}\Leftrightarrow\pair{\exists i\in I}a^-\subseteq S_i^+\Rightarrow a^-\subseteq\bigcup_{i\in I}S_i^+\Leftrightarrow a\in d\pair{\bigcup_{i\in I}S_i}.
\end{align*}
The converse to the first result does not hold in general. Consider the argument framework $\ang{\set{a,b,c,e,f},\set{\pair{f,b},\pair{e,c},\pair{b,a},\pair{c,a}}}$. This AF is depicted in Figure \ref{fig:d_cup_cap1}.

\begin{figure}[H]
\begin{center}
\begin{tikzpicture}[>=stealth',shorten >=1pt,node distance=2cm,on grid,initial/.style    ={}]
\tikzset{mystyle/.style={->,relative=false,in=0,out=0}};
\node[state] (a) at (0,0) {$ a $};
\node[state] (b) at (2,1) {$ b $};
\node[state] (c) at (2,-1) {$ c $};
\node[state] (e) at (4,-1) {$ e $};
\node[state] (f) at (4,1) {$ f $};
\draw [->] (b) to (a);
\draw [->] (c) to (a);
\draw [->] (f) to (b);
\draw [->] (e) to (c);
\end{tikzpicture}
\end{center}
\caption{An AF that shows the converse to the first result of Corollary \ref{cor:d_cup_cap} is false.}\label{fig:d_cup_cap1}
\end{figure}
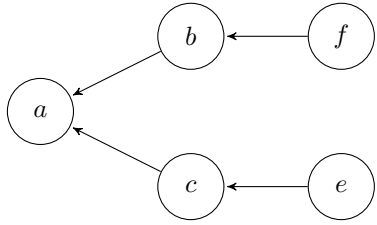

Let $S_1=\set{f}$ and $S_2=\set{e}$. We have $d\pair{S_1\cup S_2}=d\pair{\set{e,f}}=\set{f,e,a}$. However, $d\pair{S_1}=\set{e,f}$ and $d\pair{S_2}=\set{e,f}$, as $a^-=\set{b,c}$. Therefore, $d\pair{S_1}\cup d\pair{S_2}=\set{e,f}$ while $d\pair{S_1\cup S_2}=\set{f,e,a}$, so $d\pair{S_1\cup S_2}\not\subseteq d\pair{S_1}\cup d\pair{S_2}$.

For the second result we apply Equation \ref{eq:plus_cap} (page \pageref{eq:plus_cap}):
\begin{align*}
&a\in d\pair{\bigcap_{i\in I}S_i}\Leftrightarrow a^-\subseteq\pair{\bigcap_{i\in I}S_i}^+\subseteq\bigcap_{i\in I}S_i^+\Rightarrow\pair{\forall i\in I}a^-\subseteq S^+_i\\
\Leftrightarrow&\pair{\forall i\in I}a\in d\pair{S_i}\Leftrightarrow a\in\bigcap_{i\in I}d\pair{S_i}.
\end{align*}
The converse to the second result does not hold in general. Recall the AF from Example \ref{eg:double_reinstatement} (page \pageref{eg:double_reinstatement}), depicted in Figure \ref{fig:double_reinstatement}. Let $S_1=\set{c}$ and $S_2=\set{e}$. We have that $d\pair{S_1}=d\pair{S_2}=\set{a,c,e}$ hence $d\pair{S_1}\cap d\pair{S_2}=\set{a,c,e}$. However, $d\pair{S_1\cap S_2}=d\pair{\es}=\set{c,e}$. Clearly, $\set{c,e,a}\not\subseteq\set{c,e}$ and hence $d\pair{S_1}\cap d\pair{S_2}\not\subseteq d\pair{S_1\cap S_2}$.
\end{proof}

\begin{theorem}
\cite[Lemma 28]{Dung:95} If $\ang{A,R}$ is finitary, then $d$ is $\omega$-continuous (Definition \ref{def:omega_cts}, page \pageref{def:omega_cts}). Else, $d$ may or may not be $\omega$-continuous.
\end{theorem}
\begin{proof}
Let $\set{S_i}_{i\in\nat}$ be an $\omega$-chain in $\pow\pair{A}$ with limit $S:=\bigcup_{i\in\nat} S_i$, where $i<j$ implies $S_i\subseteq S_j$. Assume $\ang{A,R}$ is finitary. Let $a\in d(S)$, then as $a^-$ is finite, let $a^-:=\set{b_1,b_2,\ldots,b_m}$. For each such $b_j\in a^-$, we have $b_j\in S^+$, which means for each $1\leq j\leq m$ there is some $b_j\in S^+_{i_j}$. Let $k:=\max\set{i_1,i_2,\ldots,i_m}$, so $a^-\subseteq S^+_k$, because the $S_i$'s form a chain. Therefore, $\pair{\exists k\in\nat}a\in d\pair{S_k}$ and hence $a\in\bigcup_{k\in\nat}d\pair{S_k}$. As $a$ is arbitrary,
\begin{align}
d(S)=d\pair{\bigcup_{i\in\nat}S_i}\subseteq\bigcup_{i\in\nat}d\pair{S_i},
\end{align}

\noindent where we have applied Corollary \ref{cor:d_cup_cap} by choosing $I=\nat$. By instantiating Definition \ref{def:omega_cts} (page \pageref{def:omega_cts}), we conclude that $d$ is $\omega$-continuous.

The following example is a non-finitary AF where $d$ is not $\omega$-continuous. Consider Example \ref{eg:non_finitary2} (page \pageref{eg:non_finitary2}). This is not a finitary AF because $a^-$ is a countably infinite set. For $i\in\nat$, let $S_i:=\set{c_j}_{j=0}^i=\set{c_0,c_1,\ldots,c_i}$. Clearly $\set{S_i}_{i\in\nat}$ is an $\omega$-chain in $\pow\pair{A}$, with limit $U=\bigcup_{i\in\nat}S_i=\set{c_0,c_1,c_2,\ldots}$. By Corollary \ref{cor:iterate_d_get_U}, we can see that $d\pair{S_i}=U$ for all $i\in\nat$, hence $\bigcup_{i\in\nat}d\pair{S_i}=\bigcup_{i\in\nat}U=U$. However, $d(U)=U\cup\set{a}$, so $d\pair{\bigcup_{i\in\nat} S_i}\not\subseteq\bigcup_{i\in\nat}d\pair{S_i}$. Therefore, $d$ for this non-finitary AF is not $\omega$-continuous.

The following example is a non-finitary AF where $d$ is $\omega$-continuous. (Example \ref{eg:non-finitary_AF2}, page \pageref{eg:non-finitary_AF2} continued) We have that $d\equiv\set{b_i}_{i\in\nat}=U$ for this non-finitary AF; this can now be seen from the definition of $d$ and Corollary \ref{cor:iterate_d_get_U}. Let $\set{S_i}_{i\in\nat}$ be any $\omega$-chain of \textit{unattacked} arguments with limit $U\subseteq A-\set{a}$. Clearly, $\bigcup_{i\in\nat}d\pair{S_i}= A - \set{a} = U$, and $d\pair{\bigcup_{i\in\nat}S_i}= U \subseteq U$, which is true. Therefore, $d$ is $\omega$-continuous.
\end{proof}


\subsection{Self-Defending Sets}\label{sec:SD}

Self-defending sets formalise the idea that a set of arguments can respond to all counterattacks.

\subsubsection{Definition}

\begin{theorem}\label{thm:sd_equiv}
Let $S\subseteq A$. $S\subseteq d(S)$ iff $S^-\subseteq S^+$.
\end{theorem}
\begin{proof}
Let $S\subseteq A$. ($\Rightarrow$) If $a\in S\subseteq d(S)$, then $a\in d(S)$, iff $a^-\subseteq S^+$, which by taking the union of both sides over all $a\in S$ gives $\bigcup_{a\in S}a^-\subseteq S^+$, iff $S^-\subseteq S^+$.

($\Leftarrow$, contrapositive) If $S\not\subseteq d(S)$, then there is some $a\in S$ such that $a^-\not\subseteq S^+$, so given this $a\in S$ there is some $b\in a^-$ such that $b\notin S^+$. As $a\in S$ this implies that $a^-\subseteq S^-$, hence there is a $b\in S^-$ such that $b\notin S^+$, therefore $S^-\not\subseteq S^+$. This means that $S^-\subseteq S^+$ implies $S\subseteq d(S)$.
\end{proof}

\begin{definition}
We say $S\subseteq A$ is \textbf{self-defending} iff it satisfies any one of the two equivalent properties in Theorem \ref{thm:sd_equiv}.
\end{definition}

\noindent Intuitively, a self-defending set of arguments attacks all of its counterarguments. Formally, these sets are postfixed points of $d$ (Definition \ref{def:prefix_postfix_fp}, page \pageref{def:prefix_postfix_fp}).

\subsubsection{Existence}

\begin{definition}
Given an underlying AF, let $SD\subseteq\pow\pair{A}$ denote the set of self-defending sets.
\end{definition}

\begin{example}
(Example \ref{eg:3cycle} continued, page \pageref{eg:3cycle}) For this AF, we have
\begin{align}
SD=\set{\es,\set{b,c,e},\set{a,b,c,e}}.
\end{align}
\end{example}

\noindent If we need to make the underlying AF $\mathcal{A}=\ang{A,R}$, explicit, then we may write $SD\pair{\mathcal{A}}$ or $SD\pair{\ang{A,R}}$.

\begin{corollary}\label{cor:es_sd}
$\es\in SD$.
\end{corollary}
\begin{proof}
We have that $\es\subseteq d\pair{\es}=U$ (Corollary \ref{cor:unattacked_d_es}, page \pageref{cor:unattacked_d_es}), trivially.
\end{proof}

\noindent Therefore, for any AF, self-defending sets always exist as $\es$ is (vacuously) self-defending. So $SD\neq\es$.

\begin{corollary}\label{cor:U_sd}
$U\in SD$.
\end{corollary}
\begin{proof}
As $\es\in SD$, $\es\subseteq d\pair{\es}$ and hence by Corollary \ref{cor:d_monotone} (page \pageref{cor:d_monotone}), $d\pair{\es}\subseteq d^2\pair{\es}$. By Corollary \ref{cor:unattacked_d_es} (page \pageref{cor:unattacked_d_es}), it follows that $U\subseteq d\pair{U}$ and hence $U\in SD$.
\end{proof}

\begin{corollary}
If $S\subseteq U$, then $S\in SD$. The converse is false.
\end{corollary}
\begin{proof}
We prove the contrapositive. If $S\notin SD$, then $S\not\subseteq d(S)$. This means there is some $a\in S$ such that $a\notin d(S)$. For this $a$, it means that $a^-\not\subseteq S^+$, i.e. there is some $b\in a^-$ such that $b\notin S^+$. But if there is some $b\in a^-$, then $a^-\neq\es$ and hence $a\notin U$. Therefore, there exists an $a\in S$ such that $a\notin U$. Therefore, $S\notin U$.

For the converse, consider Example \ref{eg:simple_reinstatement} (page \pageref{eg:simple_reinstatement}) for $S=\set{a,c}\in SD$ but $\set{a,c}\not\subseteq\set{c}=U$.
\end{proof}

\begin{corollary}\label{cor:d_closed_SD}
$SD$ is closed under $d$.
\end{corollary}
\begin{proof}
Let $S\in SD$, then $S\subseteq d(S)$. As $d$ is $\subseteq$-monotonic, then $d(S)\subseteq d\sqbra{d\pair{S}}$. Therefore, $d(S)\in SD$. As $S\in SD$ is arbitrary, $d(S)\in SD$ and hence $d:SD\to SD$.
\end{proof}

\begin{corollary}\label{cor:n_closed_SD_not}
$SD$ is not closed under $n$.
\end{corollary}
\begin{proof}
Consider the following AF: $A=\set{a,b,c,e}$ and $R=\set{(a,b),(b,a),(e,c)}$, depicted in Figure \ref{fig:n_closed_SD_not}.

\begin{figure}[H]
\begin{center}
\begin{tikzpicture}[>=stealth',shorten >=1pt,node distance=2cm,on grid,initial/.style    ={}]
\tikzset{mystyle/.style={->,relative=false,in=0,out=0}};
\node[state] (a) at (0,0) {$ a $};
\node[state] (b) at (2,0) {$ b $};
\draw [->] (a) to [out = 45, in = 135] (b);
\draw [->] (b) to [out = 225, in = -45] (a);

\node[state] (c) at (6,0) {$ c $};
\node[state] (e) at (4,0) {$ e $};
\draw [->] (e) to (c);
\end{tikzpicture}
\caption{The AF from Corollary \ref{cor:n_closed_SD_not}.}\label{fig:n_closed_SD_not}
\end{center}
\end{figure}
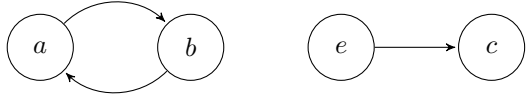

\noindent Clearly, $\set{a}\in SD$, because $d\pair{\set{a}}=\set{x\in A\:\vline\:x^-\subseteq\set{b}}=\set{a,e}\supseteq\set{a}$. Consider $n\pair{\set{a}}=\set{c,e}$. Is $\set{c,e}\in SD$? No, because $\set{c,e}^+=\set{e}$ and that $d\pair{\set{c,e}}=\set{x\in A\:\vline\:x^-\subseteq\set{e}}=\set{e}\not\supseteq\set{c,e}$. Therefore, $\set{a}\in SD$ but $n\pair{\set{a}}\notin SD$.
\end{proof}

\subsubsection{Lattice-Theoretic Structure}


\begin{corollary}
$SD$ is not in general $\subseteq$-upward closed.
\end{corollary}
\begin{proof}
From Example \ref{eg:simple_reinstatement} (page \pageref{eg:simple_reinstatement}), clearly that $\es\in SD$, and $\es\subseteq\set{b}$, but $\set{b}\notin SD$.
\end{proof}

\begin{corollary}
$A\in SD$ iff $A$ is the $\subseteq$-largest fixed point of $d$.
\end{corollary}
\begin{proof}
($\Rightarrow$) For any set $S\subseteq A$, $d(S)\subseteq A$ hence $d(A)\subseteq A$. As $A\in SD$, then $A\subseteq d(A)$. It follows that $A$ is a fixed point of $d$. For any set $S\in\pow\pair{A}$, we have $S\subseteq A$ hence $A$ is the $\subseteq$-largest fixed point of $d$.

($\Leftarrow$) If $d(A)=A$ then trivially $A\subseteq d(A)$ and hence $A\in SD$.
\end{proof}

\begin{corollary}\label{cor:SD_arb_union}
$SD$ is closed under arbitrary unions.
\end{corollary}
\begin{proof}
Let $\set{S_i}_{i\in I}$ be a family of self-defending subsets of $A$. For each $i\in I$, $S_i\subseteq d\pair{S_i}$. We have by Corollary \ref{cor:d_cup_cap} (page \pageref{cor:d_cup_cap}),
\begin{align*}
\bigcup_{i\in I}S_i\subseteq\bigcup_{i\in I}d\pair{S_i}\subseteq d\pair{\bigcup_{i\in I}S_i}.
\end{align*}
Therefore, $\bigcup_{i\in I}S_i$ is also self-defending.
\end{proof}

\noindent Notice if we take the empty union in Corollary \ref{cor:SD_arb_union}, then $\es\in SD$ which is also true. It follows that $SD$ is $\omega$-complete, chain complete and directed complete when instantiating this arbitrary family of self-defending sets into any $\omega$-chain, chain or directed set, respectively.



\begin{corollary}
$SD$ is not in general closed under intersections.
\end{corollary}
\begin{proof}
Consider Example \ref{eg:double_reinstatement} (page \pageref{eg:double_reinstatement}) where $\ang{\set{a,b,c,e},\set{\pair{c,b},\pair{e,b},\pair{b,a}}}$. Let $S_1=\set{c,a}$ and $S_2=\set{e,a}$. Clearly $d\pair{S_1}=\set{c,a}$ and $d\pair{S_2}=\set{e,a}$, therefore both $S_1$ and $S_2$ are self-defending. However, $S_1\cap S_2=\set{a}$. Further, $d\pair{\set{a}}=\set{c,e}$, and $\set{a}\not\subseteq d\pair{\set{a}}$, hence $S_1\cap S_2$ is not self-defending.
\end{proof}


\subsection{On the Interaction Between the Neutrality and Defence Functions}

Assume a fixed underlying $AF$ with neutrality and defence functions $n$ and $d$ respectively.

\subsubsection{Composing \texorpdfstring{$d$}{d} and \texorpdfstring{$n$}{n}}

The first result shows that $d$ is the square of $n$, or that $n$ is the ``square root'' of $d$.

\begin{theorem}\label{thm:d_n_squared}
\cite[Lemma 45]{Dung:95} For all $S\subseteq A$, we have that $d(S)=n^2(S)$.
\end{theorem}
\begin{proof}
Let $S\subseteq A$ be arbitrary. We have that
\begin{align*}
a\in n^2(S)\Leftrightarrow& a\notin\pair{A-S^+}^+\text{ by Definition \ref{def:neutrality_function} (page \pageref{def:neutrality_function}),}\\
\Leftrightarrow&\pair{\forall b\in A-S^+}a\notin b^+\text{ by Definition \ref{def:S_pm} (page \pageref{def:S_pm}),}\\
\Leftrightarrow&\pair{\forall b\in A-S^+}b\notin a^-\text{ by Corollary \ref{cor:plus_minus}, page \pageref{cor:plus_minus},}\\
\Leftrightarrow&\pair{\forall b\in A}\sqbra{b\notin S^+\Rightarrow b\notin a^-}\text{ by bounded quantifiers,}\\
\Leftrightarrow&\pair{\forall b\in A}\sqbra{b\in a^-\Rightarrow b\in S^+}\text{ by the contrapositive,}\\
\Leftrightarrow& a^-\subseteq S^+.
\end{align*}
This shows the result.
\end{proof}

The second result shows that $d$ and $n$ commute when composed.

\begin{theorem}\label{thm:n_d_commute}
For all $S\subseteq A$, we have $d\pair{n\pair{S}}=n\pair{d\pair{S}}$.
\end{theorem}
\begin{proof}
From Theorem \ref{thm:d_n_squared}, $d=n^2$ and hence $n\circ d = n\circ\pair{n\circ n}=\pair{n\circ n}\circ n$ by associativity of composition, hence $n\circ d=d\circ n$.
\end{proof}

\subsubsection{Preservation of \texorpdfstring{$CF$}{CF} by \texorpdfstring{$d$}{d}}

Theorem \ref{thm:n_d_commute} has many consequences. Firstly, unlike $n$ (Corollary \ref{cor:n_not_closed_on_CF}, page \pageref{cor:n_not_closed_on_CF}), $d$ preserves cf sets.

\begin{corollary}\label{cor:d_preserves_cf}
For any AF, if $S\in CF$ then $d(S)\in CF$. The converse is false.
\end{corollary}
\begin{proof}
The result follows from the $\subseteq$-monotonicity of $d$ (Corollary \ref{cor:d_monotone}, page \pageref{cor:d_monotone}) and Theorem \ref{thm:n_d_commute}. If $S\in CF$, then $S\subseteq n(S)$, then $d(S)\subseteq d\pair{n\pair{S}}=n\pair{d\pair{S}}$, therefore $d(S)\in CF$.

The converse is false. The counter-example is as follows: consider $\ang{A,R}$ where $A=\set{a,b,c}$ and $R=\set{(a,b),(b,a),(a,c),(c,b)}$; this is depicted in Figure \ref{fig:d_preserves_cf_counter}.

\begin{figure}[H]
\begin{center}
\begin{tikzpicture}[>=stealth',shorten >=1pt,node distance=2cm,on grid,initial/.style    ={}]
\tikzset{mystyle/.style={->,relative=false,in=0,out=0}};
\node[state] (a) at (0,0) {$ a $};
\node[state] (b) at (2,0) {$ b $};
\draw [->] (a) to [out = 45, in = 135] (b);
\draw [->] (b) to [out = 225, in = -45] (a);
\node[state] (c) at (1,-2) {$ c $};
\draw [->] (a) to (c);
\draw [->] (c) to (b);
\end{tikzpicture}
\caption{The AF that is a counter-example to the converse of Corollary \ref{cor:d_preserves_cf}.}\label{fig:d_preserves_cf_counter}
\end{center}
\end{figure}
Consider the subset of arguments $S=\set{a,c}$. We see that $d(S)=\set{a}$. We can see that $d(S)\in CF$ and $S\notin CF$. Therefore, the converse of the result is not true.
\end{proof}


\noindent By induction, any finite iteration of $d$ also preserves cf sets.

\begin{corollary}\label{cor:d_sequence_CF}
For $S\in CF$, the $\pow\pair{A}$-sequence $\set{S_i}_{i\in\nat}$ where $S_0:=S$ and $S_{i+1}:=d\pair{S_i}$ is a $CF$-sequence.
\end{corollary}
\begin{proof}
We show that $\pair{\forall i\in\nat} S_i\in CF$ by induction on $i$.

\begin{enumerate}
\item (Base) By assumption $S_0\in CF$.
\item (Inductive) If $S_i\in CF$ then $d\pair{S_i}\in CF$ by Corollary \ref{cor:d_preserves_cf}.
\end{enumerate}
Therefore, by induction, the result follows.
\end{proof}

Not only that finite iterations of $d$ on $S\in CF$ is also in $CF$, but that the limit of the ascending\footnote{This is an ascending chain as $d$ is $\subseteq$-monotone by Corollary \ref{cor:d_monotone} (page \pageref{cor:d_monotone}).} $\omega$-chain $\set{d^{k}\pair{S}}_{k\in\nat}$ is also cf.

\begin{corollary}\label{cor:iterate_d_limit_CF}
Let $S\in CF$. The limit of the chain $\set{d^i\pair{S}}_{i\in\nat}$, $\bigcup_{i\in\nat}d^i\pair{S}$, is also cf.
\end{corollary}
\begin{proof}
Immediate from Corollary \ref{cor:d_sequence_CF} (page \pageref{cor:d_sequence_CF}) and Theorem \ref{thm:cf_omega_complete} (page \pageref{thm:cf_omega_complete}).
\end{proof}

Further, if $d$ is $\omega$-continuous then the supremum of iterating $d$ on a self-defending set is also a fixed point of $d$.

\begin{corollary}
If $d$ is $\omega$-continuous (Definition \ref{def:omega_cts}, page \pageref{def:omega_cts}), then for $S\in SD$, the limit of the chain $\set{d^i\pair{S}}_{i\in\nat}$, $\bigcup_{i\in\nat}d^i\pair{S}$, is a fixed point of $d$.
\end{corollary}
\begin{proof}
By $\omega$-continuity of $d$, we have that
\begin{align}
d\pair{\bigcup_{i\in\nat}d^i\pair{S}} = \bigcup_{i\in\nat}d^{i+1}\pair{S}=\bigcup_{i\in\nat^+}d^i\pair{S}=S\cup\bigcup_{i\in\nat}d^{i+1}\pair{S},
\end{align}
where the last equality is because $S=d^0\pair{S}\subseteq d\pair{S}$, hence $d(S) = S\cup d(S)$. This means
\begin{align}
S\cup\bigcup_{i\in\nat}d^{i+1}\pair{S}=\bigcup_{i\in\nat}d^i\pair{S}.
\end{align}
Therefore, the limit is a fixed point of $d$.
\end{proof}

We can further strengthen this result via transfinite induction on ordinal-valued iterations of $d$. This is necessary as we do not assume the AFs we are dealing with are finite. For finite $\abs{A}$ it is sufficient to have ordinary induction over $\nat$ (as $\omega$). But if $\abs{A}$ is any cardinal number then we need to perform induction over a sufficiently large ordinal number.

\begin{lemma}\label{lem:ordinal_iteration_d_monotone}
Let $\alpha$ and $\beta$ be ordinal numbers and $S\subseteq A$. We have
\begin{align}
\alpha<\beta\Rightarrow d^\alpha(S)\subseteq d^\beta(S).
\end{align}
\end{lemma}
\begin{proof}
Let $\alpha$ be a given ordinal. If $\beta=\alpha+1$, then the result follows by the $\subseteq$-monotonicity of $d$. If $\beta$ is a limit ordinal larger than $\alpha$, then
\begin{align}
d^{\beta}\pair{S}=\bigcup_{\gamma<\beta}d^{\gamma}\pair{S}.
\end{align}
However, one of those terms in the union is $\gamma=\alpha$ and hence $d^{\alpha}\pair{S}\subseteq d^{\beta}\pair{S}$. Therefore, the result follows.
\end{proof}

\begin{theorem}\label{thm:d_presv_CF}
Let $\beta$ be an ordinal number. If $S\in CF$ then $\pair{\forall\alpha<\beta}d^\alpha(S)\in CF$.
\end{theorem}
\begin{proof}
We apply transfinite induction on $\alpha$.
\begin{enumerate}
\item (Base) If $\alpha=0$ then $d^0(S)=S\in CF$.
\item (Successor) If $d^\alpha(S)\in CF$ then $d\pair{d^\alpha\pair{S}}=d^{\alpha+1}\pair{S}\in CF$, by Corollary \ref{cor:d_preserves_cf}.
\item (Limit) Let $\gamma<\beta$ be a limit ordinal. Assume that $\pair{\forall\alpha<\gamma}d^\alpha(S)\in CF$, then
\begin{align}
T:=d^\gamma(S)=\bigcup_{\alpha<\gamma}d^\alpha(S).
\end{align}
Assume for contradiction that $T\notin CF$, then there are some $a,b\in T$ such that $R(a,b)$. Therefore, there are ordinal numbers $\alpha_1,\alpha_2<\gamma$ such that $a\in d^{\alpha_1}\pair{S}$ and $b\in d^{\alpha_2}\pair{S}$. By Lemma \ref{lem:ordinal_iteration_d_monotone}, we let $\delta:=\max\pair{\alpha_1,\alpha_2}<\gamma$ and hence $a,b\in d^{\delta}\pair{S}$. As $R(a,b)$, this means $d^{\delta}\pair{S}\notin CF$ -- contradiction, as we had assumed for all $\gamma<\beta$, $d^{\gamma}(S)\in CF$. Therefore,
\begin{align}
d^\gamma(S)=\bigcup_{\alpha<\gamma}d^\alpha(S)\in CF.
\end{align}
\end{enumerate}
By transfinite induction, this shows the result.
\end{proof}

\noindent Therefore, $d$ is closed on $CF$. This means $d:CF\to CF$ is well-defined.

\subsubsection{Interaction of Fixed Points of \texorpdfstring{$n$}{n} and \texorpdfstring{$d$}{d}}

The following result shows that $d$ is closed on the set of fixed points of $n$.

\begin{corollary}
If $S\subseteq A$ is a fixed point of $n$, then $d(S)$ is also a fixed point of $n$. The converse is false.
\end{corollary}
\begin{proof}
If $S=n(S)$ then $d(S)=d\pair{n(S)}=n\pair{d\pair{S}}$ by Theorem \ref{thm:n_d_commute}. The converse is false: consider $S=\set{c}$ for Example \ref{eg:simple_reinstatement} (page \pageref{eg:simple_reinstatement}). We have $d(S)=\set{a,c}$ which is a fixed point of $n$. However, $S$ is not a fixed point of $n$.
\end{proof}

\noindent Further, any fixed point of $n$ is also a fixed point of $d$.

\begin{corollary}\label{cor:fpn_fpd}
If $S\subseteq A$ is a fixed point of $n$, then it is also a fixed point of $d$. The converse is false.
\end{corollary}
\begin{proof}
If $S=n(S)$, then $n^2(S)=n(S)=S$, and hence $d(S)=S$. The converse is false: consider $S=\es$ for Example \ref{eg:nixon} (page \pageref{eg:nixon}), then $d\pair{\es}=\es$ yet $n(S)=A\neq\es$. Therefore, $S$ is a fixed point of $d$ but not of $n$.
\end{proof}

If $S$ is a fixed point for $d$, then so is $n(S)$.

\begin{corollary}
If $S\subseteq A$ is a fixed point of $d$, then $n(S)$ is also a fixed point of $d$. The converse is false.
\end{corollary}
\begin{proof}
We have that $S=d(S)$ means $n(S)=n\pair{d\pair{S}}=d\pair{n\pair{S}}$. For the converse: consider Example \ref{eg:simple_reinstatement} (page \pageref{eg:simple_reinstatement}) again, where $S=\set{c}$ and $n\pair{S}=\set{a,c}$ is a fixed point of $d$, but $\set{c}$ is not a fixed point of $d$.
\end{proof}

\subsection{Summary}

\begin{itemize}
\item Given an AF, its defence function $d:\pow\pair{A}\to\pow\pair{A}$ is defined as $a\in d(S)\Leftrightarrow a^-\subseteq S^+$.
\item $d$ is $\subseteq$-monotone. Therefore, let $F_d$ denote the set of all fixed points of $d$, then $\ang{F_d,\subseteq}$ is a complete lattice, so $d$ has a fixed point.
\item $d$ satisfies the following properties: $U=d\pair{\es}$, for any $S\subseteq A$, $U\subseteq d(S)$, and
\begin{align}
\bigcup_{i\in I}d\pair{S_i}\subseteq d\pair{\bigcup_{i\in I}S_i}\text{ and }\bigcap_{i\in I}d\pair{S_i}\supseteq d\pair{\bigcap_{i\in I}S_i},
\end{align}
\item If the underlying AF is finitary, then $d$ is $\omega$-continuous, else $d$ may or may not be $\omega$-continuous.
\item A set $S\subseteq A$ is self-defending iff $S\subseteq d\pair{S}$. The set of all self-defending sets for an AF $\mathcal{A}$ is $SD(\mathcal{A})$ or just $SD$.
\item For any AF, $SD$ satisfies $\es,\:U\in SD$, and the restriction of the defence function $d:SD\to SD$ is well-defined.
\item $SD$ is not $\subseteq$-upward closed, not closed under intersections, but closed under arbitrary unions.
\item $A\in SD$ iff $A$ is the $\subseteq$-largest fixed point of $d$.
\item Given an AF with neutrality function $n$ and defence function $d$, both from $\pow\pair{A}\to\pow\pair{A}$, we have $d\circ n=n\circ d$ and $n^2=d$.
\item $d:CF\to CF$ is well-defined. Further, for any $S\in CF$, one can generate a conflict-free-sequence $\set{S_i}_{i\in\nat}$ where $S_0=S$ and $S_{i+1}=d\pair{S_i}$, whose limit is also conflict-free.
\item If $d$ is $\omega$-continuous, then for $S\in SD$, the limit of the ascending chain of iterates on $S$ under $d$ is a fixed point of $d$.
\item If $S$ is a fixed point of $n$, then $d(S)$ is also a fixed point of $n$.
\item If $S$ is a fixed point of $n$, then it is also a fixed point of $d$.
\item If $S$ is a fixed point of $d$, then $n(S)$ is also a fixed point of $d$.
\end{itemize}

\newpage

\section{Admissible Sets}\label{sec:ADM}

Recall that conflict-freeness formalises a self-consistent set of arguments (Section \ref{sec:CF}, page \pageref{sec:CF}), and self-defence formalises the idea that a set of arguments replies adequately to all external criticisms (Section \ref{sec:SD}, page \pageref{sec:SD}). We now combine both these ideas to define when is it that an argument is ``winning''.

\subsection{Definition}

\begin{definition}
\cite[Definition 6(2), Lemma 18]{Dung:95} The set $S\subseteq A$ is an \textbf{admissible set} iff $S\subseteq n(S)$ and $S\subseteq d(S)$.
\end{definition}



\noindent Intuitively, admissible sets serve as the starting point for ``winning'' set of arguments, because these are the arguments that are mutually consistent and attacks all counterarguments. This addresses the main criticism against naive semantics (Section \ref{sec:nai_bad}, page \pageref{sec:nai_bad}) by now requiring that conflict-free sets of arguments also defend themselves.

\begin{definition}\label{def:admissible}
Given an underlying AF, let $ADM\subseteq\pow\pair{A}$ denote the set of admissible sets.
\end{definition}

\noindent If the underlying AF $\mathcal{A}=\ang{A,R}$ needs to be explicitly specified, we can write $ADM\pair{\mathcal{A}}$ or $ADM\pair{\ang{A,R}}$.

\begin{example}
(Example \ref{eg:fri}, page \pageref{eg:fri}) For floating reinstatement, we have
\begin{align}
ADM = \set{\es,\set{a},\set{b},\set{a,e},\set{b,e}}.
\end{align}
\end{example}

\begin{example}\label{eg:BH_eg221}
\cite[Example 2.2.1]{EoA} Consider the AF $A=\set{a,b,c,f,e}$ and \[R=\set{(a,b),(c,b),(c,f),(f,c),(f,e),(e,e)}.\]This is depicted in Figure \ref{fig:BH_eg221}.

\begin{figure}[H]
\begin{center}
\begin{tikzpicture}[>=stealth',shorten >=1pt,node distance=2cm,on grid,initial/.style    ={}]
\tikzset{mystyle/.style={->,relative=false,in=0,out=0}};
\node[state] (a) at (0,0) {$ a $};
\node[state] (b) at (2,0) {$ b $};
\node[state] (c) at (4,0) {$ c $};
\node[state] (f) at (6,0) {$ f $};
\node[state] (e) at (8,0) {$ e $};
\draw [->] (a) to (b);
\draw [->] (c) to (b);
\draw [->] (f) to [out = 225, in = -45] (c);
\draw [->] (c) to [out = 45, in = 135] (f);
\draw [->] (f) to (e);
\draw [->] (e) to [out=135,in=180] ($(e) + (0,1)$) to [out = 0, in = 45] (e);
\end{tikzpicture}
\caption{The AF from Example \ref{eg:BH_eg221}.}\label{fig:BH_eg221}
\end{center}
\end{figure}
\noindent We have $ADM = \set{\es,\set{a},\set{c},\set{a,c},\set{f},\set{a,f}}$.
\end{example}

\begin{example}
(Example \ref{eg:bi_inf} continued, page \pageref{eg:bi_inf}) We claim for this AF the non-empty admissible sets take the form $S\pair{a_i}:=\set{a_j\in A\:\vline\:j\geq i}$ and $S\pair{b_i}:=\set{b_j\in A\:\vline\:j\geq i}$, for $i\in\integ$.

Recall the AF as illustrated in Figure \ref{fig:bi_inf}, which we repeat here for convenience:

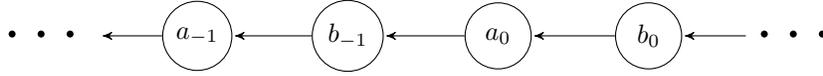
\begin{figure}[H]
\begin{center}
\begin{tikzpicture}[>=stealth',shorten >=1pt,node distance=2cm,on grid,initial/.style    ={}]
\tikzset{mystyle/.style={->,relative=false,in=0,out=0}};
\node (dotsn) at (-2,0) {\Huge $ \cdots $};
\node[state] (a-1) at (0,0) {$ a_{-1} $};
\node[state] (b-1) at (2,0) {$ b_{-1} $};
\node[state] (a0) at (4,0) {$ a_0 $};
\node[state] (b0) at (6,0) {$ b_0 $};
\node (dotsp) at (8,0) {\Huge $ \cdots $};
\draw [->] (a-1) to (dotsn);
\draw [->] (b-1) to (a-1);
\draw [->] (a0) to (b-1);
\draw [->] (b0) to (a0);
\draw [->] (dotsp) to (b0);
\end{tikzpicture}
\end{center}
\end{figure}

\noindent For example, we have $S\pair{a_{-1}}=\set{a_{-1}, a_0, a_1, a_2,\ldots}$. Notice that $S\pair{a_{-1}}^+=\set{b_j\in A\:\vline\:j\geq -2}$. Further, $S\pair{a_{-1}}^-=\set{b_j\in A\:\vline\:j\geq -1}\subseteq S\pair{a_{-1}}^+$.

More generally, it is easy to see that all sets of this form are cf. Given $S\pair{a_i}$ as above, by inspecting the AF, we can see that $S\pair{a_i}^-=\set{b_j\in A\:\vline\:j\geq i}$, because $b_i$ attacks $a_i$, $b_{i+1}$ attacks $a_{i+1}\in S\pair{a_i}$... etc. Further, $S\pair{a_i}^+=\set{b_j\in A\:\vline\:j \geq i-1}$, because $a_{i}$ attacks $b_{i-1}$, $a_{i+1}$ attacks $b_{i}$... etc. Therefore, we have that $S\pair{a_i}^-\subseteq S\pair{a_i}^+$ for all $i\in\integ$ and hence $S\pair{a_i}\in SD$ by Theorem \ref{thm:sd_equiv} (page \pageref{thm:sd_equiv}). By a similar argument, $S\pair{b_i}\in SD$ for all $i\in\integ$. Therefore, all sets of this form are admissible.

To show that these are the only non-empty, admissible sets, we can see that for any other conflict-free set, we need to defend against all attacks. If a cf set $S$ is finite, then it cannot be admissible as we choose the argument $a_i$ or $b_i\in A$ such that $i$ is the largest index and one of $a_i$ or $b_i$ has an attacker not defended by the set $S$, so no finite set of arguments is admissible. Further, no admissible set can have both $a_i$ and $b_j$ arguments for a given $i,j\in\integ$, because the attackers of $a_i$ or $b_i$ can only be defended against by, respectively, an $b_i$ or $a_i$ argument such that if you include all defenders, the resulting set cannot be cf. Satisfying self-defence would mean that the set includes all defender arguments with indices larger indices $i$. Therefore,
\begin{align}
ADM=&\set{\set{a_j\in A\:\vline\:j\geq i},\:\set{b_j\in A\:\vline\: j\geq i}\:\vline\:i\in\integ}\cup\set{\es}.
\end{align}
\end{example}

\begin{example}
\cite[Exercise 6(a)]{Caminada:08} Consider the AF in Figure \ref{fig:Caminada_ex2} (page \pageref{fig:Caminada_ex2}). Is $\set{a}\in ADM$? Yes, because $\set{a}\in CF$ and $a$ has no attackers so it is vacuously self defending.
\end{example}

\begin{example}\label{eg:Caminada_ex6b}
\cite[Exercise 6(b)]{Caminada:08} Consider the following AF \cite[Figure 5]{Caminada:08}, depicted in Figure \ref{fig:Caminada_ex6b}.

\begin{figure}[H]
\begin{center}
\begin{tikzpicture}[>=stealth',shorten >=1pt,node distance=2cm,on grid,initial/.style    ={}]
\tikzset{mystyle/.style={->,relative=false,in=0,out=0}};
\node[state] (e) at (-6,0) {$ a $};
\node[state] (a) at (0,0) {$ e $};
\node[state] (b) at (-2,0) {$ c $};
\node[state] (c) at (-4,0) {$ b $};
\node[state] (d) at (2,0) {$ f $};
\draw [->] (a) to [out = 225, in = -45] (b);
\draw [->] (b) to [out = 45, in = 135] (a);
\draw [->] (c) to (b);
\draw [->] (a) to (d);
\draw [->] (e) to (c);
\end{tikzpicture}
\caption{The AF depicting \cite[Figure 5]{Caminada:08}, from Example \ref{eg:Caminada_ex6b}.}\label{fig:Caminada_ex6b}
\end{center}
\end{figure}
\noindent Is $\set{c}\in ADM$? No, because $b$ attacks $c$, and $c$ does not attack $b$ back in turn. Therefore, $\set{c}\notin SD$.
\end{example}

\begin{example}\label{eg:Caminada_ex6c}
\cite[Exercise 6(c)]{Caminada:08} Consider the following AF \cite[Figure 6]{Caminada:08}, depicted in Figure \ref{fig:Caminada_ex6c}.

\begin{figure}[H]
\begin{center}
\begin{tikzpicture}[>=stealth',shorten >=1pt,node distance=2cm,on grid,initial/.style    ={}]
\tikzset{mystyle/.style={->,relative=false,in=0,out=0}};

\node[state] (a) at (0,0) {$ a $};
\node[state] (c) at (2,0) {$ c $};
\draw [->] (a) to [out=135,in=180] ($(a) + (0,1)$) to [out = 0, in = 45] (a);
\draw [->] (a) to (c);

\node[state] (b) at (4,1) {$ b $};
\node[state] (e) at (4,-1) {$ e $};
\draw [->] (b) to (c);
\draw [->] (c) to (e);

\end{tikzpicture}
\caption{The AF depicting \cite[Figure 6]{Caminada:08}, from Example \ref{eg:Caminada_ex6c}.}\label{fig:Caminada_ex6c}
\end{center}
\end{figure}
\noindent Is $\set{a}\in ADM$? No, because $\set{a}\notin CF$ as it is self-attacking.
\end{example}

\begin{example}
\cite[Exercise 6(d)]{Caminada:08} Consider the AF in Figure \ref{fig:Caminada_ex3} (page \pageref{fig:Caminada_ex3}). Is $\set{a,c,e}\in ADM$? No, because $c$ attacks $e$ and hence $\set{a,c,e}\notin CF$.
\end{example}

\subsection{Existence}

\begin{corollary}\label{cor:adm_cf_sd}
$ADM=CF\cap SD$.
\end{corollary}
\begin{proof}
Immediate from Definition \ref{def:admissible} (page \pageref{def:admissible}).
\end{proof}

\begin{corollary}\label{cor:empty_is_admissible}
$\es\in ADM$.
\end{corollary}
\begin{proof}
By Corollary \ref{cor:adm_cf_sd} (page \pageref{cor:adm_cf_sd}) and that $\es\in CF$ (Corollary \ref{cor:es_cf}) and $\es\in SD$ (Corollary \ref{cor:es_sd}, page \pageref{cor:es_sd}).
\end{proof}

\noindent Therefore, for any AF, admissible sets exist ($ADM\neq\es$), with $\es$ being admissible.

The following result shows why the set of unattacked arguments can always be seen as winning.

\begin{corollary}
$U\in ADM$.
\end{corollary}
\begin{proof}
Immediate from Corollaries \ref{cor:U_cf} and \ref{cor:U_sd} (pages \pageref{cor:U_cf} and \pageref{cor:U_sd}, respectively).
\end{proof}

\begin{corollary}\label{cor:pow_U_in_ADM}
If $S\subseteq U$ then $S\in ADM$. The converse is not true.
\end{corollary}
\begin{proof}
As $U\in CF$, $S\in CF$. We apply Corollary \ref{cor:iterate_d_get_U} (page \pageref{cor:iterate_d_get_U}) to $S$, so $U\subseteq d(S)$. Therefore, $S\subseteq U\subseteq d(S)$ and hence $S\in SD$. Therefore, $S\in ADM$.

The converse is not true, e.g. Example \ref{eg:simple_reinstatement}, $\set{a,c}\in ADM$ but $\set{a,c}\not\subseteq U=\set{c}$.
\end{proof}

\noindent Corollary \ref{cor:pow_U_in_ADM} can be written succinctly as $\pow\pair{U}\subset ADM$.

\begin{corollary}\label{cor:ADM_in_CF}
$ADM\subseteq CF$, and the converse is generally not true.
\end{corollary}
\begin{proof}
$ADM\subseteq CF$ is immediate from Corollary \ref{cor:adm_cf_sd}.

For the converse, consider the AF from Corollary \ref{cor:CF_no_cup} (page \pageref{cor:CF_no_cup}). Clearly $\set{b}\in CF$ but $d\pair{\set{b}}=\set{a}$ and $\set{b}\not\subseteq\set{a}$. Therefore, $\set{b}\notin ADM$.
\end{proof}

\begin{corollary}\label{cor:d_presv_ADM}
If $S\in ADM$ then $d(S)\in ADM$. The converse is not true.
\end{corollary}
\begin{proof}
If $S\in ADM$ then $S\in CF$ and $S\in SD$. By Corollary \ref{cor:d_preserves_cf} (page \pageref{cor:d_preserves_cf}), $d:CF\to CF$ and hence $d(S)\in CF$. By Corollary \ref{cor:d_closed_SD} (page \pageref{cor:d_closed_SD}), $d:SD\to SD$ and hence $d(S)\in SD$. Therefore, $d(S)\in CF\cap SD=ADM$.

For the converse, we can see from Example \ref{eg:BH_eg221} that $\set{a,f}\in ADM$ but $\set{a,f}\not\subseteq U=\set{a}$.
\end{proof}

\noindent Therefore, $d:ADM\to ADM$ is well-defined.

\begin{corollary}
If $S\in ADM$ then $n(S)\notin ADM$ in general.
\end{corollary}
\begin{proof}
In Example \ref{eg:nixon} (page \pageref{eg:nixon}), we have $n(\es)=\set{a,b}$, where $\es\in ADM$ but $\set{a,b}\notin CF$ and hence $\set{a,b}\notin ADM$.
\end{proof}

\subsection{Lattice Theoretic Properties}

\begin{lemma}
$ADM$ is in general not closed under intersections.
\end{lemma}
\begin{proof}
Example \ref{eg:fri} (page \pageref{eg:fri}), we have $\set{a,d},\set{b,d}\in ADM$ but $\set{a,d}\cap\set{b,d}=\set{d}\notin ADM$.
\end{proof}

\begin{lemma}
$ADM$ is in general not closed under unions.
\end{lemma}
\begin{proof}
Example \ref{eg:fri} (page \pageref{eg:fri}), we have $\set{a,d},\set{b,d}\in ADM$ but $\set{a,d}\cup\set{b,d}=\set{a,b,d}\notin ADM$.
\end{proof}

\noindent But if the union of a family of admissible sets is cf, then that the union of that family is also admissible. The following result generalises \cite[Lemma 1]{Caminada_IE:07} to accommodate possibly infinite AFs and generalised unions.

\begin{lemma}\label{lem:union_cf_adm_adm}
\cite[Lemma 1]{Caminada_IE:07} Let $I$ be an index set and $\set{S_i}_{i\in I}\subseteq ADM$. $\bigcup_{i\in I}S_i\in CF$ iff $\bigcup_{i\in I}S_i\in ADM$. 
\end{lemma}
\begin{proof}
($\Rightarrow$) If $I=\es$, then $\set{S_i}_{i\in I}=\es$ and the result follows by Corollary \ref{cor:empty_is_admissible}, as $\es\in ADM$. Otherwise, we know that $\pair{\forall i\in I}S_i\subseteq d\pair{S_i}$. Therefore, $\bigcup_{i\in I}S_i\subseteq\bigcup_{i\in I}d\pair{S_i}$. By Corollary \ref{cor:d_cup_cap} (page \pageref{cor:d_cup_cap}), $\bigcup_{i\in I}d\pair{S_i}\subseteq d\pair{\bigcup_{i\in I}S_i}$. Therefore, $\bigcup_{i\in I}S_i\in SD$ and hence $\bigcup_{i\in I}S_i\in ADM$.

($\Leftarrow$) Trivial, as $ADM\subseteq CF$.
\end{proof}

\begin{corollary}
If $S\in ADM$, then the limit of the $\omega$-chain $\set{d^k\pair{S}}_{k\in\nat}$ is also in $ADM$.
\end{corollary}
\begin{proof}
As $d:ADM\to ADM$, this chain is in $ADM$ by induction on $k$. By Corollary \ref{cor:iterate_d_limit_CF} (page \pageref{cor:iterate_d_limit_CF}), the limit of this chain is in $CF$. By Lemma \ref{lem:union_cf_adm_adm}, the result follows.
\end{proof}

After iterating $d$ a transfinite number of times on any $S\in ADM$, the result is still in $ADM$.

\begin{theorem}\label{thm:TFI_iterate_d_ADM}
Let $S\in ADM$. Let $\beta$ be an ordinal number. We have that
\begin{align}
\pair{\forall\alpha<\beta}d^{\alpha}\pair{S}\in ADM.
\end{align}
\end{theorem}
\begin{proof}
We apply transfinite induction on $\alpha$.
\begin{enumerate}
\item (Base) If $\alpha=0$ then $d^0\pair{S}=S\in ADM$.
\item (Successor) If $d^{\alpha}\pair{S}\in ADM$ then $d\pair{d^{\alpha}\pair{S}}\in ADM$, by Corollary \ref{cor:d_presv_ADM}.
\item (Limit) Let $\gamma<\beta$ be a limit ordinal. Assume that $\pair{\forall\alpha<\gamma}d^{\alpha}\pair{S}\in ADM$. Then
\begin{align}
T:=d^{\gamma}\pair{S}=\bigcup_{\alpha<\gamma}d^{\alpha}\pair{S}.
\end{align}
From the limit case of Theorem \ref{thm:d_presv_CF}, $T\in CF$. Now assume for some $a\in T$, $b\in a^-$. There is some $\alpha<\gamma$ such that $a\in d^{\alpha}\pair{S}$. As $d^{\alpha}\pair{S}\in ADM$ by assumption, we must have $b\in\pair{d^{\alpha}\pair{S}}^+$. Therefore, $b\in T^+$ and hence $T\in ADM$.
\end{enumerate}
By transfinite induction, this shows the result.
\end{proof}

\begin{theorem}
$\ang{ADM,\subseteq}$ is $\omega$-complete.
\end{theorem}
\begin{proof}
This follows from Theorem \ref{thm:cf_omega_complete} (page \pageref{thm:cf_omega_complete}) and Corollary \ref{cor:SD_arb_union} (page \pageref{cor:SD_arb_union}). Let $\set{S_i}_{i\in\nat}$ be an ascending chain in $ADM$. Let $S=\bigcup_{i\in\nat}S_i$. Clearly $S\in CF$. Further, if $a\in S$, then $\pair{\exists i\in\nat} a\in S_i\in ADM$ hence $\pair{\exists i\in\nat} a\in d\pair{S_i}$. As $S_i\subseteq S$, we have $d\pair{S_i}\subseteq d\pair{S}$ and hence $a\in d\pair{S}$. Therefore, $S\subseteq d(S)$. Therefore, $S\in ADM$.
\end{proof}

\begin{theorem}
$\ang{ADM,\subseteq}$ is chain complete.
\end{theorem}
\begin{proof}
Let $Ch\subseteq ADM$ be an arbitrary $\subseteq$-chain. By Theorem \ref{cor:ADM_in_CF}, $Ch\subseteq CF$ and by Theorem \ref{cor:CF_di_comp} (page \pageref{cor:CF_di_comp}), $\bigcup Ch\in CF$. By Lemma \ref{lem:union_cf_adm_adm}, it follows that $\bigcup Ch\in ADM$. As $Ch$ is any chain in $ADM$, the result follows.
\end{proof}

\begin{corollary}\label{cor:ADM_has_max}
$\max_{\subseteq}ADM\neq\es$.
\end{corollary}
\begin{proof}
Every chain in $ADM$ has a least upper bound in $ADM$ because the poset $\ang{ADM,\subseteq}$ is chain complete from the preceding theorem. By Zorn's lemma, $\ang{ADM,\subseteq}$ has at least one $\subseteq$-maximal element.
\end{proof}

\noindent Just as in Theorem \ref{thm:existence_of_naive} (page \pageref{thm:existence_of_naive}), we have used Zorn's lemma to show that there exist maximal admissible sets. Therefore, Zorn's lemma is sufficient for showing Corollary \ref{cor:ADM_has_max}. However, we may ask whether Zorn's lemma is \textit{necessary} for Corollary \ref{cor:ADM_has_max}. We answer this in Appendix \ref{app:AC} (page \pageref{app:AC}).

\begin{corollary}\label{cor:adm_dir_set}
Let $\mathcal{D}\subseteq ADM$ be a directed subset under $\subseteq$. Its supremum $\bigcup\mathcal{D}\in ADM$.
\end{corollary}
\begin{proof}
We have that $\mathcal{D}\subseteq ADM\subseteq CF$ where $\subseteq$ is the underlying partial order, so $\mathcal{D}$ is also a directed set in $CF$ and hence $\bigcup\mathcal{D}\in CF$. By choosing $\set{S_i}_{i\in I}$ from Lemma \ref{lem:union_cf_adm_adm} (page \pageref{lem:union_cf_adm_adm}) to $\mathcal{D}$, it follows that $\bigcup\mathcal{D}\in ADM$.
\end{proof}

\begin{corollary}\label{cor:ADM_cpo1}
\cite[Theorem 11(1)]{Dung:95} The set $\ang{ADM,\subseteq}$ is a pointed directed complete partial order (dcpo, Definition \ref{def:dcpo}, page \pageref{def:dcpo}).
\end{corollary}
\begin{proof}
This is immediate from Corollaries \ref{cor:adm_dir_set} and \ref{cor:empty_is_admissible} (page \pageref{cor:empty_is_admissible}).
\end{proof}

\begin{corollary}\label{cor:non-empty_adm_glb}
Every non-empty set of admissible sets has a $\subseteq$-glb.
\end{corollary}
\begin{proof}
Let $\es\neq\mathcal{S}\subseteq ADM$ be a non-empty set of admissible sets. We show that $\mathcal{S}$ has a $\subseteq$-glb $S\in ADM$.

Let $LB:=\set{T\in ADM\:\vline\:\pair{\forall S'\in\mathcal{S}}T\subseteq S'}$ be the set of admissible $\subseteq$-lower bounds of $\mathcal{S}$. As $\es\in LB$ by Corollary \ref{cor:empty_is_admissible}, we have $LB\neq\es$. Let $S:=\bigcup LB$.\footnote{Note that $LB=\set{\es}\neq\es$ iff $S=\es$. In such a case, $\es\in ADM$ is the $\subseteq$-glb of $\mathcal{S}$.} To show that $S\in ADM$, take any $S'\in\mathcal{S}$, which is cf. As for all $T\in LB$, $T\subseteq S'$, it follows that $S\subseteq S'$. Therefore, $S\in CF$. By Lemma \ref{lem:union_cf_adm_adm}, $S\in ADM$. Clearly, $S$ is a lower bound of $\mathcal{S}$. For any admissible $\subseteq$-lower bound of $\mathcal{S}$, say $S''$, we have $S''\in ADM$ and hence $S''\in LB$, thus $S''\subseteq S$. Therefore, $S\in ADM$ is the greatest $\subseteq$-lower bound of $\mathcal{S}$.
\end{proof}

\begin{corollary}
The poset $\ang{ADM,\subseteq}$ is a complete semilattice.
\end{corollary}
\begin{proof}
By Corollary \ref{cor:ADM_cpo1}, $\ang{ADM,\subseteq}$ is directed complete and hence chain complete. By Corollary \ref{cor:non-empty_adm_glb}, every non-empty subset of $\ang{ADM,\subseteq}$  has a $\subseteq$-glb. The result follows.
\end{proof}

\noindent In summary, $\ang{ADM,\subseteq}$ is a complete semilattice that is also directed complete. This generalises \cite[Theorem 25(3)]{Dung:95} from complete extensions to admissible sets.

\subsection{Dung's Fundamental Lemma}

The intuition behind \textit{Dung's fundamental lemma} is that whatever one can defend, one can also incorporate into one's own knowledge in a consistent manner. This is formalised as follows:

\begin{lemma}\label{lem:dfl}
(\textbf{Dung's fundamental lemma} \cite[Lemma 10]{Dung:95}) Let $S\in ADM$ and $a,b\in d(S)$, then
\begin{enumerate}
\item $S\cup\set{a}\in ADM$ and
\item $b\in d\pair{S\cup\set{a}}$.
\end{enumerate}
\end{lemma}
\begin{proof}
In turn:
\begin{enumerate}
\item We have to show that $S\cup\set{a}\in CF\cap SD$. Note that $S\cup\set{a}\in SD$ is true because $a\in d(S)\Leftrightarrow\set{a}\subseteq d(S)$ and $d(S)\subseteq d\pair{S\cup\set{a}}$. As $S\in ADM$ and hence $S\in SD$, so $S\subseteq d(S)\subseteq d\pair{S\cup\set{a}}$, it follows that $S\cup\set{a}\subseteq d\pair{S\cup\set{a}}$. Now we need to show $d\pair{S\cup\set{a}}\in CF$. Assume for contradiction that $d\pair{S\cup\set{a}}\notin CF$. There exists $x,y\in d\pair{S\cup\set{a}}$ such that $R(x,y)$. There are four cases:
\begin{enumerate}
\item $x,y\in S$ -- this is impossible because $S\in ADM$ means $S\in CF$.
\item $x\in S$ and $y=a$ -- it follows that $a\in S^+$ so $S\cap a^-\neq\es$, but as $a\in d(S)$, $a^-\subseteq S^+$ and hence $S\cap S^+\neq\es$ -- contradiction because $S\in ADM\subseteq CF$.
\item $x=a$ and $y\in S$ -- as $S\in ADM$, $a\in S^+$, which is impossible for the same reasons as the preceding case..
\item $x=y=a$ -- if $a$ is self-defeating, then as $a\in d(S)$, $a\in S^+$, which is impossible for the same reasons as the previous two cases.
\end{enumerate}
Since all four cases lead to contradiction, it follows that $d\pair{S\cup\set{a}}\in CF$. This means $d\pair{S\cup\set{a}}\in ADM$.
\item $b\in d\pair{S}\subseteq d\pair{S\cup\set{a}}$ by Corollary \ref{cor:d_monotone} (page \pageref{cor:d_monotone}).
\end{enumerate}
This shows the result.
\end{proof}

We can generalise Dung's fundamental lemma as follows.

\begin{lemma}\label{lem:gen_dfl}
(Generalised fundamental lemma) If $S\in ADM$ and $W,V\subseteq d(S)$, then
\begin{enumerate}
\item $S':=S\cup W\in ADM$ and
\item $V\subseteq d\pair{S'}$.
\end{enumerate}
\end{lemma}
\begin{proof}
In turn:
\begin{enumerate}
\item As $S\in ADM$, then $S\subseteq d(S)$. Similarly, $W\subseteq d(S)$ and hence $S':=S\cup W\subseteq d(S)\subseteq d\pair{S\cup W}$. Therefore, $S'\in SD$. Similarly, as $S\subseteq d(S)$, $S\cup d(S)=d(S)\in CF$ by Corollary \ref{cor:d_preserves_cf} (page \pageref{cor:d_preserves_cf}). As $W\subseteq d(S)$, we have that $S'=S\cup W\subseteq d(S)$ and hence $S'\in CF$ by Corollary \ref{cor:CF_down_closed} (page \pageref{cor:CF_down_closed}). Therefore, $S'\in ADM$.
\item As $S\subseteq S\cup W=S'$, we have $d(S)\subseteq d\pair{S'}$ but as $V\subseteq d(S)$, we have $V\subseteq d\pair{S'}$.
\end{enumerate}
This shows the result.
\end{proof}

\noindent The fundamental lemma is thus recovered from the generalised fundamental lemma by choosing $W$ and $V$ to be singleton sets. Therefore, Lemma \ref{lem:dfl} and Lemma \ref{lem:gen_dfl} are logically equivalent.

\subsection{Symmetric Argumentation Frameworks}

\begin{theorem}\label{thm:sym_AF_CF_ADM_equal}
An non-empty, non-trivial AF without self-attacking arguments is symmetric iff $CF = ADM$.
\end{theorem}
\begin{proof}
($\Rightarrow$) (From \cite[Proposition 4]{Coste:05}) From Corollary \ref{cor:adm_cf_sd}, we have $ADM\subseteq CF$ for all AFs. Therefore, it is sufficient to show $CF\subseteq ADM$ for symmetric AFs.

Let $S\in CF$. We need to show $S\in SD$, which by Theorem \ref{thm:sd_equiv} (page \pageref{thm:sd_equiv}) is equivalent to $S^-\subseteq S^+$. Then $S\in SD$ follows, because
\begin{align}
a\in S^-\Leftrightarrow\pair{\exists b\in S}R(a,b)\Rightarrow\pair{\exists b\in S}R(b,a)\Leftrightarrow a\in S^+,
\end{align}
by symmetry of $R$. Therefore, $S\in SD$ and hence $S\in ADM$. The result follows as $S$ is arbitrary.

($\Leftarrow$, contrapositive) Assume our AF $\ang{A,R}$ is not symmetric. As we have assumed our AF is neither empty nor trivial and has no self-attacking arguments, then $R$ is a non-empty non-symmetric relation. There exists $a,b\in A$ such that $R(a,b)$ and $\neg R(b,a)$. Clearly, $\set{b}\in CF$ and $\set{b}\notin ADM$, as it is not self-defending. Therefore, $CF\neq ADM$.
\end{proof}

\subsection{Summary}

\begin{itemize}
\item $S\subseteq A$ is admissible iff it is both conflict-free and self-defending. We denote the set of all admissible sets of an AF $\mathcal{A}$ $ADM(\mathcal{A})$ or just $ADM$.
\item For any AF, we have $\es,\:U\in ADM$, and $ADM\subset CF$. Further, $d:ADM\to ADM$ is well-defined.
\item $ADM$ is not closed under unions or intersections.
\item For any family of sets $\mathcal{S}\subseteq ADM$, if $\bigcup\mathcal{S}\in CF$ then $\bigcup\mathcal{S}\in ADM$.
\item If $S\in ADM$ then the limit of the chain $\set{d^k\pair{S}}_{k\in\nat}$ is also in $ADM$.
\item $\ang{ADM,\subseteq}$ is a complete semilattice that is also directed complete, and $\max_{\subseteq}ADM\neq\es$.
\item Every non-empty family of admissible sets has a $\subseteq$-glb that is also admissible.
\item If $S\in ADM$ and $W,V\subseteq d(S)$, then $S\cup W\in ADM$ and $V\subseteq d\pair{S\cup W}$.
\item An AF is symmetric iff $CF=ADM$. This assumes the AF is non-empty, non-trivial and has no self-attacking arguments.
\end{itemize}




\newpage

\section{Complete Extensions}\label{sec:COMP}

\subsection{Definition}

%

\begin{definition}
\cite[Definition 23 and Lemma 24]{Dung:95} $S\subseteq A$ is a \textbf{complete extension} iff $S\in CF\cap F_d$.
\end{definition}

\noindent Intuitively, complete extensions are stronger than admissible sets because a complete extension demands that you believe everything that you can defend while still maintaining consistency.

\begin{definition}
Given an underlying AF, let $COMP$ denote the set of all complete extensions.
\end{definition}

\noindent If the underlying AF $\mathcal{A}$ needs to be explicitly specified, we can write $COMP\pair{\mathcal{A}}$.

\begin{example}
(Example \ref{eg:BH_eg221} continued, page \pageref{eg:BH_eg221}) Recall we have arguments $A=\set{a,b,c,f,e}$ and attacks $R=\set{(a,b),(c,b),(c,f),(f,c),(f,e),(e,e)}.$ We have\[COMP = \set{\set{a},\set{a,c},\set{a,f}}.\]
\end{example}

\begin{example}\label{eg:fri_COMP}
(Example \ref{eg:fri} continued, page \pageref{eg:fri}) It can be shown that $COMP=\set{\es,\set{a,e},\set{b,e}}$.
\end{example}

\begin{example}
(Example \ref{eg:bi_inf} continued, page \pageref{eg:bi_inf}). We show that $COMP=\set{\es,\set{a_i}_{i\in\integ},\set{b_i}_{i\in\integ}}$. Clearly, $\es\in COMP$ because there are no unattacked arguments. Also, $\set{a_i}_{i\in\integ}$ is a complete extension because it is in $ADM$ and every argument that is defended by $\set{a_i}_{i\in\integ}$ also belongs to $\set{a_i}_{i\in\integ}$. The same can be argued for $\set{b_i}_{i\in\integ}$.
\end{example}

\subsection{Existence}

\subsubsection{Existence from Fixed Point Theory}\label{sec:comp_exists_fp}

For each AF, complete extensions exist. To prove that $COMP\neq\es$, we first prove the following theorem.

\begin{theorem}\label{thm:grounded_def2}
\cite[Theorem 25(2)]{Dung:95} Let $\ang{A,R}$ be an arbitrary AF with defence function $d:\pow\pair{A}\to\pow\pair{A}$. The following two statements are equivalent.
\begin{enumerate}
\item $S$ is the $\subseteq$-least complete extension.
\item $S$ is the $\subseteq$-least element of $F_d$.
\end{enumerate}
\end{theorem}
\begin{proof}
($1\Rightarrow 2$, contrapositive) Assume that $S$ is \textit{not} the $\subseteq$-least element of $F_d$. Either $S\notin F_d$, or $S\in F_d$ and $S$ is not the $\subseteq$-least element of $F_d$. 
\begin{enumerate}
\item If $S\notin F_d$, then $d(S)\neq S$ and hence $S\notin COMP$. Therefore, $S$ cannot be the $\subseteq$-least complete extension.
\item If $S\in F_d$ and $S$ is not the $\subseteq$-least element of $F_d$, then $\pair{\exists T\in F_d}T\subset S$ by definition of $\subseteq$-least. Either $T\in CF$ or $T\notin CF$.
\begin{enumerate}
\item If $T\in CF$, then $T\in COMP$. As $T\subset S$, $S$ cannot be the $\subseteq$-least complete extension.
\item If $T\notin CF$, then as $T\subseteq S$, $S\notin CF$ (Corollary \ref{cor:non_CF_up_closed}, page \pageref{cor:non_CF_up_closed}) either so $S\notin COMP$. Therefore, $S$ cannot be the $\subseteq$-least complete extension.
\end{enumerate}
\end{enumerate}
\noindent In all cases, $S$ cannot be the $\subseteq$-least complete extension.

($2\Rightarrow 1$, contrapositive) Assume $S$ is \textit{not} the $\subseteq$-least complete extension, then either $S\notin COMP$, or $S\in COMP$ and $S$ is not $\subseteq$-least.
\begin{enumerate}
\item If $S\in COMP$ and $S$ is not $\subseteq$-least, then $\pair{\exists T\in COMP}T\subset S$, but as $S,T\in COMP$, we have $S,T\in F_d$ and hence $S$ is not the $\subseteq$-least element of $F_d$, because $T\subset S$.
\item If $S\notin COMP$, then either $S\notin CF$ or $S\notin F_d$.
\begin{enumerate}
\item If $S\notin F_d$, then $S$ cannot be the $\subseteq$-least element of $F_d$.
\item If $S\notin CF$, then assume for contradiction that $S$ \textit{is} the $\subseteq$-least fixed point of $d$. Therefore, $S\in F_d$ and $\pair{\forall T\in F_d}S\subseteq T$. It follows that $\pair{\forall T\in F_d}T\notin CF$, because any superset of a non-cf set cannot be cf (Corollary \ref{cor:non_CF_up_closed}, page \pageref{cor:non_CF_up_closed}). It follows that $F_d\cap CF=\es$, which means $COMP=\es$. However, we have assumed that the underlying AF is \textit{arbitrary}. It cannot be true that $COMP=\es$ for arbitrary AFs. For example, in Example \ref{eg:fri} (page \pageref{eg:fri}), we have an AF where $COMP\neq\es$. Therefore, $S$ is \textit{not} the $\subseteq$-least fixed point of $d$.
\end{enumerate}
\end{enumerate}
\noindent In all cases, $S$ is not the $\subseteq$-least fixed point of $d$.

The result follows.
\end{proof}

\begin{corollary}\label{cor:comp_exists}
For any $AF$, $COMP\neq\es$.
\end{corollary}
\begin{proof}
Given an AF, let $d$ be its defence function. The least fixed point of $d$ exists, call it $G$. By Theorem \ref{thm:grounded_def2}, $G$ is also the $\subseteq$-least complete extension. Therefore, $G\in COMP$, hence $COMP\neq\es$.
\end{proof}

\begin{definition}\label{def:pre_grounded}
For an AF $\ang{A,R}$, let $G\subseteq A$ denote the $\subseteq$-least complete extension.
\end{definition}

\begin{corollary}\label{cor:cap_comp_is_G}
We have that $G=\bigcap COMP$.
\end{corollary}
\begin{proof}
As $G$ is the $\subseteq$-least complete extension, we have $\pair{\forall C\in COMP}G\subseteq C$ and hence $G\subseteq\bigcap_{C\in COMP}C=\bigcap COMP$. Therefore, $G\subseteq\bigcap COMP$. Now let $a\in\bigcap COMP$, then for all $C\in COMP$, $a\in C$, in particular as $G\in COMP$, $a\in G$. Therefore, as $a$ is arbitrary, $\bigcap COMP\subseteq G$. The result follows.
\end{proof}

\begin{corollary}\label{cor:empty_grounded}
$U=\es$ iff $G=\es$.
\end{corollary}
\begin{proof}
($\Rightarrow$) If $U=\es$ then as $U=d\pair{\es}$, $\es$ is the least fixed point of $d$ by definition hence $G=\es$.

($\Leftarrow$) If $G=\es$, then as $U\subseteq G=\es$, we must have $U=\es$.
\end{proof}

\begin{corollary}\label{cor:U_subset_COMP}
Let $U$ be the set of undefeated arguments. For any $C\in COMP$, we have $U\subseteq C$.
\end{corollary}
\begin{proof}
Recall that for any $S\subseteq A$, $U\subseteq d(S)$. Let $C\in COMP$, then $d(C)=C$ and hence $U\subseteq C$.
\end{proof}

\noindent This result shows that every complete extension will have at least the unattacked arguments. The additional arguments are those that can be indirectly defended by the unattacked arguments.

Trivially, if there are unattacked arguments then the complete extensions cannot be empty.

\begin{corollary}
If $U\neq\es$, then $\es\notin COMP$.
\end{corollary}
\begin{proof}
If $U\neq\es$, then as $U\subseteq C$ for all $C\in COMP$, $C\neq\es$ for all $C\in COMP$. Hence, $\es\notin COMP$.
\end{proof}

\begin{example}
\cite[Exercise 4(a)]{Caminada:08} From Figure \ref{fig:Caminada_ex2} (page \pageref{fig:Caminada_ex2}), we show that $COMP=\set{\set{a,c}}$. This is because $U=\set{a}$ and $c$ is reinstated by $a$. There are no other complete extensions.
\end{example}

\begin{example}
\cite[Exercise 4(b)]{Caminada:08} From Figure \ref{fig:Caminada_ex6b} (page \pageref{fig:Caminada_ex6b}) we can see that $COMP=\set{\set{a},\set{a,c,f},\set{a,e}}$. This is because $U=\set{a}$, and $d(\set{a})=\set{a}$ because $c$ is not reinstated by $a$ due to the attack on $c$ from $e$. As $a$ is not self-attacking, we conclude that $\set{a}\in COMP$. Further, $\set{a,e}$ is a complete extension because it is conflict-free and conatins all arguments it can defend by defending against the attack $c$ against $e$. Similarly, $\set{a,c,f}$ is a complete extension because it is conflict-free and it contains all arguments it defends by defending against the attacks $b$ against $c$ and $e$ against $f$.
\end{example}

\begin{example}
\cite[Exercise 4(c)]{Caminada:08} From Figure \ref{fig:Caminada_ex6c} (page \pageref{fig:Caminada_ex6c}) we can see that $COMP=\set{\set{b,e}}$, because $b$ is unattacked and $e$ is reinstated, and there are no further arguments defended. Note $a$ is self-attacking and cannot be in any complete extension. This is the only complete extension.
\end{example}

\begin{example}
\cite[Exercise 4(d)]{Caminada:08} From Figure \ref{fig:Caminada_ex3} (page \pageref{fig:Caminada_ex3}), we have that $COMP=\set{\es,\set{a},\set{b,e}}$. Firstly, $\es\in CF$ and there are no unattacked arguments, so $d(\es)=\es$ and hence $\es\in COMP$. The set $\set{a}\in CF$ and defends only itself as $c$ is not reinstated by $a$ due to the attack from $f$. Finally, $\set{b,e}\in CF$ and defends exactly itself by having $b$ attacking $a$ and $c$ and that $e$ attacks $f$.
\end{example}

\subsubsection{Existence from Admissible Sets}

\begin{theorem}\label{thm:iterate_d_adm_to_comp}
Let our AF be $\ang{A,R}$ and let $S\in ADM$. Let $\beta$ be a sufficiently large ordinal number such that the ordinal-indexed sequence $\set{d^{\alpha}\pair{S}}_{\alpha<\beta}$ has stabilised, for a given cardinal number $\abs{A}$. We have that
\begin{align}
\bigcup_{\alpha<\beta}d^{\alpha}\pair{S}\in COMP.
\end{align}
\end{theorem}
\begin{proof}
By Theorem \ref{thm:TFI_iterate_d_ADM} (page \pageref{thm:TFI_iterate_d_ADM}), the limit $L:=\bigcup_{\alpha<\beta}d^{\alpha}\pair{S}\in ADM$. Therefore we need to show this limit is in $COMP$. It is sufficient to show $d(L)\subseteq L$. For $a\in A$,
\begin{align*}
a\in d(L)\Leftrightarrow& a^-\subseteq L^+\\
\Leftrightarrow& a^-\subseteq\pair{\bigcup_{\alpha<\beta}d^{\alpha}\pair{S}}^+\\
\Leftrightarrow& a^-\subseteq\bigcup_{\alpha<\beta}\sqbra{d^{\alpha}\pair{S}}^+\text{ by Corollary \ref{cor:cup_cap_plus_minus} (page \pageref{cor:cup_cap_plus_minus}),}\\
\Leftrightarrow&\pair{\exists\alpha<\beta}a^-\subseteq\sqbra{d^{\alpha}\pair{S}}^+\text{, ``$\Rightarrow$'' follows as $\set{d^{\alpha}\pair{S}}_{\alpha<\beta}$ is a chain,\footnotemark}\\
\Leftrightarrow&\pair{\exists\alpha<\beta}a\in d^{\pair{\alpha+1}}\pair{S}\text{ by definition of $d$,}\\
\Leftrightarrow&\pair{\exists\alpha<\beta}a\in d^{\alpha}\pair{S}\text{ by definition of $\exists$,}\\
\Leftrightarrow&a\in\bigcup_{\alpha<\beta}d^{\alpha}\pair{S}=L,
\end{align*} \footnotetext{The $\Leftarrow$ direction is trivial because $a^-\subseteq\sqbra{d^{\alpha}(S)}^+\subseteq\bigcup_{\alpha<\beta}\sqbra{d^\alpha(S)}^+$. The $\Rightarrow$ direction is valid by the chain property because if the elements of $a^-$ are spread out over different sets in the generalised union $\bigcup_{\alpha<\beta}\sqbra{d^\alpha(S)}^+$, we can choose the largest of these sets which exist by the chain property, such that there is some sufficiently large $\alpha$ strictly less than $\beta$ (because $d$ is $\subseteq$-monotone), which will contain the set $a^-$.\label{fn:chain_reasoning}}
and therefore $d(L)=L$, so $L\in COMP$. Therefore, any transfinite iteration of an admissible set by $d$ stabilises into a complete extension, for a suitably large ordinal.
\end{proof}

\begin{corollary}\label{cor:ADM_complete_closure}
Let $S\in ADM$, then there exists an $L\in COMP$ such that $S\subseteq L$.
\end{corollary}
\begin{proof}
Given $S$, we iterate $d$ on $S$ a transfinite number of times until the sequence stabilises at some limit $L$, which is guaranteed because $d$ is $\subseteq$-monotonic and when the cardinality of the ordinal number denoting the iteration is at least $\abs{A}$. Clearly, $S\subseteq L$, and by Theorem \ref{thm:iterate_d_adm_to_comp}, $L\in COMP$.
\end{proof}

\subsubsection{Existence of Non-Empty Complete Extensions from Limited Controversy}

\begin{definition}
We say $a\in A$ \textbf{threatens} $S\subseteq A$ iff $a\in S^--S^+$.
\end{definition}

\noindent In other words, the set of all threats to $S$ are the set of arguments attacking $S$ that $S$ fails to defend against.

\begin{corollary}
$S\in SD$ iff no argument threatens $S$, i.e. $S^- -S^+=\es$.
\end{corollary}
\begin{proof}
We have $S\in SD$ iff $S^-\subseteq S^+$ (Theorem \ref{thm:sd_equiv}, page \pageref{thm:sd_equiv}) iff $S^- - S^+=\es$ iff no argument threatens $S$.
\end{proof}

\begin{definition}\label{def:defense_set}
We say $D\subseteq A$ is a \textbf{defense of $S\subseteq A$} iff $S^- -S^+\subseteq D^+$.
\end{definition}

\begin{corollary}
$S\in SD$ iff $\es$ is a defense of $S$.
\end{corollary}
\begin{proof}
$\es$ is a defense of $S$ iff $S^- -S^+\subseteq\es^+$ iff $S^- - S^+\subseteq\es$ iff $S^- - S^+=\es$ iff $S^-\subseteq S^+$ iff $S\in SD$.
\end{proof}

\noindent In other words, self-defending sets do not need anything else as a defense.

\begin{corollary}\label{cor:defense_indirectly_defends}
If $D$ is a defense of $S$, then there exists an argument in $D$ that indirectly defends the arguments in $S$.
\end{corollary}
\begin{proof}
If $S^--S^+\subseteq D^+$, then for $a\in S^- -S^+$, there is some $b\in S$ such that $R(a,b)$. However, as $a\in D^+$, then there is some $c\in D$ such that $R(c,a)$. Therefore, there is an even-length path from $c\in D$ to $b\in S$ and hence there is an argument in $D$ that indirectly defends arguments in $S$.
\end{proof}

\begin{corollary}\label{cor:defense_S_fills_defence}
$D$ is a defense of $S$ iff $S\subseteq d\pair{S\cup D}$.
\end{corollary}
\begin{proof}
($\Rightarrow$) Let $a\in S$ and $b\in A$ be arbitrary such that $R(b,a)$, hence $b\in S^-$. Either $b\in S^+$ or $b\notin S^+$. In the former case, $b\in S^+$ would mean $a\in d(S)\subseteq d\pair{S\cup D}$. In the latter case, $b\in S^--S^+$ and hence $b\in D^+$, which means $b\in D^+\cup S^+=\pair{S\cup D}^+$ by Corollary \ref{cor:cup_cap_plus_minus} (page \pageref{cor:cup_cap_plus_minus}), $a\in d\pair{S\cup D}$. The result follows.

($\Leftarrow$) If $S\subseteq d\pair{S\cup D}$, then $S^-\subseteq\pair{S\cup D}^+=S^+\cup D^+$ by Corollary \ref{cor:cup_cap_plus_minus} (page \pageref{cor:cup_cap_plus_minus}). Now let $b\in S^-- S^+$ be arbitrary, then $b\in S^-$ and $b\notin S^+$. The first case implies that $b\in S^+\cup D^+$, but the second case means $b\in D^+- S^+\subseteq D^+$. Therefore, $S^-- S^+\subseteq D^+$, and hence $D$ is a defense of $S$. 
\end{proof}

\begin{theorem}\label{thm:LC_non_empty_comp}
\cite[Lemma 34]{Dung:95} If an AF is limited controversial then there exists a non-empty complete extension.
\end{theorem}
\begin{proof}
Recall that $G$ is the $\subseteq$-least complete extension (Definition \ref{def:pre_grounded}, page \pageref{def:pre_grounded}). If $G\neq\es$, then the result follows as $G\in COMP$.

If $G=\es$, then $U=\es$ by Corollary \ref{cor:empty_grounded} (page \pageref{cor:empty_grounded}), and the AF is not well-founded by the contrapositive of Corollary \ref{cor:wf_unattacked} (page \pageref{cor:wf_unattacked}). As the AF is limited controversial, there are no infinite sequences of arguments $\set{a_i}_{i\in\nat}$ such that $a_{i+1}$ is controversial w.r.t. $a_i$. Therefore, all such sequences must terminate. Therefore, there is some argument $a\in A$ such that for all $b\in A$, $b$ is not controversial w.r.t. $a$.

Let $E_0=\set{a}$. For $i\in\nat$ Let $E_{i+1}:=E_i\cup D_i$, where $D_i\subseteq A$ is a $\subseteq$-minimal defense set of $E_i$ (Definition \ref{def:defense_set}). We prove by strong induction on $i$ that $E_i\in CF$ and each argument in $E_i$ indirectly defends $a$.
\begin{enumerate}
\item (Base) As our AF is limited controversial, then there are no self-attacking arguments by Corollary \ref{cor:LC_no_self_attack} (page \pageref{cor:LC_no_self_attack}). Hence, $E_0=\set{a}\in CF$. By Corollary \ref{cor:no_loop_indir_self_def} (page \pageref{cor:no_loop_indir_self_def}), $a$ indirectly defends itself.
\item (Inductive) Assume that $E_k\in CF$ and all arguments in $E_k$ defend $a$, for all $0\leq k\leq i$. As $U=\es$, every argument in $E_k$ is attacked by some other argument. We can construct a $\subseteq$-minimal defense set $D_k$ such that $E_k^--E_k^+\subseteq D_k^+$. By Corollary \ref{cor:defense_indirectly_defends}, there are arguments in $D_k$ that indirectly defend the arguments in $E_k$. As $D_k$ is $\subseteq$-minimal, we may assume that all arguments in $D_k$ indirectly defend the arguments in $E_k$. In particular, there is an even-length path from all arguments in $D_k$ to $a$. By the inductive hypothesis, all arguments in $E_k\cup D_k=:E_{k+1}$ indirectly defends $a$.

To show that $E_{k+1}\in CF$, assume for contradiction that $E_{k+1}\notin CF$. There are some $b,c,\in E_{k+1}$ such that $R(b,c)$. But as each argument in $E_{k+1}$ indirectly defends $a$, then $b$ also indirectly attacks $a$ and hence $b$ is controversial w.r.t. $a$ -- contradiction. Therefore, $E_{k+1}\in CF$.
\end{enumerate}

As $\set{E_i}_{i\in\nat}$ is a chain in $CF$, let $L:=\bigcup_{i\in\nat}E_i$ be its limit. As each $E_i\in CF$, by Theorem \ref{thm:cf_omega_complete} (page \pageref{thm:cf_omega_complete}), $L\in CF$. Let $b\in L$ and assume $R(c,b)$ for some $c\in A$. There is some $i\in\nat^+$ such that $b\in E_i=E_{i-1}\cup D_{i-1}\subseteq d\pair{E_i\cup D_i}$ by Corollary \ref{cor:defense_S_fills_defence}. Therefore, either $c$ is attacked by $E_{i}$, or $D_{i}$, which means $c$ is attacked by $L$. Therefore, $L$ is self-defending and hence $L\in ADM$. As $a\in L$, $L\neq\es$. By Theorem \ref{thm:iterate_d_adm_to_comp} (page \pageref{thm:iterate_d_adm_to_comp}), we iterate $d$ on $L$ a transfinite number of times to get our desired non-empty complete extension, $C$.
\end{proof}

\begin{theorem}\label{thm:lem35}
\cite[Lemma 35]{Dung:95} Let AF be uncontroversial. Let $a\notin G\cup G^+$. There exists $C_1,C_2\in COMP$ such that $a\in C_1\cap C_2^+$.
\end{theorem}
\begin{proof}
Let $A':=A-\pair{G\cup G^+}\subseteq A$. By assumption, $a\in A'$ so $A'\neq\es$. Consider the induced sub-framework $\mathcal{A}'$ on $A'$, which is also uncontroversial (because $A$ does not have any controversial arguments). We repeat the proof of Theorem \ref{thm:LC_non_empty_comp} on the uncontroversial (and therefore limited controversial) argument $a$ in the uncontroversial sub-framework AF' to construct a non-empty complete extension $C$ of AF' containing $a$. Consider $C_1:=G\cup C$. Clearly, $a\in C_1$. We show that $C_1$ is a complete extension of AF.

Now as $a\notin G$ and $U\subseteq G$, $a\notin U$ and hence $a$ is attacked by some argument $b\in A'$. As AF' is uncontroversial we repeat the proof of Theorem \ref{thm:LC_non_empty_comp} on $b$ to construct a complete extension $C'$ of $AF'$ such that $b\in C'$ and $a\in C'^+$. Therefore, $C_2:=G\cup C'$ is our desired complete extension.
\end{proof}

\begin{theorem}\label{thm:comp_in_adm}
$COMP\subseteq ADM$. The converse is not true in general.
\end{theorem}
\begin{proof}
If $S\in COMP$ then $S$ is cf and $S=d(S)$, which implies that $S\subseteq d(S)$ and hence $S\in ADM$.

The converse is not true in general: from Example \ref{eg:BH_eg221} (page \pageref{eg:BH_eg221}), $\es\in ADM$ but $\es\notin COMP$.
\end{proof}

\begin{corollary}
$d$ is closed on $COMP$.
\end{corollary}
\begin{proof}
Trivially, $S\in COMP=CF\cap F_d$ and hence $S\in F_d$ so $d(S)=S\in COMP$.
\end{proof}

\noindent Therefore, $d:COMP\to COMP$ is well-defined.

Now there are some situations where AFs can fail to infer anything.

\begin{corollary}\label{cor:empty_comp}
$ADM=\set{\es}$ iff $COMP=\set{\es}$.
\end{corollary}
\begin{proof}
($\Rightarrow$) As $\es\neq COMP\subseteq ADM=\set{\es}$ by Corollary \ref{cor:comp_exists} and Theorem \ref{thm:comp_in_adm}, the result follows.

($\Leftarrow$, contrapositive) If $ADM\neq\set{\es}$, then as $ADM\neq\es$, there is some $S\in ADM$ such that $S\neq\es$. By Corollary \ref{cor:ADM_complete_closure} (page \pageref{cor:ADM_complete_closure}), there is some $L\in COMP$ such that $S\subseteq L$. Clearly, $L\neq\es$ because $S\neq\es$. Therefore, $COMP\neq\set{\es}$.
\end{proof}

\begin{example}
Recall Example \ref{eg:ab_self_attack} (page \pageref{eg:ab_self_attack}). Clearly $CF=\set{\es,\set{b}}$ and $ADM=\set{\es}$ because $b$ cannot be self-defending. Therefore, $COMP=\set{\es}$ as well. This is consistent with that $U=\es$ hence $G=\es$ (Corollary \ref{cor:empty_grounded}, page \pageref{cor:empty_grounded}).
\end{example}

\subsection{Lattice Theoretic Properties}

\noindent We can generalise Theorem \ref{thm:iterate_d_adm_to_comp} (page \pageref{thm:iterate_d_adm_to_comp}) to arbitrary chains of complete extensions.

\begin{theorem}\label{thm:COMP_chain_comp}
$\ang{COMP,\subseteq}$ is chain complete.
\end{theorem}
\begin{proof}
Let $\ang{Ch,\subseteq}\subseteq\ang{COMP,\subseteq}$ be any chain. By Theorem \ref{thm:comp_in_adm}, $\ang{Ch,\subseteq}$ is also a chain (and hence a directed set) in $\ang{ADM,\subseteq}$. By Corollary \ref{cor:adm_dir_set} (page \pageref{cor:adm_dir_set}), $S:=\bigcup Ch\in ADM$. To show $S\in COMP$, we need to show $d(S)\subseteq S$. This is done as follows:
\begin{align*}
& a\in d(S)\\
\Leftrightarrow& a^-\subseteq S^+\\
\Leftrightarrow& a^-\subseteq\pair{\bigcup_{C\in Ch}C}^+\\
\Leftrightarrow& a^-\subseteq\bigcup_{C\in Ch}C^+\text{ by Corollary \ref{cor:cup_cap_plus_minus} (page \pageref{cor:cup_cap_plus_minus}),}\\
\Leftrightarrow&\pair{\exists C\in Ch}a^-\subseteq C^+\text{ as the ``$\Rightarrow$'' direction follows as $Ch$ is a chain,\footnotemark}\\
\Leftrightarrow&\pair{\exists C\in Ch}a\in d(C)\\
\Leftrightarrow&\pair{\exists C\in Ch}a\in C\\
\Leftrightarrow&a\in\bigcup Ch=S,
\end{align*}
\footnotetext{See Footnote \ref{fn:chain_reasoning} (page \pageref{fn:chain_reasoning}).}
and therefore, $d(S)\subseteq S$, so $S\in COMP$. Therefore, $\ang{COMP,\subseteq}$ is chain complete.
\end{proof}

\noindent It immediately follows that $\ang{COMP,\subseteq}$ is $\omega$-complete. We can generalise further:

\begin{theorem}
$\ang{COMP,\subseteq}$ is directed complete.
\end{theorem}
\begin{proof}
Let $\ang{\mathcal{D},\subseteq}\subseteq\ang{COMP,\subseteq}$ be any directed subset. By Theorem \ref{thm:comp_in_adm}, $\ang{\mathcal{D},\subseteq}$ is also a directed subset of $\ang{ADM,\subseteq}$. By Corollary \ref{cor:ADM_cpo1} (page \pageref{cor:ADM_cpo1}), $\bigcup\mathcal{D}\in ADM$. Therefore, it is sufficient to show $d\pair{\bigcup\mathcal{D}}\subseteq\bigcup\mathcal{D}$. We can apply the same reasoning in Theorem \ref{thm:COMP_chain_comp} and by invoking that $\mathcal{D}$ is a directed subset. Therefore, $\bigcup\mathcal{D}$ is a fixed point of $d$ and hence $\bigcup\mathcal{D}\in COMP$. As $\mathcal{D}$ is an arbitrary directed subset, the result follows.
\end{proof}


\begin{theorem}\label{thm:COMP_glb_non_es}
Let $\es\neq\mathcal{C}\subseteq COMP$, then there exists $S\in COMP$ such that $S$ is the $\subseteq$-glb of $\mathcal{C}$.
\end{theorem}
\begin{proof}
If $\bigcap\mathcal{C}\in COMP$, then the result follows by choosing $C=\bigcap\mathcal{C}$. 

Otherwise,\footnote{Note that from Example \ref{eg:fri_COMP} (page \pageref{eg:fri_COMP}), if we choose $\mathcal{C}=\set{\set{a,e},\set{b,e}}$, we have $\bigcap\mathcal{C}=\set{e}\notin COMP$.} similar to the proof of Corollary \ref{cor:non-empty_adm_glb} (page \pageref{cor:non-empty_adm_glb}), let $LB:=\set{T\in ADM\:\vline\:\pair{\forall C'\in\mathcal{C}}T\subseteq C'}$. By Corollary \ref{cor:empty_is_admissible} (page \pageref{cor:empty_is_admissible}), $\es\in LB$ so $LB\neq\es$. Let $T\in LB$ and $C'\in\mathcal{C}$ be arbitrary. As $T\subseteq C'$, then $d\pair{T}\subseteq d\pair{C'}=C'$, hence $\pair{\forall C'\in\mathcal{C}}d\pair{T}\subseteq C'$ as well. Therefore, $d(T)\in LB$ and hence $d:LB\to LB$ is well-defined. Now let $S:=\bigcup LB$, which for the same reasons as Corollary \ref{cor:non-empty_adm_glb}, it follows that $S\in ADM$. Furthermore, as $S=\bigcup_{T\in LB}T$ and $\pair{\forall C'\in\mathcal{C}}T\subseteq C'$ and hence $S\in LB$. Therefore, $d(S)\in LB$. As $S\in ADM$ we have $S\subseteq d(S)$, and as $S$ is $\subseteq$-maximal in $LB$, $d(S)\subseteq S$. This establishes that $S\in COMP$.

Therefore, we have found a complete extension $S$ that is a lower bound of $\mathcal{C}$ as $S\in LB$, and is a greatest such lower bound because for any $S'\in LB$, $S'\subseteq S$ by definition. The result follows.
\end{proof}


\begin{theorem}
\cite[Theorem 25(3)]{Dung:95} $\ang{COMP,\subseteq}$ is a complete semilattice.
\end{theorem}
\begin{proof}
This follows from Theorems \ref{thm:COMP_chain_comp} and \ref{thm:COMP_glb_non_es}.
\end{proof}

\noindent The main difference between the proofs of Theorem \ref{thm:COMP_glb_non_es} and \cite[Theorem 25(3)]{Dung:95} is that the latter applies transfinite induction (Corollary \ref{cor:ADM_complete_closure}) to locate the complete extension that is the $\subseteq$-glb, while Theorem \ref{thm:COMP_glb_non_es} avoids transfinite induction by showing $d:LB\to LB$ and applies maximality to show that $\bigcup LB$ is indeed the $\subseteq$-glb complete extension.

\subsection{Summary}

\begin{itemize}
\item $S\subseteq A$ is complete iff it is both conflict-free and a fixed point of $d$. The set of all complete extensions is $COMP(AF)$ or just $COMP$.
\item It can be shown that $COMP\neq\es$, and the $\subseteq$-smallest complete extension is the least fixed point of $d$.
\item If $U=\es$ then $\es\in COMP$. For all $C\in COMP$, $U\subseteq C$.
\item Clearly, $COMP\subseteq ADM$, and $d:COMP\to COMP$. For $S\in ADM$ there exists a $\subseteq$-least $C\in COMP$ such that $S\subseteq C$. Further, $ADM=\set{\es}$ iff $COMP=\set{\es}$. 
\item If the underlying AF is limited controversial then there is a non-empty complete extension.
\item Let $G\in COMP$ be the $\subseteq$-minimal complete extension. If the underlying AF is uncontroversial then for $a\notin G\cup G^+$, there are complete extensions $C_1$ and $C_2$ such that $a\in C_1\cap C_2^+$.
\item $\ang{COMP,\subseteq}$ is a complete semilattice that is also directed complete.
\end{itemize}

\newpage

\section{Preferred, Stable and Grounded Extensions}

We now discuss the most important types of complete extensions used to draw inferences from AFs.

\subsection{Preferred Extensions}\label{sec:PREF}

\subsubsection{Definition}

\begin{definition}\label{def:preferred_ext}
\cite[Definition 7]{Dung:95} A \textbf{preferred extension} is a $\subseteq$-maximal admissible set.
\end{definition}

\noindent Preferred extensions are thus the (order-theoretic) largest admissible sets. We will prove that they are also complete extensions in what follows.

\begin{corollary}\label{cor:pref_max_adm}
We have that $PREF=\max_{\subseteq}ADM$.
\end{corollary}
\begin{proof}
Immediate by Definition \ref{def:preferred_ext}.
\end{proof}

\begin{definition}
Given an AF, let $PREF\subseteq\pow\pair{A}$ denote the set of preferred extensions.
\end{definition}

\noindent If we need to make the underlying AF $\mathcal{A}$ explicit, we may write $PREF\pair{\mathcal{A}}$.

\begin{example}
(Example \ref{eg:nixon} continued, page \pageref{eg:nixon}, from \cite[Example 9]{Dung:95}) In the Nixon diamond, we have $PREF=\set{\set{a},\set{b}}$.
\end{example}

\begin{example}
(Example \ref{eg:Israel_Arab} continued, page \pageref{eg:Israel_Arab}, from Example \cite[Example 8]{Dung:95}) In this case, we have $PREF=\set{\set{a,c}}$.
\end{example}

\begin{example}
(Example \ref{eg:BH_eg221} continued, page \pageref{eg:BH_eg221}) Here, $A=\set{a,b,c,f,e}$ and $R=\set{(a,b),(c,b),(c,f),(f,c),(f,e),(e,e)}$, so $PREF = \set{\set{a,c},\set{a,f}}$.
\end{example}

\noindent Do not confuse the meaning of ``$\subseteq$-maximal'' with ``containing the most arguments'', as the following example shows:

\begin{example}\label{eg:accrual}
Consider the AF with arguments $A=\set{a,b,c}$ and attacks $R=\set{(a,b),(b,a),(b,c),(c,b)}$. This AF is depicted in Figure \ref{fig:accrual}.

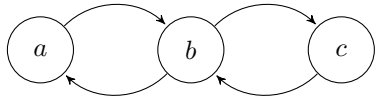
\begin{figure}[H]
\begin{center}
\begin{tikzpicture}[>=stealth',shorten >=1pt,node distance=2cm,on grid,initial/.style    ={}]
\tikzset{mystyle/.style={->,relative=false,in=0,out=0}};
\node[state] (a) at (0,0) {$ a $};
\node[state] (b) at (2,0) {$ b $};
\node[state] (c) at (4,0) {$ c $};
\draw [->] (a) to [out = 45, in = 135] (b);
\draw [->] (b) to [out = 225, in = -45] (a);
\draw [->] (b) to [out = 45, in = 135] (c);
\draw [->] (c) to [out = 225, in = -45] (b);
\end{tikzpicture}
\caption{The AF from Example \ref{eg:accrual}.}\label{fig:accrual}
\end{center}
\end{figure}
\noindent In this case, $PREF=\set{\set{a,c},\set{b}}$.
\end{example}

\subsubsection{Existence and Lattice-Theoretic Structure}

\begin{corollary}\label{cor:ADM2PREF}
\cite[Theorem 11(2)]{Dung:95} For any $S\in ADM$ there exists a $E\in PREF$ such that $S\subseteq E$. 
\end{corollary}
\begin{proof}
Given $S\in ADM$, consider its upper-set $\mathcal{S}:=\set{T\in ADM\:\vline\:S\subseteq T}$. We show $\mathcal{S}$ has a $\subseteq$-maximal element, which would be our $E$.

Let $\mathcal{D}\subseteq\mathcal{S}$ be a directed subset. As $\mathcal{D}$ consists of sets of the form $T\in ADM$ and $S\subseteq T$, then $S\subseteq \bigcup\mathcal{D}\in ADM$ by Corollary \ref{cor:adm_dir_set} (page \pageref{cor:adm_dir_set}). But as $S\subseteq\bigcup\mathcal{D}$, we have that $\bigcup\mathcal{D}\in\mathcal{S}$. As $\mathcal{D}$ is any directed subset of $\mathcal{S}$, and every directed subset is a chain, and $\mathcal{D}$ has an upper bound $\bigcup\mathcal{D}$, we have shown that every chain in $\mathcal{S}$ has an upper bound. Therefore, by Zorn's lemma, $\mathcal{S}$ has a maximal element, which we will call $E$. As $\mathcal{S}$ is a $\subseteq$-up-set, we have that $E\in PREF$. The result follows.
\end{proof}

\noindent We offer another proof of Corollary \ref{cor:ADM2PREF} that makes use of the fundamental lemma (Lemma \ref{lem:dfl}, page \pageref{lem:dfl}).


\begin{proof}
(Of Corollary \ref{cor:ADM2PREF}) Let $S\in ADM$, so $S\subseteq d(S)$. Consider the set $d(S)-S$ and invoke the well-ordering theorem such that we can write its elements as an ordinal-indexed list.\footnote{This is because we are not making any assumptions that the underlying AF is finite. Further, the well-ordering theorem is equivalent to the axiom of choice.} Starting from the element indexed by the least ordinal, say $a_0$, by the fundamental lemma, $S\cup\set{a_0}\in ADM$. We can append such elements one by one to $S$ and the fundamental lemma guarantees that the result is still in $ADM$. For the limit case, we apply the fact that $ADM$ is a dcpo (Corollary \ref{cor:ADM_cpo1}, page \pageref{cor:ADM_cpo1}) such that the union is still in $ADM$. More precisely, this is done via transfinite induction on the ordinal-valued indices of the elements in $d(S)-S$.
\begin{enumerate}
\item (Base) $S\in ADM$
\item (Successor) $S\cup\set{a_0,a_1,\ldots,a_{\beta}}\in ADM$ and $a_{\beta+1}\in d(S)-S$ implies that $S\cup\set{a_0,a_1,\ldots,a_{\beta},a_{\beta+1}}\in ADM$ by the fundamental lemma.
\item (Limit) If $S\cup\set{a_{\beta}}_{\beta<\gamma}\in ADM$ where $\gamma$ is a limit ordinal, then the sequence of sets $T_{\beta}:=S\cup\set{a_{\beta}}_{\beta<\gamma}$ is an ascending chain of admissible sets, such that $\bigcup_{\beta<\gamma}T_{\beta}\in ADM$ by Corollary \ref{cor:ADM_cpo1}.
\end{enumerate}
Therefore, transfinite induction shows that $d(S)\in ADM$. We then apply transfinite induction a second time to show that for a suitably large ordinal $\gamma$, $E:=d^\gamma(S)\in ADM$ (recall also Theorem \ref{thm:TFI_iterate_d_ADM}, page \pageref{thm:TFI_iterate_d_ADM}), and is also a fixed point of $d$ (Theorem \ref{thm:iterate_d_adm_to_comp}, page \pageref{thm:iterate_d_adm_to_comp}). Let $a\in A - E$, then $a\notin d(E)$ so $a^-\not\subseteq E^+$. Therefore, $E\cup\set{a}\notin ADM$. Therefore, $E\in PREF$ and $S\subseteq E$.
\end{proof}

The following result applies Corollaries \ref{cor:empty_is_admissible} (page \pageref{cor:empty_is_admissible}) and \ref{cor:ADM2PREF} to give a different proof of Corollary \ref{cor:ADM_has_max} (page \pageref{cor:ADM_has_max}).

\begin{corollary}\label{cor:pref_exist}
\cite[Corollary 12]{Dung:95} Any AF has a preferred extension, i.e. $PREF\neq\es$.
\end{corollary}
\begin{proof}
By Corollary \ref{cor:empty_is_admissible}, we can choose $S=\es\in ADM$ and invoke Corollary \ref{cor:ADM2PREF}. Alternatively, this follows from Corollary \ref{cor:ADM_has_max} (page \pageref{cor:ADM_has_max}).
\end{proof}

Note that in proving Corollary \ref{cor:ADM2PREF}, we have used either Zorn's lemma or the well-ordering theorem, and hence the axiom of choice (AC), to prove that every AF has a preferred extension. Clearly, AC is sufficient to demonstrate this. It is reasonable to ask whether it is \textit{necessary}. In Appendix \ref{app:AC} (page \pageref{app:AC}), we show that AC is also necessary, hence the assertion that all AFs have preferred extensions is equivalent to AC.

\begin{theorem}\label{thm:pref_in_comp}
\cite[Theorem 25(1)]{Dung:95} $PREF\subseteq COMP$, and the converse is not true.
\end{theorem}
\begin{proof}
Let $S\in PREF$, then $S\in\max_{\subseteq}ADM$. We need to show $d(S)\subseteq S$. Let $a\in d(S)$, then $S\cup\set{a}\in ADM$ by Dung's fundamental lemma (Lemma \ref{lem:dfl}, page \pageref{lem:dfl}). Trivially, $S\subseteq S\cup\set{a}$, but as $S\in\max_{\subseteq}ADM$, $S\cup\set{a}\subseteq S$ and hence $a\in S$. Therefore, as $a$ is arbitrary, $d(S)\subseteq S$. Therefore, $S\in COMP$.

The converse is not true. From Example \ref{eg:BH_eg221} (page \pageref{eg:BH_eg221}), $\set{a}\in COMP$ but $\set{a}\notin PREF$.
\end{proof}

\begin{theorem}\label{thm:pref_max_comp}
$PREF=\max_{\subseteq}COMP$.
\end{theorem}
\begin{proof}
($\Rightarrow$) Let $P\in PREF$. Let $C\in COMP$ such that $P\subseteq C$. Clearly, $C\in ADM$ so $C\subseteq P$. Therefore, $C=P$ and hence $P\in\max_{\subseteq}COMP$.

($\Leftarrow$, contrapositive) Let $P\notin PREF$, then either $P\notin ADM$, or $P\in ADM$ and is not $\subseteq$-maximal.
\begin{enumerate}
\item If $P\notin ADM$, then $P\notin COMP$, which means $P\notin \max_{\subseteq} COMP$.
\item If $P\in ADM$ and is not $\subseteq$-maximal, then by Corollary \ref{cor:ADM2PREF}, there is some $E\in PREF$ such that $P\subseteq E$. By assumption, $P\subset E$. Clearly, $E\in COMP$. Assume that $P\in COMP$ (else $P\notin\max_{\subseteq} COMP$ and the result follows), we have that $P\subset E$ and hence $P\notin\max_{\subseteq} COMP$. 
\end{enumerate}
The result follows.
\end{proof}

\noindent This means that every complete extension is contained in a preferred extension.

\begin{example}
\cite[Exercise 7(a)]{Caminada:08} From Figure \ref{fig:Caminada_ex2} (page \pageref{fig:Caminada_ex2}), $COMP=\set{\set{a,c}}$ and hence $PREF=\set{\set{a,c}}$.
\end{example}

\begin{example}
\cite[Exercise 7(b)]{Caminada:08} From Figure \ref{fig:Caminada_ex6b} (page \pageref{fig:Caminada_ex6b}), $COMP=\set{\set{a},\set{a,c,f},\set{a,e}}$ and hence $PREF=\set{\set{a,c,f},\set{a,e}}$.
\end{example}

\begin{example}
\cite[Exercise 7(c)]{Caminada:08} From Figure \ref{fig:Caminada_ex6c} (page \pageref{fig:Caminada_ex6c}), $COMP=\set{\set{b,e}}$ and hence $PREF=\set{\set{b,e}}$.
\end{example}

\begin{example}
\cite[Exercise 7(d)]{Caminada:08} From Figure \ref{fig:Caminada_ex3} (page \pageref{fig:Caminada_ex3}), $COMP=\set{\es,\set{a},\set{b,e}}$ and hence $PREF=\set{\set{a},\set{b,e}}$.
\end{example}

\begin{corollary}\label{cor:empty_pref}
$COMP=\set{\es}$ iff $PREF=\set{\es}$.
\end{corollary}
\begin{proof}
($\Rightarrow$) If $COMP=\set{\es}$, then $\max_{\subseteq}COMP=\set{\es}=PREF$.

($\Leftarrow$, contrapositive) Let $COMP\neq\set{\es}$. As $COMP\neq\es$, either $\es\in COMP$ or $\es\notin COMP$. If $\es\in COMP$ and $COMP\neq\set{\es}$, then there is some non-empty $C\in COMP$. Therefore, $C\in\max_{\subseteq}COMP=PREF$ and hence $PREF\neq\set{\es}$. If $\es\notin COMP$, then as $PREF\subseteq COMP$, $\es\notin PREF$ either so $PREF\neq\set{\es}$. In either case, $PREF\neq\set{\es}$.
\end{proof}

\begin{theorem}
If the AF is isomorphic to $C_n$ where $n$ is even, then $\abs{PREF}=2$.
\end{theorem}
\begin{proof}
WLOG let $A=\set{a_1,\ldots, a_n}$ where $n$ is even. Partition this into two sets $S_1:=\set{a_i\in A\:\vline\:1\leq i\leq n,\:i\text{ odd}}$ and $S_2:=\set{a_i\in A\:\vline\:1\leq i\leq n,\:i\text{ even}}$. Clearly both are cf sets as $R\pair{a_i,\:a_{i+1}}$ for $1\leq i\leq n-1$ and $R\pair{a_n,a_1}$ are the only attacks. Further, they are self-defending: for any $a_{k-1}$ attacking $S_i$ for $i=1,2$, there is some $a_{k-2}\in S_i$ that attacks $a_{k-1}$. These are clearly $\subseteq$-maximal. Therefore, $PREF=\set{S_1,S_2}$ and hence $\abs{PREF}=2$.
\end{proof}

\noindent As $C_n$ for $n$ even are bipartite graphs, any AF isomorphic to $C_n$ will have each partition forming a preferred extension.

\begin{theorem}\label{thm:odd_cycle_empty_PREF}
If the AF is isomorphic to $C_n$ where $n$ is odd, then $PREF=\set{\es}$. 
\end{theorem}
\begin{proof}
WLOG let $C_n=\set{a_1,\ldots,a_n}$ for $n$ odd. Let $S\subseteq A$ be a non-empty cf set, which can have at most $\lfloor\abs{S}/2\rfloor$ arguments. For the $\subseteq$-largest such set, there is at least a pair of arguments (say $a$ and $b$) in $S$ whose path length is 3 (say from $a$ to $b$), which implies that there is some argument (in this case $b$) in $S$ whose attacker is not attacked by $S$ (as $a$ cannot reach it). Therefore, any such $S$ cannot be self-defending, and any subset of this $S$ cannot be self-defending either. However, as $U=\es$, the only admissible extension is $\es$ and hence $PREF=\set{\es}$.
\end{proof}

\noindent So AFs whose underlying digraph is an odd cycle will have $PREF=\set{\es}$, and hence $COMP=ADM=\set{\es}$ (Corollary \ref{cor:empty_pref} above and Corollary \ref{cor:empty_comp}, page \pageref{cor:empty_comp}).

\begin{theorem}\label{thm:no_even_cycle_pref_unique}
\cite[Theorem 2.3.1]{EoA} If the underlying AF is finite and has no even cycle, then the preferred extension is unique.
\end{theorem}
\begin{proof}
(From \cite[Theorem 2.6]{BenchCapon:14}, contrapositive) Suppose there are at least two preferred extensions, and pick two of them $P$ and $Q$, where $P\neq Q$. Consider the sets $P-Q$ and $Q-P$, neither of which can be empty else $P\subseteq Q$ or $Q\subseteq P$, which contradicts the fact that $P$ and $Q$ are distinct preferred extensions. As the AF is finite, WLOG let $P-Q=\set{p_1,\ldots, p_n}$ and $Q-P=\set{q_1,\ldots, q_m}$ for some $n,m\in\nat^+$. Let $p\in P-Q$ and $q\in Q-P$. Either $R(p,q)$ or $R(q,p)$, for if neither, then (say) $p$ can be added to $Q$ making $Q\cup\set{p}\supset Q$, contradicting that $Q\in PREF$.

We now use this to construct our even cycle. WLOG suppose $R(p,q)$, then as $q\in Q-P\subseteq Q\in PREF\subseteq ADM$, there is some $r_1\in Q$ such that $R(r_1,p)$. If $r_1=q$, then $\set{p,q}$ forms an even cycle. If $r_1\neq q$, then as $p\in P-Q\subseteq P\in PREF\subseteq ADM$, there exists an $r_2\in P$ such that $R(r_2,r_1)$. If $r_2=p$ then $\set{p,r_1}$ forms an even cycle. If $r_2\neq p$, then by the same argument as above by invoking an appropriate counter-attack which is guaranteed to exist as preferred extensions are self-defending, we will yield an even cycle as we alternate between the preferred extensions $P$ and $Q$. This process must terminate as we have assumed that the AF is finite. Therefore, the AF has an even cycle.
\end{proof}

The following result strengthens Theorem \ref{thm:no_even_cycle_pref_unique}.

\begin{theorem}\label{thm:pref_only_empty}
\cite[Corollary 2.3.1]{EoA} Let $\ang{A,R}$ be a finite AF such that $U=\es$ and has no even cycle, then $PREF=\set{\es}$.
\end{theorem}
\begin{proof}
(Contrapositive) Assume that our AF is finite and $U=\es$. Assume that $PREF\neq\set{\es}$, we prove that there must exist an even cycle.

As $PREF\neq\es$, we have some $S\in PREF$ such that $S\neq\es$. Clearly, $S$ is a finite set. Choose $a_0\in S$. As $U=\es$, there is some $b_0\in A-S$ such that $R\pair{b_0,a_0}$. As $U=\es$, and $S\in PREF$, there is some $a_1\in S$ such that $R\pair{a_1,b_0}$. If $a_1=a_0$, then we have an even cycle $a_0$ attacks $b_0$ attacks $a_0$ and the result follows.

Otherwise, $a_1\neq a_0$, and as before, we have some $b_1\in A-S$ attacking $a_1$, and some $a_2\in S$ attacking $b_1$. If $a_2=a_0$, then we have an even cycle $a_0$ attacks $b_1$ attacks $a_1$ attacks $b_0$ attacks $a_0$, and the result follows. Similarly, if $a_2=a_1$, we have an even cycle again.

We cannot extend this path $(a_0,b_0,a_1,b_1,a_2,b_2,\ldots)$ such that $a_i\notin\set{a_j}_{j<i}$ indefinitely, because $A$ is finite. Therefore, there exists some $a_k$ that repeats twice in this path, and this guarantees the existence of an even cycle.
\end{proof}



\begin{definition}
Let $S$ be a set. We say a family $\mathcal{F}$ of sets \textbf{covers} $S$ iff $S\subseteq\bigcup\mathcal{F}$.
\end{definition}

\begin{corollary}\label{cor:if_pref_cover_then_rg1}
\cite[Lemma 1]{Croitoru:13} For any AF $\ang{A,R}$, if $PREF$ covers $A$, then $\bigcap PREF=U$. If the hypothesis is not true, then the consequent may or may not be true.
\end{corollary}
\begin{proof}
($\Leftarrow$) By Theorem \ref{thm:pref_in_comp} (page \pageref{thm:pref_in_comp}), $PREF\subseteq COMP$. By Corollary \ref{cor:U_subset_COMP} (page \pageref{cor:U_subset_COMP}), for all $C\in COMP$, $U\subseteq C$. Therefore, it follows that for all $E\in PREF$, $U\subseteq E$. Therefore, $U\subseteq\bigcap_{E\in PREF}E=\bigcap PREF$.

($\Rightarrow$) Let $a\in \bigcap PREF$. Assume for contradiction that $a\notin U$. Therefore, there is some $b\in a^-$. But as this $b\in A$, there is some $E\in PREF$ such that $b\in E$ by our hypothesis. As $a\in\bigcap PREF$, clearly $a\in E$ as well. Therefore, we have found $a,b\in E\in PREF$ such that $b\in a^-$, which means $E\notin CF$ -- contradiction, as $PREF\subseteq CF$. Therefore, $a\in U$.

Now suppose that the preferred extensions do not cover $A$. We give two examples where $\bigcap PREF = U$ and $\bigcap PREF\neq U$.
\begin{enumerate}
\item In Example \ref{eg:BH_eg221} (page \pageref{eg:BH_eg221}), $PREF=\set{\set{a,c},\set{a,f}}$, which clearly does not cover $A=\set{a,b,c,f,e}$, but $\bigcap PREF=\set{a}=U$.
\item In Example \ref{eg:simple_reinstatement} (page \pageref{eg:simple_reinstatement}), $PREF=\set{\set{a,c}}$ and $A=\set{a,b,c}$. Therefore, $PREF$ does not cover $A$, and $\bigcap PREF=\set{a,c}\neq U=\set{c}$.
\end{enumerate}
This shows the result.
\end{proof}

\begin{corollary}
\cite[Proposition 6]{Coste:05} For all symmetric AFs, if there are no self-attacking arguments, then $PREF$ covers $A$.
\end{corollary}
\begin{proof}
Let $\ang{A,R}$ be such an AF. Let $a\in A$, then $\set{a}\in CF$ because there are no self-attacking arguments. By Theorem \ref{thm:sym_AF_CF_ADM_equal} (page \pageref{thm:sym_AF_CF_ADM_equal}), $\set{a}\in ADM$. By Theorem \ref{cor:ADM2PREF} (page \pageref{cor:ADM2PREF}), there is a preferred extension that contains $\set{a}$. As $a$ is arbitrary, then every argument in $A$ belongs to some preferred extension $P\subseteq A$, hence $PREF$ covers $A$.
\end{proof}


The lattice-theoretic properties are trivial for $PREF$.

\begin{corollary}\label{cor:PREF_antichain}
$\ang{PREF,\subseteq}$ is an antichain.
\end{corollary}
\begin{proof}
Immediate, as $PREF=\max_{\subseteq}ADM$, by Corollary \ref{cor:pref_max_adm} (page \pageref{cor:pref_max_adm}).
\end{proof}

\subsubsection{Relationship with Naive Extensions}\label{sec:naive_and_preferred}

Recall that naive extensions are $\subseteq$-maximal $CF$ sets. Of course, $PREF\subseteq CF$ and $PREF=\max_{\subseteq}COMP$. How are preferred and naive extensions related? It turns out that neither implies the other.

\begin{lemma}
Preferred extensions are not always naive extensions.
\end{lemma}
\begin{proof}
Consider the AF whose underlying digraph is $C_3$, i.e. $A=\set{a,b,c}$ and $R=\set{(a,b),(b,c),(c,a)}$. Clearly, $CF=\set{\es,\set{a},\set{b},\set{c}}$ and hence $NAI=\set{\set{a},\set{b},\set{c}}$. However, by Theorem \ref{thm:odd_cycle_empty_PREF}, $PREF=\set{\es}$. Therefore, preferred extensions are not always naive extensions.
\end{proof}

\begin{lemma}\label{lem:nai_not_pref}
Naive extensions are not always preferred extensions.
\end{lemma}
\begin{proof}
Consider the AF $A=\set{a,b,c,d}$ and $R=\set{(a,b),(a,c),(b,d),(c,d)}$. This AF is depicted in Figure \ref{fig:nai_not_pref}.

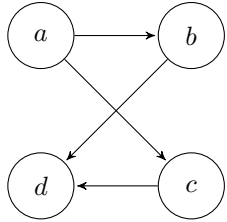
\begin{figure}[H]
\begin{center}
\begin{tikzpicture}[>=stealth',shorten >=1pt,node distance=2cm,on grid,initial/.style    ={}]
\tikzset{mystyle/.style={->,relative=false,in=0,out=0}};
\node[state] (a) at (0,0) {$ a $};
\node[state] (b) at (2,0) {$ b $};
\node[state] (c) at (0,-2) {$ c $};
\node[state] (d) at (2, -2) {$ d $};
\draw [->] (a) to (b);
\draw [->] (a) to (c);
\draw [->] (b) to (d);
\draw [->] (c) to (d);
\end{tikzpicture}
\caption{The AF from Lemma \ref{lem:nai_not_pref}}\label{fig:nai_not_pref}
\end{center}
\end{figure}
\noindent Clearly, $CF=\set{\es,\set{a},\set{b},\set{c},\set{d},\set{a,d},\set{b,c}}$. By inspection, we have that $NAI=\set{\set{a,d},\set{b,c}}$. However, $PREF=\set{\set{a,d}}$. Therefore, we have a naive extension $\set{b,c}$ that is not preferred. 
\end{proof}

\begin{corollary}\label{cor:sym_PREF_NAI_equal}
If an AF is symmetric, then $PREF=NAI$.
\end{corollary}
\begin{proof}
If an AF is symmetric, then by Theorem \ref{thm:sym_AF_CF_ADM_equal} (page \pageref{thm:sym_AF_CF_ADM_equal}), $CF=ADM$, hence $NAI=\max_{\subseteq}CF=\max_{\subseteq}ADM=PREF$.
\end{proof}

\subsection{Stable Extensions}\label{sec:STAB}

\subsubsection{Definition}

\begin{theorem}\label{thm:stable_TFAE}
\cite[Lemma 14]{Dung:95} For $S\subseteq A$, $S=n(S)$ iff $S$ is cf and $A-S\subseteq S^+$.
\end{theorem}
\begin{proof}
($\Rightarrow$) If $S=n(S)$, then $S=A-S^+$ by Definition \ref{def:neutrality_function}, which means $S^+=A-S$, hence $S$ attacks all arguments outside of it. Further, $S=n(S)$ implies $S\subseteq n(S)$, hence $S$ is cf by Definition \ref{definition:cf}.

($\Leftarrow$) Let $S$ be a cf set that attacks all arguments outside of it. The latter means $S^+=A-S$ and hence $S=A-S^+=n(S)$, hence $S=n(S)$ so $S$ must be a stable extension.
\end{proof}

\begin{definition}\label{def:stable_ext}
\cite[Definition 13]{Dung:95} We say $S\subseteq A$ is a \textbf{stable extension} iff $S$ satisfies any one of the two equivalent properties in Theorem \ref{thm:stable_TFAE}.
\end{definition}

\subsubsection{Existence}

\begin{definition}
Given an underlying AF, let $STAB\subseteq\pow\pair{A}$ denote the set of stable extensions.
\end{definition}

\noindent If the underlying AF $\mathcal{A}=\ang{A,R}$ needs to be explicitly specified, we can write $STAB(\mathcal{A})$ or $STAB(\ang{A,R})$.

\begin{example}
(Example \ref{eg:BH_eg221} continued, page \pageref{eg:BH_eg221}) For the AF $A=\set{a,b,c,f,e}$ and $R=\set{(a,b),(c,b),(c,f),(f,c),(f,e),(e,e)}$. We have $STAB = \set{\set{a,f}}$.
\end{example}

\begin{corollary}\label{cor:empty_AF_iff_STAB_has_empty}
\cite[Theorem 2.2.3]{EoA} $A=\es$ iff $\es\in STAB$.
\end{corollary}
\begin{proof}
$\es$ is stable iff $n\pair{\es}=\es=A$ by Corollary \ref{cor:neutrality_es} (page \pageref{cor:neutrality_es}).
\end{proof}

\begin{corollary}
\cite[Example 2.2.5]{EoA} It is possible for $STAB=\es$.
\end{corollary}
\begin{proof}
Let $\ang{A,R}$ be the 1-cycle, i.e. $A=\set{a}$ and $R=\set{\pair{a,a}}$. We have $\pow\pair{A}=\set{\es,\set{a}}$. Further, $n\pair{\set{a}}=\es$ and $n\pair{\es}=\set{a}$. Therefore, $n$ has no fixed points, hence by Theorem \ref{thm:stable_TFAE}, $STAB=\es$.
\end{proof}

\begin{theorem}\label{thm:STAB_in_PREF}
\cite[Lemma 15]{Dung:95} $STAB\subseteq PREF$, and the converse is not true in general.
\end{theorem}
\begin{proof}
This is trivial if $STAB=\es$, so assume that $STAB\neq\es$. If $S\in STAB$, then $S^+=A-S$ and hence any strict superset of $S$ cannot be in $CF$. Therefore, as $S$ is cf and self-defending (the latter follows from Corollary \ref{cor:fpn_fpd} (page \pageref{cor:fpn_fpd}), $S\in PREF$.

The converse is not true, e.g. in Example \ref{eg:BH_eg221} (page \pageref{eg:BH_eg221}), $\set{a,c}\in PREF$ but $\set{a,c}\notin STAB$.
\end{proof}



It also follows by Corollary \ref{cor:PREF_antichain} that $STAB$ is also a $\subseteq$-antichain thus its lattice-theoretic properties are trivial.

\begin{corollary}\label{cor:empty_pref_no_stab}
For a non-empty AF, if $PREF=\set{\es}$, then $STAB=\es$.
\end{corollary}
\begin{proof}
By Theorem \ref{thm:STAB_in_PREF}, $STAB\subseteq PREF=\set{\es}$. As $A\neq\es$, by Corollary \ref{cor:empty_AF_iff_STAB_has_empty}, $\es\notin STAB$. Therefore, $STAB=\es$.
\end{proof}

\begin{corollary}
For all $S\in STAB$, $U\subseteq S$, and it is not neceessarily the case that $U=S$.
\end{corollary}
\begin{proof}
From Corollary \ref{cor:U_subset_COMP} (page \pageref{cor:U_subset_COMP}) and Theorems \ref{thm:pref_in_comp} and \ref{thm:STAB_in_PREF}, we have that $STAB\subseteq PREF\subseteq COMP$, so for $S\in STAB$, $S\in COMP$ and hence $U\subseteq S$.

Alternatively, for any $a\in U$, $a^-=\es$ by definition so for any $S\subseteq A$, $a\in n(S)$. If $S\in STAB$, $a\in n(S)=S$ and the result follows.

For the converse, Example \ref{eg:simple_reinstatement} (page \pageref{eg:simple_reinstatement}), we have $\set{a,c}\in STAB$ but $U=\set{c}\subseteq\set{a,c}$.
\end{proof}

\begin{example}
\cite[Exercise 8(a)]{Caminada:08} From Figure \ref{fig:Caminada_ex2} (page \pageref{fig:Caminada_ex2}), we have that $PREF=\set{\set{a,c}}$ and because $a$ attacks $b$ and $c$ attacks $d$, we have $STAB=\set{\set{a,c}}$.
\end{example}

\begin{example}
\cite[Exercise 8(b)]{Caminada:08} From Figure \ref{fig:Caminada_ex6b} (page \pageref{fig:Caminada_ex6b}), we have that $PREF=\set{\set{a,c,f},\set{a,e}}$, and $STAB=PREF$.
\end{example}

\begin{example}
\cite[Exercise 8(c)]{Caminada:08} From Figure \ref{fig:Caminada_ex6c} (page \pageref{fig:Caminada_ex6c}), we have that $PREF=\set{\set{b,e}}$, but $STAB=\es$ because the self-attacking argument $a$ is not attacked by anything else.
\end{example}

\begin{example}
\cite[Exercise 8(d)]{Caminada:08} From Figure \ref{fig:Caminada_ex3} (page \pageref{fig:Caminada_ex3}), we have that $PREF=\set{\set{a},\set{b,e}}$, but $\set{a}$ does not attack all arguments outside of it while $\set{b,e}$ does, so $STAB=\set{\set{b,e}}$.
\end{example}

\subsubsection{Relationship with Naive and Preferred Extensions}\label{sec:stab_nai_pref}

\begin{corollary}\label{cor:stab_are_nai}
For all AFs, $STAB\subseteq NAI$.
\end{corollary}
\begin{proof}
The result is trivial for $STAB=\es$, so suppose we have an AF $\ang{A,R}$ with $S\in STAB$, then $S\in CF$. Let $a\in A - S$, then $a\in S^+$ hence $S\cup\set{a}\notin CF$. Therefore, $S$ is a $\subseteq$-maximal $CF$ set, i.e. $S\in NAI$.
\end{proof}

It is natural to ask when the converse of Theorem \ref{thm:STAB_in_PREF} is true. One sufficient condition is:

\begin{lemma}\label{lem:pref_naive_implies_stable}
Assume an AF with no self-attacking arguments, then if $PREF\subseteq NAI$ or $NAI\subseteq PREF$, then $PREF=STAB$. The assumption that there are no self-attacking arguments is necessary.
\end{lemma}
\begin{proof}
Let $\ang{A,R}$ be any AF without self-attacking arguments. It is sufficient to show that $PREF\subseteq STAB$ by Theorem \ref{thm:STAB_in_PREF}.

Assume $PREF\subseteq NAI$, then let $S\in PREF$ be naive. Therefore, for any $a\notin S$, $S\cup\set{a}\notin CF$. We have three possibilities:
\begin{enumerate}
\item $a$ is self-attacking, which we have excluded.
\item $a\in S^+$.
\item $a\in S^-$, but as $S\in PREF\subseteq ADM\subseteq SD$, $a\in S^-\subseteq S^+$ by Theorem \ref{thm:sd_equiv} (page \pageref{thm:sd_equiv}), therefore, $a\in S^+$.
\end{enumerate}
Therefore, in all cases, $a\in S^+$. This means for all $a\notin S$, $a\in S^+$, hence $A-S\subseteq S^+$, so $S=A-S^+=n(S)$, therefore $S\in STAB$. As $S\in PREF$ is arbitrary, we have shown $STAB=PREF$.

Assume $NAI\subseteq PREF$, then let $S\in NAI$ be preferred, i.e. $S\in SD$. For $a\notin S$, $S\cup\set{a}\notin CF$ because $S\in NAI$. By analogous reasoning to the above case, we have that $a\in S^+$ and hence $S\in STAB$. Therefore, $STAB=PREF$.
\end{proof}

\subsection{The Grounded Extension}\label{sec:G}

We have already encountered the $\subseteq$-least complete extension $G$ (Section \ref{sec:comp_exists_fp}, page \pageref{sec:comp_exists_fp}). We now study its properties further.

\subsubsection{Definition}

\begin{definition}\label{def:grounded}
\cite[Definition 20]{Dung:95} Given an AF, its \textbf{grounded extension}, $G$, is the $\subseteq$-least fixed point of $d$.
\end{definition}

\noindent This just names Definition \ref{def:pre_grounded} (page \pageref{def:pre_grounded}). We now give some examples:

\begin{example}
(Example \ref{eg:nixon}, page \pageref{eg:nixon} continued) In the Nixon diamond, $G=\es$. This is because $COMP=\set{\es,\set{a},\set{b}}$ and hence $\bigcap COMP=\es$.
\end{example}

\begin{example}
(Example \ref{eg:simple_reinstatement}, page \pageref{eg:simple_reinstatement} continued) In simple reinstatement, $G=\set{a,c}$. This is because $COMP=\set{\set{a,c}}$.
\end{example}

\begin{example}
\cite[Exercise 5(a)]{Caminada:08} For Figure \ref{fig:Caminada_ex2} (page \pageref{fig:Caminada_ex2}), $G=\set{a,c}$. This is because $COMP=\set{\set{a,c}}$.
\end{example}

\begin{example}
\cite[Exercise 5(b)]{Caminada:08} For Figure \ref{fig:Caminada_ex6b} (page \pageref{fig:Caminada_ex6b}), $G=\set{a}$. This is because $COMP=\set{\set{a},\set{a,c,f},\set{a,e}}$ hence $\bigcap COMP=\set{a}$.
\end{example}

\begin{example}
\cite[Exercise 5(c)]{Caminada:08} For Figure \ref{fig:Caminada_ex6c} (page \pageref{fig:Caminada_ex6c}), $G=\set{b,e}$. This is because $COMP=\set{\set{b,e}}$.
\end{example}

\begin{example}
\cite[Exercise 5(d)]{Caminada:08} For Figure \ref{fig:Caminada_ex3} (page \pageref{fig:Caminada_ex3}), $G=\es$, because $U=\es$. This is because $COMP=\set{\es,\set{a},\set{b,e}}$, hence $\bigcap COMP=\es$.
\end{example}

\subsubsection{Existence and Uniqueness}

\begin{theorem}
Given an AF, its grounded extension exists and is unique.
\end{theorem}
\begin{proof}
This is immediate from the Knaster-Tarski theorem (Theorem \ref{thm:KT}, page \pageref{thm:KT}), as $d$ is a $\subseteq$-monotonic function on the complete lattice $\ang{\pow\pair{A},\subseteq}$.
\end{proof}

Intuitively, the grounded extension captures skeptical reasoning. The agent believes in only those arguments that are either in $U$ or ultimately reinstated by, i.e. grounded in, $U$.

\begin{theorem}
Given an AF, its grounded extension is the $\subseteq$-smallest complete extension.
\end{theorem}
\begin{proof}
This follows from Theorem \ref{thm:grounded_def2} (page \pageref{thm:grounded_def2}).
\end{proof}

\begin{corollary}\label{cor:U_subset_G}
$U\subseteq G$.
\end{corollary}
\begin{proof}
By Theorem \ref{thm:grounded_def2} (page \pageref{thm:grounded_def2}), $G\in COMP$. By Corollary \ref{cor:U_subset_COMP} (page \pageref{cor:U_subset_COMP}), $U\subseteq G$.
\end{proof}

\noindent Recall also from Corollary \ref{cor:empty_grounded} that $U=\es$ iff $G=\es$.

\begin{example}
\cite[Example 22]{Dung:95} In Example \ref{eg:nixon} (page \pageref{eg:nixon}), as $U=\es$, we have $G=\es$.
\end{example}

\begin{corollary}
If $U\neq\es$ and $a\in A$ is indirectly defended by $U$ (Definition \ref{def:indirectly_defends}, page \pageref{def:indirectly_defends}), then $a\in G$.
\end{corollary}
\begin{proof}
If $U\neq\es$ and $a\in A$ is indirectly defended by $U$, then there exists $b\in U$ such that there is an even-length path from $b$ to $a$. Suppose the length is $2n$ for some $n\in\nat$. We prove by induction on $n$.
\begin{enumerate}
\item (Base) If $n=0$, then $b=a$, i.e. the path length is 0, then $a\in U\subseteq G$ by Corollary \ref{cor:U_subset_G}.
\item (Inductive) Suppose the path length is $2n$ and that $a\in G$. Let $a'$ be of length 2 away from $a$, therefore for all $b\in a'^-$, $b\in a^+$ so $b\in G^+$. Therefore, $a'\in d(G)=G$.
\end{enumerate}
The result follows by mathematical induction.
\end{proof}

The following result is important as it provides an efficient way of calculating the grounded extension for finite argumentation frameworks.

\begin{theorem}
If $d$ is $\omega$-continuous, then $\bigcup_{i\in\nat}d^i\pair{\es}$ is the grounded extension. 
\end{theorem}
\begin{proof}
We instantiate Kleene's fixed point theorem (Theorem \ref{thm:fp_thm}, page \pageref{thm:fp_thm}) where $\ang{D,\leq,\bot}=\ang{\pow\pair{A},\subseteq,\es}$ and $f=d$. Then $F:=\bigcup_{i\in\nat}d^i\pair{\es}$ is the least fixed point of $d$. By Definition \ref{def:grounded}, $F=G$.
\end{proof}

\begin{example}
(Example \ref{eg:Israel_Arab} continued, page \pageref{eg:Israel_Arab}, from \cite[Example 21]{Dung:95}) We have that $d\pair{\es}=\set{c}$, $d^2\pair{\es}=\set{a,c}=d^k\pair{\es}$ for $k>2$. Therefore, $G=\set{a,c}$.
\end{example}

\begin{corollary}
If $S\subseteq G$ then $S$ is not necessarily admissible.
\end{corollary}
\begin{proof}
From Example \ref{eg:simple_reinstatement} (page \pageref{eg:simple_reinstatement}), as $G=\set{c,a}$, let $S=\set{a}$, which is clearly not admissible because it does not defend itself against the attack from $b$.
\end{proof}

\subsection{Summary}

On preferred extensions:

\begin{itemize}
\item $S\subseteq A$ is preferred iff it is a $\subseteq$-maximal admissible extension. The set of all preferred extensions of a given AF $\mathcal{A}$ is $PREF\pair{\mathcal{A}}$ or just $PREF$.
\item For any AF, $PREF\neq\es$, and for any $S\in ADM$ there is an $E\in PREF$ such that $S\subseteq E$.
\item We have $PREF=\max_{\subseteq}ADM=\max_{\subseteq}COMP$, and that $COMP=\set{\es}$ iff $ADM=\set{\es}$.
\item If $\ang{A,R}$ is an even cycle, then $PREF$ has two sets. If $\ang{A,R}$ is an odd cycle, then $PREF=\set{\es}$.
\item If the underlying AF has no even cycle, then $PREF$ is unique.
\item If the underlying AF has no even cycle and $U=\es$, then $PREF=\set{\es}$.
\item $PREF$ is a $\subseteq$-antichain.
\item If $PREF$ covers $A$, then $\bigcap PREF= U$.
\item In general, $PREF\not\subseteq NAI$ and $NAI\not\subseteq PREF$, but if the AF is symmetric, then $PREF=NAI$.
\end{itemize}

\noindent On stable extensions:

\begin{itemize}
\item $S\subseteq A$ is stable iff it is a fixed point of $n$. The set of all stable extensions is $STAB(AF)$ or just $STAB$.
\item $STAB\subseteq PREF$, $STAB\subseteq NAI$, and if $PREF=\set{\es}$, then $STAB=\es$.
\item $A=\es$ iff $STAB=\set{\es}$.
\item If either $PREF\subseteq NAI$ or $NAI\subseteq PREF$, then $STAB=PREF$.
\end{itemize}

\noindent On the grounded extension

\begin{itemize}
\item The $\subseteq$-least complete extension is the grounded extension, $G$, which exists and is unique for all AF.
\item $U\subseteq G$, and $U=\es$ iff $G=\es$.
\item If $d$ is $\omega$-continuous, then the limit of the $\subseteq$-ascending chain $\set{d^k\pair{\es}}_{k\in\nat}$ is $G$.
\end{itemize}

\newpage

\section{Which Arguments are Justified?}

\subsection{Credulous and Sceptical Justification}

The Dung semantics are the four main semantics that we have discussed, and all are variations of admissible sets.

\begin{definition}
For an AF, its \textbf{Dung semantics} are: complete, preferred, stable and grounded.
\end{definition}

Given that there are multiple notions of what it means for a set of arguments to be justified, we now define what it means for an individual argument to be justified. Let AF be an argumentation framework.

\begin{definition}
We say $a\in A$ is \textbf{skeptically justified w.r.t. preferred / stable semantics} iff
\begin{align}
a\in\bigcap PREF\text{ and }a\in\bigcap STAB
\end{align}
respectively. If $STAB=\es$ then no argument can be skeptically justified w.r.t. stable semantics.
\end{definition}

\begin{example}
(Example \ref{eg:fri}, page \pageref{eg:fri} continued), we have $\bigcap PREF=\set{e}$ hence the argument $e$ is skeptically justified w.r.t. preferred semantics.
\end{example}


\begin{definition}
We say $a\in A$ is \textbf{credulously justified w.r.t. complete / preferred / stable semantics} iff
\begin{align}
a\in\bigcup COMP,\:a\in\bigcup PREF\text{ and }a\in\bigcup STAB
\end{align}
respectively. If $STAB=\es$ then no argument can be credulously justified w.r.t. stable semantics.
\end{definition}

\begin{example}
(Example \ref{eg:fri}, page \pageref{eg:fri} continued), we have $\bigcup PREF=\set{a,b,e}$ hence the argument $a$ is credulously justified w.r.t. preferred semantics.
\end{example}

\begin{definition}
We say $a\in A$ is \textbf{justified w.r.t. the grounded semantics} iff $a\in G$.
\end{definition}

\noindent Notice that as $G$ is unique, skeptical inference and credulous inference coincide. Further, as $G=\bigcap COMP$ (Corollary \ref{cor:cap_comp_is_G}, page \pageref{cor:cap_comp_is_G}), we do not consider skeptical justification w.r.t. complete semantics.

\begin{example}
(Example \ref{eg:nixon}, page \pageref{eg:nixon} continued) As $G=\es$, no argument is justified w.r.t. the grounded semantics.
\end{example}

\begin{definition}
\cite[Definition 2.13]{Baroni:09} An argument is \textbf{overruled w.r.t. a given semantics} iff it is not credulously justified w.r.t. to that semantics.
\end{definition}

To summarise, the Dung semantics consists of the complete, preferred, stable and grounded extensions. An argument is skeptically justified w.r.t. a given semantics iff it is in all of the extensions of that given type. Notice this means that skeptical complete is the same as grounded. An argument is credulously justified w.r.t. a given semantics iff it is in at least one of the extensions of that given type.


\subsection{Coincidence of Semantics}

\subsubsection{Equality of All Dung Semantics}

So far we have established $STAB\subset PREF\subset COMP\subset ADM\subset CF$, because the reverse inclusions do not hold in general.\footnote{This result has been shown in Theorem \ref{thm:STAB_in_PREF} (page \pageref{thm:STAB_in_PREF}), Theorem \ref{thm:pref_in_comp} (page \pageref{thm:pref_in_comp}), Theorem \ref{thm:comp_in_adm} (page \pageref{thm:comp_in_adm}) and Corollary \ref{cor:ADM_in_CF} (page \pageref{cor:ADM_in_CF}).} Further, we have established that $G\in COMP$. We can illustrate this in the following Hasse diagram:

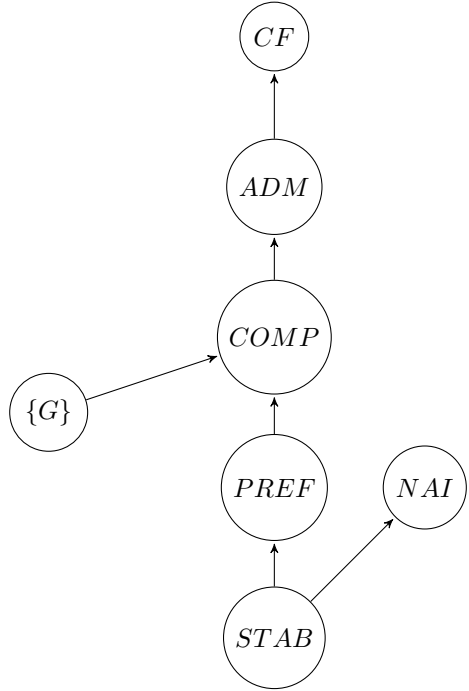
\begin{figure}[H]
\begin{center}
\begin{tikzpicture}[>=stealth',shorten >=1pt,node distance=2cm,on grid,initial/.style    ={}]
\tikzset{mystyle/.style={->,relative=false,in=0,out=0}};
\node[state] (STAB) at (0,0) {$ STAB $};
\node[state] (PREF) at (0,2) {$ PREF $};
\node[state] (COMP) at (0,4) {$ COMP $};
\node[state] (ADM) at (0,6) {$ ADM $};
\node[state] (CF) at (0,8) {$ CF $};
\node[state] (G) at (-3,3) {$ \set{G} $};
\node[state] (NAI) at (2,2) {$ NAI $};

\draw [->] (STAB) to (PREF);
\draw [->] (PREF) to (COMP);
\draw [->] (COMP) to (ADM);
\draw [->] (ADM) to (CF);
\draw [->] (G) to (COMP);
\draw [->] (STAB) to (NAI);
\end{tikzpicture}
\caption{A Hasse diagram where the arrows represent the containment relations between each type of sets of arguments.}\label{fig:semantics_hasse1}
\end{center}
\end{figure}

Under which circumstances for the AF can we have equality between the semantics? The strongest form of this equality is when all Dung semantics collapse and there is only one set of winning arguments. Figure \ref{fig:semantics_hasse1} suggests that this will happen if $G\in STAB$. We prove the following result.

\begin{lemma}\label{lem:grd_stab}
If $G\in STAB$, then $PREF=COMP=\set{G}=STAB$. Therefore, all semantics coincide. The converse is false.
\end{lemma}
\begin{proof}
Let $C\in COMP$ be arbitrary, then $G\subseteq C$. But as $G\in STAB\subseteq PREF=\max_{\subseteq}COMP$, then $C\subseteq G$ and hence $C=G$. As $G\in COMP$ and $C\in COMP$ is arbitrary, we have shown that any complete extension is equal to $G$ and hence $COMP=\set{G}$. From $PREF=\max_\subseteq COMP$, we conclude $PREF=\set{G}$. As $G\in STAB$, then $\es\neq STAB\subseteq PREF=\set{G}$ and hence $STAB=\set{G}$ as well. Therefore, all semantics coincide because $G$ is grounded, complete, preferred and stable.

The converse is false. Suppose we have a non-empty AF where $COMP=\set{\es}$ hence $G=\es$ and $STAB=\es$. Therefore, $G\notin STAB$.
\end{proof}

\begin{example}
(Example \ref{eg:Israel_Arab} continued, page \pageref{eg:Israel_Arab}) As $G=\set{a,c}$, we have $G\in STAB$ because the only argument outside is $b$, which is attacked by $c$. Therefore, $STAB=PREF=COMP=\set{\set{a,c}}$.
\end{example}

\begin{corollary}
If the AF is empty then all four semantics coincide.
\end{corollary}
\begin{proof}
If $A=\es$ then $G=\es$ because $U=\es$, and further, $n\pair{\es}=\es-\es^+=\es$, hence the grounded extension $\es$ is stable. By Lemma \ref{lem:grd_stab}, all four semantics coincide.
\end{proof}

\noindent One further sufficient condition to achieve equality of all semantics is the following.

\begin{theorem}\label{thm:wf_one_ext}
\cite[Theorem 30]{Dung:95} If the AF is non-empty and well-founded (Definition \ref{def:well_founded}, page \pageref{def:well_founded}), then all four semantics coincide. The converse is not true in general.
\end{theorem}
\begin{proof}
(Contrapositive) If all four semantics do not coincide, then the grounded extension $G$ is not stable by Lemma \ref{lem:grd_stab}. Therefore, $\pair{\exists a\notin G}a\notin G^+$. Define the set $S:=\set{x\in A-G\:\vline\:x\notin G^+}$. As $a\in S$, then $S\neq\es$. As $a\notin G=d(G)$, then $a\notin d(G)$, so $a^-\not\subseteq G^+$. There is some $b\in a^-$ such that $b\notin G^+$. Note as $G\in CF$, $b\notin G$ hence $b\in A-G$ and $b\notin G^+$. Therefore, we have found some $b\in S$ such that $R(b,a)$, so $a\in S^+$. As $a$ is arbitrary, $S\subseteq S^+$ and $S\neq\es$. By Corollary \ref{cor:non_wf_core} (page \pageref{cor:non_wf_core}), the underlying $\ang{A,R}$ is not well-founded.

If $G=\es$, which does not conflict with the fact that $G$ is not stable as $A\neq\es$, then $S=A\neq\es$ and by the same argument we show that $A\subseteq A^+$, hence $\ang{A,R}$ is also not well-founded.

For the converse, Example \ref{eg:Israel_Arab} (page \pageref{eg:Israel_Arab}) satisfies $STAB=PREF=COMP$ $=\set{G}$, but has a 2-cycle $(a,b),(b,a)\in R$, so this AF is not well-founded.
\end{proof}


\begin{corollary}
For finite acyclic AFs, there is only one extension of all four types.
\end{corollary}
\begin{proof}
This follows from Corollary \ref{cor:fin_acyc_wf} (page \pageref{cor:fin_acyc_wf}) and Theorem \ref{thm:wf_one_ext}.
\end{proof}

It is easy to see that if all Dung semantics coincide, then $ADM$ has $G$ to be its $\subseteq$-greatest element, and thus becomes a bounded poset $ADM\subseteq\pow\pair{G}$.




\subsubsection{Coherent Argumentation Frameworks}

A weaker case is to investigate when $PREF = STAB$.

\begin{definition}\label{def:coh_AF}
\cite[Definition 31(1)]{Dung:95} An AF is \textbf{coherent} iff $PREF = STAB$.
\end{definition}

\begin{corollary}\label{cor:wf_implies_coh}
If an AF is non-empty and well-founded, then it is coherent. The converse is not true.
\end{corollary}
\begin{proof}
If an AF is non-empty and well-founded, then by Theorem \ref{thm:wf_one_ext}, all four semantics coincide. In particular $PREF = STAB$, hence by Definition \ref{def:coh_AF}, this AF is also coherent.

For the converse, consider the AF with $A=\set{a,b,c,e}$ and $$R=\set{(a,b),(b,c),(c,a),(e,a)}.$$ This is depicted in Figure \ref{fig:wf_implies_coh}.

\begin{figure}[H]
\begin{center}
\begin{tikzpicture}[>=stealth',shorten >=1pt,node distance=2cm,on grid,initial/.style    ={}]
\tikzset{mystyle/.style={->,relative=false,in=0,out=0}};
\node[state] (a) at (0,0) {$ a $};
\node[state] (b) at (1,2) {$ b $};
\node[state] (c) at (2,0) {$ c $};
\node[state] (e) at (-2,0) {$ e $};
\draw [->] (a) to (b);
\draw [->] (b) to (c);
\draw [->] (c) to (a);
\draw [->] (e) to (a);
\end{tikzpicture}
\caption{The AF for the converse of Corollary \ref{cor:wf_implies_coh}.}\label{fig:wf_implies_coh}
\end{center}
\end{figure}
\noindent We have that $PREF=STAB=\set{\set{b,e}}$. Therefore, this AF is coherent. However, the 3-cycle $\set{a,b,c}$ means this AF is not well-founded (although it is non-empty).
\end{proof}

\begin{corollary}\label{cor:coh_AF_has_stab}
If an AF is coherent, then $STAB\neq\es$.
\end{corollary}
\begin{proof}
This follows from Definition \ref{def:coh_AF} and Corollary \ref{cor:pref_exist} (page \pageref{cor:pref_exist}).
\end{proof}

\begin{theorem}\label{thm:LC_coherent}
\cite[Theorem 33(1)]{Dung:95} Every limited controversial AF is coherent. The converse is not true.
\end{theorem}
\begin{proof}
Let $\ang{A,R}$ be a limited controversial AF. Assume for contradiction that it is not coherent, i.e. $PREF\neq STAB$. It follows that there is some $E\in PREF-STAB$. This means $E\cup E^+\subset A$ and hence $A-\pair{E\cup E^+}\neq\es$. Let $A':=A-\pair{E\cup E^+}$ and let $\ang{A',R'}\subseteq_g\ang{A,R}$. By Corollary \ref{cor:LC_subAF} (page \pageref{cor:LC_subAF}), $\ang{A',R'}$ is also limited controversial. By Theorem \ref{thm:LC_non_empty_comp} (page \pageref{thm:LC_non_empty_comp}), there is $C\subseteq A'$, $C\in COMP\pair{\ang{A',R'}}$ such that $C\neq\es$.

Now, as $\pair{\forall a\in A'}a\notin E^+$ and $C\subseteq A'$, we have that $E\cup C\in CF(\ang{A,R})$. Now let $b\in\pair{E\cup C}^-$. Either $b\in E^-$ or $b\in C^-$ by Corollary \ref{cor:cup_cap_plus_minus} (page \pageref{cor:cup_cap_plus_minus}). But as $E\in PREF\pair{\ang{A,R}}$ and $C\in COMP\pair{\ang{A',R'}}$, we must have $b\in E^+$ and $b\in C^+$. Therefore, $E\cup C\in ADM\pair{\ang{A,R}}$. However, as $C\neq\es$, we have constructed a strict superset of $E$ that is admissible, which contradicts the claim that $E$ is preferred.

For the converse, refer to the AF depicted in Figure \ref{fig:wf_implies_coh} in Corollary \ref{cor:wf_implies_coh}, which is a coherent AF that is not limited controversial due to its 3-cycle.
\end{proof}

\begin{corollary}
\cite[Corollary 36]{Dung:95} If the underlying AF is limited controversial then $STAB\neq\es$.
\end{corollary}
\begin{proof}
Immediate from  Corollary \ref{cor:coh_AF_has_stab} and Theorem \ref{thm:LC_coherent}.
\end{proof}

\begin{theorem}\label{thm:uncontroversial_coherent}
\cite[Definition 33(2)]{Dung:95} Every uncontroversial AF is coherent. The converse is not true.
\end{theorem}
\begin{proof}
Immediate from Corollary \ref{cor:uncont_LC} (page \pageref{cor:uncont_LC}) and Theorem \ref{thm:LC_coherent}. For the converse, the coherent AF depicted in Figure \ref{fig:wf_implies_coh} has argument $a$ which indirectly attacks and defends itself and hence this AF is controversial.
\end{proof}

\noindent We have a useful result that gives an equivalent characterisation of coherent AFs that do not have self-attacking arguments.

\begin{lemma}\label{lem:coh_implies_PREF_subset_NAI}
If an AF is coherent, then $PREF\subseteq NAI$.
\end{lemma}
\begin{proof}
By Corollary \ref{cor:stab_are_nai} (page \pageref{cor:stab_are_nai}), $PREF=STAB\subseteq NAI$. Therefore, $PREF\subseteq NAI$.
\end{proof}

\begin{lemma}\label{lem:pref_subset_nai_implies_coh_if_no_self_attack}
If an AF has no self-attacking arguments, and $PREF\subseteq NAI$, then the AF is coherent.
\end{lemma}
\begin{proof}
Let $S\in PREF$, then $S\in NAI$ and by Lemma \ref{lem:pref_naive_implies_stable}, $S\in STAB$. It follows that $PREF\subseteq STAB$, but as $STAB\subseteq PREF$, this AF is coherent.
\end{proof}

\begin{theorem}\label{thm:no_self_attack_coh_equiv}
For AFs that do not have self-attacking arguments, $PREF\subseteq NAI$ iff the AF is coherent.
\end{theorem}
\begin{proof}
($\Rightarrow$) This is Lemma \ref{lem:pref_subset_nai_implies_coh_if_no_self_attack}. ($\Leftarrow$) This is Lemma \ref{lem:coh_implies_PREF_subset_NAI}.
\end{proof}

\begin{example}\label{eg:not_coh_AF}
We illustrate an example of Theorem \ref{thm:no_self_attack_coh_equiv} where we have an AF that has no self-attacking arguments, and is not coherent as $PREF\not\subseteq NAI$. Consider the AF where $A=\set{a_0,a_1,a_2,a_3,a_4,a_5}$ with the attacks illustrated in Figure \ref{fig:not_coh_AF}, where we have abbreviated 2-cycles with two-headed arrows.

\begin{figure}[H]
\begin{center}
\begin{tikzpicture}[>=stealth',shorten >=1pt,node distance=2cm,on grid,initial/.style    ={}]
\tikzset{mystyle/.style={->,relative=false,in=0,out=0}};
\node[state] (0) at (0,0) {$ a_0 $};
\node[state] (1) at (-2,0) {$ a_1 $};
\node[state] (2) at (2,2) {$ a_2 $};
\node[state] (3) at (0,2) {$ a_3 $};
\node[state] (4) at (2,0) {$ a_4 $};
\node[state] (5) at (-2,2) {$ a_5 $};
\draw [<->] (0) to (4);
\draw [->] (0) to (5);
\draw [->] (1) to (0);
\draw [->] (1) to (3);
\draw [<->] (2) to (3);
\draw [<->] (2) to (4);
\draw [->] (3) to (0);
\draw [->] (3) to (5);
\draw [->] (4) to (3);
\draw [->] (5) to (1);
\end{tikzpicture}
\caption{The AF from Example \ref{eg:not_coh_AF}}\label{fig:not_coh_AF}
\end{center}
\end{figure}
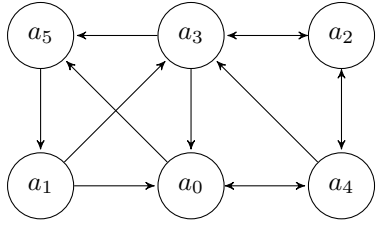

For the AF depicted in Figure \ref{fig:not_coh_AF}, we have that
\begin{align*}
NAI =& \set{\set{a_2,a_0},\set{a_2,a_1},\set{a_2,a_5},\set{a_3},\set{a_4,a_1},\set{a_4,a_5}}\\
PREF =& \set{\set{a_2},\set{a_4,a_5}}\\
STAB =& \set{\set{a_4,a_5}}.
\end{align*}
This AF is not coherent despite not having any self-attacking arguments, because $PREF\not\subseteq NAI$. Further, by the contrapositive of Theorem \ref{thm:LC_coherent}, this AF is not limited coherent, as shown by the 3-cycle $a_1$, $a_3$ and $a_5$. Similarly, by the contrapositive of Theorem \ref{thm:uncontroversial_coherent}, this AF is controversial, with (e.g.) the argument $a_3$ being controversial with respect to $a_5$.
\end{example}

\begin{example}\label{eg:no_self_attack_coh_equiv2}
We again illustrate Theorem \ref{thm:no_self_attack_coh_equiv} with a second example. Consider the AF with $A=\set{a_0,a_1,a_2,a_3,a_4}$ and attacks illustrated in Figure \ref{fig:not_coh_AF2}, where we have abbreviated 2-cycles with two-headed arrows.

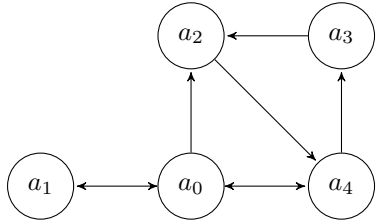
\begin{figure}[H]
\begin{center}
\begin{tikzpicture}[>=stealth',shorten >=1pt,node distance=2cm,on grid,initial/.style    ={}]
\tikzset{mystyle/.style={->,relative=false,in=0,out=0}};
\node[state] (0) at (0,0) {$ a_0 $};
\node[state] (1) at (-2,0) {$ a_1 $};
\node[state] (2) at (0,2) {$ a_2 $};
\node[state] (3) at (2,2) {$ a_3 $};
\node[state] (4) at (2,0) {$ a_4 $};
\draw [<->] (0) to (1);
\draw [->] (0) to (2);
\draw [<->] (0) to (4);
\draw [->] (2) to (4);
\draw [->] (3) to (2);
\draw [->] (4) to (3);
\end{tikzpicture}
\caption{The AF from Example \ref{eg:no_self_attack_coh_equiv2}}\label{fig:not_coh_AF2}
\end{center}
\end{figure}

For the AF depicted in Figure \ref{fig:not_coh_AF2}, we have
\begin{align*}
NAI =& \set{\set{a_1,a_2},\set{a_0,a_3},\set{a_1,a_3},\set{a_1,a_4}}\\
PREF =& \set{\set{a_1},\set{a_0,a_3}}\\
STAB =& \set{\set{a_0,a_3}}.
\end{align*}
Despite there not being any self-attacking arguments, we have that $PREF\not\subseteq NAI$ and hence $PREF\neq STAB$. Notice that the three-cycle $a_2$, $a_4$ and $a_3$ renders this AF not limited controversial.
\end{example}

Notice both Examples \ref{eg:not_coh_AF} and \ref{eg:no_self_attack_coh_equiv2} each have a three-cycle ($a_5, a_1, a_0$ in the former, and $a_2, a_4, a_3$ in the latter). This is actually a general result for finite AFs:

\begin{corollary}
\cite[Fact 19]{dunne2002coherence} If a finite AF is not coherent, then it has an odd cycle. The converse is not true.
\end{corollary}
\begin{proof}
Let an AF be finite and not coherent, then by the contrapositive of Theorem \ref{thm:LC_coherent}, this AF cannot be limited controversial. By the contrapositive of Corollary \ref{cor:finite_no_odd_cycles_LC} (page \pageref{cor:finite_no_odd_cycles_LC}), this AF must have an odd cycle.

The converse is not true, e.g. Figure \ref{fig:wf_implies_coh} depicts an AF with a three cycle and that it is coherent.
\end{proof}

\begin{theorem}
\cite[Proposition 5]{Coste:05} If an AF is symmetric with no self-attacking arguments, then it is coherent.
\end{theorem}
\begin{proof}
For a given AF with no self-attacking arguments, if it is symmetric, then $PREF=NAI$ by Corollary \ref{cor:sym_PREF_NAI_equal} (page \pageref{cor:sym_PREF_NAI_equal}). It follows that $PREF\subseteq NAI$ and, by Theorem \ref{thm:no_self_attack_coh_equiv}, this AF is coherent.
\end{proof}

\subsubsection{Relatively Grounded Argumentation Frameworks}

\begin{lemma}\label{lem:cap_pref_not_G}
\cite[Remark 26]{Dung:95} The intersection of all preferred extensions may not be the grounded extension.
\end{lemma}
\begin{proof}
From Example \ref{eg:fri} (page \pageref{eg:fri}), we have $PREF=\set{\set{a,e},\set{b,e}}$ and hence $\bigcap PREF=\set{e}$, but as $U=\es$, we have $G=\es$. Therefore, $\bigcap PREF\neq G$.
\end{proof}

Lemma \ref{lem:cap_pref_not_G} motivates us to classify AFs by whether they do satisfy the property that the intersection of all preferred extensions is the grounded extension.

\begin{definition}
\cite[Definition 31(2)]{Dung:95} An AF is \textbf{relatively grounded} iff $\bigcap PREF = G$.
\end{definition}

We give two examples of relatively grounded AFs, one where $G$ is empty and the other where $G$ is not empty.

\begin{example}\label{eg:RG_AF}
(Example \ref{eg:no_self_attack_coh_equiv2} continued) Notice for this AF, $G=\es$ because $U=\es$ (Corollary \ref{cor:empty_grounded}, page \pageref{cor:empty_grounded}). Further, $PREF=\set{\set{a_1},\set{a_0,a_3}}$ and hence $\bigcap PREF= G$. Therefore, this AF is relatively grounded. This is also an example of an AF that is not coherent and relatively grounded.
\end{example}

\begin{example}\label{eg:RG_AF2}
Consider the following AF where $A=\set{a_0,a_1,a_2,a_3}$ and the attacks are depicted in Figure \ref{fig:RG_nonempty_G}, where we have used two-headed arrows to represent 2-cycles.

\begin{figure}[H]
\begin{center}
\begin{tikzpicture}[>=stealth',shorten >=1pt,node distance=2cm,on grid,initial/.style    ={}]
\tikzset{mystyle/.style={->,relative=false,in=0,out=0}};
\node[state] (0) at (0,0) {$ a_0 $};
\node[state] (1) at (2,0) {$ a_1 $};
\node[state] (2) at (2,-2) {$ a_2 $};
\node[state] (3) at (0,-2) {$ a_3 $};
\draw [<->] (0) to (1);
\draw [<->] (0) to (3);
\draw [<->] (3) to (1);
\end{tikzpicture}
\caption{The AF from Example \ref{eg:RG_AF2}}\label{fig:RG_nonempty_G}
\end{center}
\end{figure}
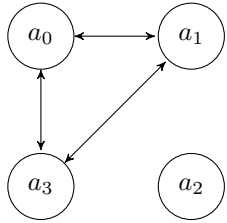

\noindent In this case, we have
\begin{align*}
G =& \set{a_2}\\
PREF =& \set{\set{a_2,a_0},\set{a_2,a_1},\set{a_2,a_3}}\\
STAB =& \set{\set{a_2,a_0},\set{a_2,a_1},\set{a_2,a_3}}
\end{align*}
\noindent Clearly this is coherent and relatively grounded, because $\bigcap PREF=G$.
\end{example}

\begin{corollary}\label{cor:wf_RG}
If an AF is well-founded and non-empty, then it is relatively grounded. The converse is not true.
\end{corollary}
\begin{proof}
If an AF is well-founded and non-empty, then by Theorem \ref{thm:wf_one_ext} its grounded extension $G$ is its unique preferred extension. Hence $PREF=\set{G}$ and $\bigcap PREF=G$. Therefore, this AF is relatively grounded.

For the converse, Example \ref{eg:no_self_attack_coh_equiv2} is a relatively grounded AF that has a 2-cycle ($a_0$ and $a_1$), hence it is not well-founded.
\end{proof}

\begin{theorem}\label{thm:uncontroversial_RG}
\cite[Theorem 33(2)]{Dung:95} Every uncontroversial AF is relatively grounded. The converse is not true.
\end{theorem}
\begin{proof}
Assume for contradiction that $G\neq\bigcap PREF$. By definition, $G\subset\bigcap PREF$. Let $a\in\bigcap PREF-G\neq\es$. If $a\in G^+$, then as $a\in\bigcap PREF$, for any $P\in PREF$, $a\in P$, and $G\subset P$, so $P\notin CF$ -- contradiction, because $PREF\subseteq CF$. Therefore, $a\notin G^+$. As the underlying AF is uncontroversial, by Theorem \ref{thm:lem35} (page \pageref{thm:lem35}), there exists complete extensions $C_1$ and $C_2$ such that $a\in C_1\cap C_2^+$. But $C_2\in ADM$ and hence by Corollary \ref{cor:ADM2PREF} (page \pageref{cor:ADM2PREF}), there is some $P\in PREF$ such that $C_2\subseteq P$, so $a\in P^+$, which contradicts that $a\in\bigcap PREF$ as $PREF\subseteq CF$. Therefore, $G=\bigcap PREF$.

For the converse, Example \ref{eg:no_self_attack_coh_equiv2} (page \pageref{eg:no_self_attack_coh_equiv2}) is an AF that is relatively grounded (see Example \ref{eg:RG_AF}) and by Figure \ref{fig:not_coh_AF2} has a controversial argument, i.e. $a_0$ is controversial with respect to $a_4$.
\end{proof}

\begin{theorem}
\cite[Lemma 1]{Croitoru:13} Given an AF $\ang{A,R}$, if $PREF$ covers $A$ then this AF is relatively grounded.
\end{theorem}
\begin{proof}
By Corollary \ref{cor:if_pref_cover_then_rg1} (page \pageref{cor:if_pref_cover_then_rg1}), if $PREF$ covers $A$, then $\bigcap PREF=U$. By Corollary \ref{cor:U_subset_G} (page \pageref{cor:U_subset_G}), $\bigcap PREF\subseteq G$. However, as $G$ is the $\subseteq$-least complete extension, and all preferred extensions are complete, we have $G\subseteq\bigcap PREF$. The result follows.
\end{proof}


\subsection{Summary}

\begin{itemize}
\item In any AF, if $G\in STAB$ then $STAB=PREF=COMP=\set{G}$.
\item If the AF is non-empty and well-founded, then $STAB=PREF=COMP=\set{G}$.
\item If the AF is limited controversial, then $PREF=STAB$.
\item For any AF, $PREF\subseteq NAI$ iff $PREF = STAB$.
\item If the AF is uncontroversial, then $\bigcap PREF=G$.
\end{itemize}

\newpage

\section{Conclusion}

In this note, we have reviewed \cite[Section 2]{Dung:95}, with the aim of making all of the proofs explicit. We do not claim originality as many of the results in this note are likely folklore. We hope that this note will be useful for students and researchers approaching abstract argumentation theory, in particular \cite{Dung:95}, for the first time.

\subsection{Summary of the Various Types of Sets of Arguments}

Given an AF $\ang{A,R}$ with neutrality function $n$ and defence function $d$:

\begin{table}[H]
\begin{center}
\begin{tabular}{ | c | c | c | }
\hline
\textbf{Type of Set of Args} & \textbf{Section} & \textbf{Definition} \\ \hline
$CF$ & \ref{sec:CF} (page \pageref{sec:CF}) & $S\in CF\Leftrightarrow S\subseteq n(S)$ \\ \hline
$NAI$ & \ref{sec:NAI} (page \pageref{sec:NAI}) & $NAI=\max_{\subseteq} CF$ \\ \hline
$SD$ & \ref{sec:SD} (page \pageref{sec:SD}) & $S\subseteq d(S)$ \\ \hline
$ADM$ & \ref{sec:ADM} (page \pageref{sec:ADM}) & $ADM=CF\cap SD$ \\ \hline
$COMP$ & \ref{sec:COMP} (page \pageref{sec:COMP}) & $S\in COMP\Leftrightarrow\sqbra{S\in CF, S=d(S)}$ \\ \hline
$PREF$ & \ref{sec:PREF} (page \pageref{sec:PREF}) & $PREF=\max_{\subseteq} ADM$ \\ \hline
$STAB$ & \ref{sec:STAB} (page \pageref{sec:STAB}) & $S\in STAB\Leftrightarrow S=n(S)$ \\ \hline
$G$ & \ref{sec:G} (page \pageref{sec:G}) & $G=\bigcap COMP$ \\ \hline
\end{tabular}
\caption{A table summarising the types of sets of arguments and their definition.}
\end{center}
\end{table}

\noindent All such sets of arguments, apart from the stable extensions, are non-empty for all AFs. Only the grounded extension is unique for all AFs. Lattice theoretically under $\subseteq$, $G$, $STAB$, $PREF$ and $NAI$ are antichains. $CF$, $SD$ and $ADM$ are bounded from below by $\es$. Further, $CF$, $ADM$ and $COMP$ are all directed complete and hence chain complete. $SD$ is closed under arbitrary unions and hence also directed complete, while $CF$ is closed under arbitrary intersections. Both $ADM$ and $COMP$ are complete semilattices.

\subsection{Acknowledgements}

The author would like to thank Phan Minh Dung, Christof Spanring and Hannes Strass for clarifying Theorem \ref{thm:COMP_glb_non_es} (page \pageref{thm:COMP_glb_non_es}) in private conversation, thereby pointing out the error in \cite{APYwrong1}. The author would also like to thank Jack Mumford for his constructive criticism and for suggesting the converse to Theorem \ref{thm:LC_coherent} (page \pageref{thm:LC_coherent}). The author would like to thank Christof Spanring again for clarifying Example \ref{eg:bi_inf} (page \pageref{eg:bi_inf}). The author would like to thank Amit Kuber for clarifying the two proofs of Corollary \ref{cor:ADM2PREF} (page \pageref{cor:ADM2PREF}) and the proof of Theorem \ref{thm:pref_only_empty} (page \pageref{thm:pref_only_empty}).

When the first version of this document was posted on ArXiV, the author was supported by the UK Engineering \& Physical Sciences Research Council (EPSRC) under grant \#EP/P010105/1.

\newpage

\appendix

\section{Directed Graphs}\label{app:digraphs}

In abstract argumentation, arguments and attacks between arguments are respectively represented as nodes and edges of a directed graph. We therefore recap some elementary notions of graph theory.

\begin{definition}
A \textbf{directed graph} (digraph) is a pair $\ang{A,R}$ where $A$ is a \textbf{set of nodes} and $R\subseteq A^2$ is a binary relation.
\end{definition}

\begin{example}\label{eg:path}
For $n\in\nat$, the digraph $P_n$ is called the directed path graph on $n$ nodes and has $A=\set{a_k}_{k=1}^n$ and $R=\set{\pair{a_k,a_{k+1}}}_{k=1}^{n-1}$. Notice $P_0$ is the empty graph $\ang{\es,\es}$. A typical non-empty graph $P_n$ is depicted in Figure \ref{fig:path_graph}.

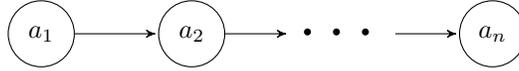
\begin{figure}[H]
\begin{center}
\begin{tikzpicture}[>=stealth',shorten >=1pt,node distance=2cm,on grid,initial/.style    ={}]
\tikzset{mystyle/.style={->,relative=false,in=0,out=0}};
\node[state] (1) at (0,0) {$ a_1 $};
\node[state] (2) at (2,0) {$ a_2 $};
\node (d) at (4,0) {\Huge $ \cdots $};
\node[state] (n) at (6,0) {$ a_n $};
\draw [->] (1) to (2);
\draw [->] (2) to (d);
\draw [->] (d) to (n);
\end{tikzpicture}
\caption{A depiction of $P_n$, from Example \ref{eg:path}.}\label{fig:path_graph}
\end{center}
\end{figure}
\end{example}

\begin{example}\label{eg:cycle}
For $n\in\nat$, the digraph $C_n$ is called the directed cycle graph on $n$ nodes and has $A=\set{a_k}_{k=1}^n$ and $R=\set{\pair{a_k,a_{k+1}}}_{k=1}^{n-1}\cup\set{\pair{a_n,a_1}}$. Clearly, $C_0=\ang{\es,\es}$. A typical non-empty graph $C_n$ is depicted in Figure \ref{fig:cycle_graph}.

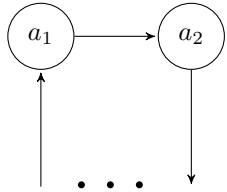
\begin{figure}[H]
\begin{center}
\begin{tikzpicture}[>=stealth',shorten >=1pt,node distance=2cm,on grid,initial/.style    ={}]
\tikzset{mystyle/.style={->,relative=false,in=0,out=0}};
\node[state] (1) at (0,0) {$ a_1 $};
\node[state] (2) at (2,0) {$ a_2 $};
\draw [->] (1) to (2);
\draw [->] (2) to (2,-2);
\node (d) at (1,-2) {\Huge $ \cdots $};
\draw [->] (0,-2) to (1);
\end{tikzpicture}
\caption{A depiction of $C_n$, from Example \ref{eg:cycle}.}\label{fig:cycle_graph}
\end{center}
\end{figure}
\end{example}

\noindent In what follows let $\ang{A,R}$ be an arbitrary digraph.

\begin{definition}\label{def:forward_back}
For $a\in A$, we define two sets.
\begin{align}
a^+:=&\set{b\in A\:\vline\: R(a,b)}\text{ and }\label{eq:singleton_plus}\\
a^-:=&\set{b\in A\:\vline\: R(b,a)}.\label{eq:singleton_minus}
\end{align}
We call $a^+$ the \textbf{forward set of $a$} and $a^-$ the \textbf{backward set of $a$}.
\end{definition}

\begin{corollary}\label{cor:plus_minus}
For all $a,b\in A$, $b\in a^-$ iff $a\in b^+$.
\end{corollary}
\begin{proof}
We have that $b\in a^-$ iff $R(b,a)$ iff $a\in b^+$ by Definition \ref{def:forward_back}.
\end{proof}

\begin{definition}\label{def:source_node}
We say $a\in A$ is a \textbf{source node} iff $a^-=\es$.
\end{definition}

\begin{definition}\label{def:S_pm}
For $S\subseteq A$, we define two sets:
\begin{align}
S^+:=&\set{a\in A\:\vline\:\pair{\exists b\in S}R(b,a)}\label{eq:S_plus}\\
S^-:=&\set{a\in A\:\vline\:\pair{\exists b\in S}R(a,b)}.\label{eq:S_minus}
\end{align}
We call $S^+$ the \textbf{forward set of $S$} and $S^-$ the \textbf{backward set of $S$}.
\end{definition}

\begin{corollary}\label{cor:pm_function}
If $S=T$ then $S^\pm=T^\pm$. The converses are not in general true.
\end{corollary}
\begin{proof}
We have that $a\in S^+$ iff $\pair{\exists b\in S}R(b,a)$ iff $\pair{\exists b\in T}R(b,a)$ because $S=T$, iff $a\in T^+$. Similarly, $a\in S^-$ iff $\pair{\exists b\in S}R(a,b)$ iff $\pair{\exists b\in T}R(a,b)$ because $S=T$, iff $a\in T^-$. The result follows.

Now consider the digraph $A=\set{a,b,c}$ and $R=\set{(a,c),(b,c)}$. This is depicted in Figure \ref{fig:pm_not_inj1}.

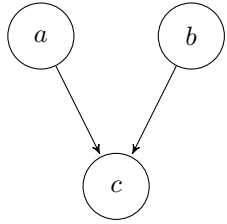
\begin{figure}[H]
\begin{center}
\begin{tikzpicture}[>=stealth',shorten >=1pt,node distance=2cm,on grid,initial/.style    ={}]
\tikzset{mystyle/.style={->,relative=false,in=0,out=0}};
\node[state] (a) at (0,0) {$ a $};
\node[state] (b) at (2,0) {$ b $};
\node[state] (c) at (1,-2) {$ c $};
\draw [->] (a) to (c);
\draw [->] (b) to (c);
\end{tikzpicture}
\caption{The first digraph mentioned in Corollary \ref{cor:pm_function}.}\label{fig:pm_not_inj1}
\end{center}
\end{figure}

\noindent Suppose $S=\set{a}$ and $T=\set{b}$. We have $S^+=T^+=\set{c}$ and $S\neq T$.

Now consider the dual digraph where $R^{\text{op}}=\set{(c,a),(c,b)}$, which is depicted in Figure \ref{fig:pm_not_inj2}.

\begin{figure}[H]
\begin{center}
\begin{tikzpicture}[>=stealth',shorten >=1pt,node distance=2cm,on grid,initial/.style    ={}]
\tikzset{mystyle/.style={->,relative=false,in=0,out=0}};
\node[state] (a) at (0,0) {$ a $};
\node[state] (b) at (2,0) {$ b $};
\node[state] (c) at (1,-2) {$ c $};
\draw [->] (c) to (a);
\draw [->] (c) to (b);
\end{tikzpicture}
\caption{The second digraph mentioned in Corollary \ref{cor:pm_function}.}\label{fig:pm_not_inj2}
\end{center}
\end{figure}
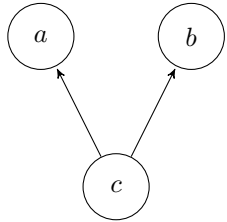

\noindent Then $S^-=T^-=\set{c}$ while $S\neq T$. Therefore, the functions $S\mapsto S^{\pm}$ are not injective.
\end{proof}

\begin{corollary}
The functions $\pair{\cdot}^{\pm}:\pow\pair{A}\to\pow\pair{A}$ where $S\mapsto S^{\pm}$ are well-defined.
\end{corollary}
\begin{proof}
Totality follows from Definition \ref{def:S_pm}. Single-valuedness follows from Corollary \ref{cor:pm_function}.
\end{proof}

\begin{corollary}\label{cor:set_pm}
We have:
\begin{align}
S^+=\bigcup_{a\in S}a^+\text{ and }S^-=\bigcup_{a\in S}a^-.
\end{align}
\end{corollary}
\begin{proof}
We have $a\in S^+$ iff $\pair{\exists b\in S}R(b,a)$ iff $\pair{\exists b\in S}a\in b^+$ by Equation \ref{eq:singleton_plus}, iff $a\in\bigcup_{b\in S}b^+$. Similarly, $a\in S^-$ iff $\pair{\exists b\in S}R(a,b)$ iff $\pair{\exists b\in S}a\in b^-$ by Equation \ref{eq:singleton_minus}, iff $a\in\bigcup_{b\in S}b^-$. The result follows.
\end{proof}

\begin{corollary}
We have that $\set{a}^\pm=a^\pm$.
\end{corollary}
\begin{proof}
Immediate from Corollary \ref{cor:set_pm} by setting $S=\set{a}$.
\end{proof}

\begin{corollary}\label{cor:es_pm}
We have that $\es^\pm=\es$.
\end{corollary}
\begin{proof}
Immediate from Corollary \ref{cor:set_pm} and the definition of the empty union.
\end{proof}

\begin{corollary}\label{cor:pm_monotonicity}
If $S\subseteq T\subseteq A$ then $S^\pm\subseteq T^\pm$. The converse is not necessarily true.
\end{corollary}
\begin{proof}
The result follows from Corollary \ref{cor:set_pm} and the properties of set-theoretic union. From Figure \ref{fig:pm_not_inj1}, it is the case that $S^+=T^+=\set{c}$ but $S=\set{a}\not\subseteq T=\set{b}$. Similarly, from Figure \ref{fig:pm_not_inj2}, we have $S^-\subseteq T^-$ and $S\not\subseteq T$. Therefore, the converse is false.
\end{proof}

\noindent Notice how Corollary \ref{cor:pm_function} follows trivially from Corollary \ref{cor:pm_monotonicity} using the definition of set equality.

\begin{definition}
Let $B\subseteq A$. The \textbf{induced directed graph (digraph) w.r.t. $B$} (a.k.a the \textbf{full subgraph w.r.t. $B$}) is the digraph $\ang{B,R_B}$ where $R_B=B^2\cap R$. We will write $\ang{B,R_B}\subseteq_{g}\ang{A,R}$.
\end{definition}

\noindent Clearly, $\subseteq_g$ is a reflexive and transitive relation on the class of digraphs.

\newpage

\section{Unions and Intersections of Bounded Quantifiers}

\begin{definition}\label{def:bdd_quant}
Let $X$ be a set and $A, P\subseteq X$. We have
\begin{align}
\pair{\exists x\in A}P(x)\Leftrightarrow&\pair{\exists x\in X}\pair{x\in A\wedge P(x)},\label{eq:bdd_E}\text{ and }\\
\pair{\forall x\in A}P(x)\Leftrightarrow&\pair{\forall x\in X}\pair{x\in A\rightarrow P(x)}.\label{eq:bdd_A}
\end{align}
where we will occasionally write ``$,$'' instead of ``$\wedge$'', and ``$\Rightarrow$'' instead of ``$\rightarrow$''.
\end{definition}

\noindent The following results are likely folklore; they are included as we make use of them.

\begin{theorem}\label{thm:bq_cup_cap}
Let $X$ be a set. Let $I$ be any index set such that $\set{A_i}_{i\in I}\subseteq\pow\pair{X}$. Let $P(x)$ be any unary predicate that may be true or false on the elements $x$ of $X$. We have the following results:
\begin{align}
\pair{\forall x\in\bigcup_{i\in I}A_i}P(x)\Leftrightarrow&\pair{\forall i\in I}\pair{\forall x\in A_i}P(x)\label{eq:all_cup},\\
\pair{\exists x\in\bigcup_{i\in I}A_i}P(x)\Leftrightarrow&\pair{\exists i\in I}\pair{\exists x\in A_i}P(x)\label{eq:some_cup},\\
\pair{\exists x\in\bigcap_{i\in I}A_i}P(x)\Rightarrow&\pair{\forall i\in I}\pair{\exists x\in A_i}P(x)\label{eq:some_cap}\text{ and }\\
\pair{\forall x\in\bigcap_{i\in I}A_i}P(x)\Leftarrow&\pair{\exists i\in I}\pair{\forall x\in A_i}P(x)\label{eq:all_cap}.
\end{align}
Further, the converses of Equations \ref{eq:some_cap} and \ref{eq:all_cap} are not true in general.
\end{theorem}
\begin{proof}
If $X=\es$, then $\set{A_i}_{i\in I}=\set{\es}$ for any index set $I$. Therefore, Equation \ref{eq:all_cup} reduces to true iff true, Equation \ref{eq:some_cup} reduces to false iff false, Equation \ref{eq:some_cap} reduces to false implies true if $I=\es$ and false implies false if $I\neq\es$, and Equation \ref{eq:all_cap} reduces to false implies true if $I=\es$ and true implies true if $I\neq\es$. In all cases, the four equations are true regardless of $I$.

If $X\neq\es$ and $I=\es$, then Equation \ref{eq:all_cup} reduces to true iff true, Equation \ref{eq:some_cup} reduces to false iff false, Equation \ref{eq:some_cap} reduces to true implies true, and Equation \ref{eq:all_cap} reduces to false implies either true or false, as $P(x)$ may not hold on all elements on $X$. In all cases, the four equations are true.

Now assume $X\neq\es$ and $I\neq\es$. For Equation \ref{eq:all_cup}:

\begin{align*}
\pair{\forall x\in\bigcup_{i\in I}A_i}P(x)\Leftrightarrow&\pair{\forall x\in X}\pair{x\in\bigcup_{i\in I}A_i\Rightarrow P(x)}\\
\Leftrightarrow&\pair{\forall x\in X}\sqbra{\pair{\exists i\in I}x\in A_i\Rightarrow P(x)}\\
\Leftrightarrow&\pair{\forall x\in X}\pair{\forall i\in I}\sqbra{x\in A_i\Rightarrow P(x)}\\
\Leftrightarrow&\pair{\forall i\in I}\pair{\forall x\in X}\sqbra{x\in A_i\Rightarrow P(x)}\\
\Leftrightarrow&\pair{\forall i\in I}\pair{\forall x\in A_i}P(x),
\end{align*}
which proves Equation \ref{eq:all_cup}.

For Equation \ref{eq:some_cup}:
\begin{align*}
\pair{\exists x\in\bigcup_{i\in I}A_i}P(x)\Leftrightarrow&\pair{\exists x\in X}\pair{x\in\bigcup_{i\in I}A_i\text{ and }P(x)}\\
\Leftrightarrow&\pair{\exists x\in X}\sqbra{\pair{\pair{\exists i\in I}x\in A_i}\text{ and }P(x)}\\
\Leftrightarrow&\pair{\exists x\in X}\pair{\exists i\in I}\pair{x\in A_i\text{ and }P(x)}\\
\Leftrightarrow&\pair{\exists i\in I}\pair{\exists x\in X}\pair{x\in A_i\text{ and }P(x)}\\
\Leftrightarrow&\pair{\exists i\in I}\pair{\exists x\in A_i}P(x),
\end{align*}
which proves Equation \ref{eq:some_cup}.

For Equation \ref{eq:some_cap}:
\begin{align}
\pair{\exists x\in\bigcap_{i\in I}A_i}P(x)\Leftrightarrow&\pair{\exists x\in X}\pair{x\in\bigcap_{i\in I}A_i\text{ and }P(x)}\nonumber\\
\Leftrightarrow&\pair{\exists x\in X}\sqbra{\pair{\pair{\forall i\in I}x\in A_i}\text{ and }P(x)}\nonumber\\
\Leftrightarrow&\pair{\exists x\in X}\pair{\forall i\in I}\pair{x\in A_i\text{ and }P(x)}\label{eq:step1}\\
\Rightarrow&\pair{\forall i\in I}\pair{\exists x\in X}\pair{x\in A_i\text{ and }P(x)}\label{eq:step2}\\
\Leftrightarrow&\pair{\forall i\in I}\pair{\exists x\in A_i}P(x),\nonumber
\end{align}
which proves Equation \ref{eq:some_cap}. The converse may not be true because Equation \ref{eq:step1} is strictly stronger than Equation \ref{eq:step2}. More concretely, let $X=\nat$ and $P(x)\Leftrightarrow x$ is even. Let $I=\set{1,2}$ and let $A_i=:A$ and $A_2=: B$. Let $A:=\set{2,3,4}$ and $B:=\set{5,6,7}$. Clearly, $\pair{\exists x\in A}P(x)$ is true with 2 as a witness. Further, $\pair{\exists x\in B}P(x)$ is true with 6 as a witness. Therefore, the right hand side of Equation \ref{eq:some_cap} is true. However, $A\cap B=\es$, which has no even numbers. Therefore, $\pair{\exists x\in A\cap B}P(x)$ is false.

Finally, for Equation \ref{eq:all_cap}:

\begin{align}
\pair{\exists i\in I}\pair{\forall x\in A_i}P(x)\Leftrightarrow&\pair{\exists i\in I}\pair{\forall x\in X}\pair{x\in A_i\Rightarrow P(x)}\label{eq:step3}\\
\Rightarrow&\pair{\forall x\in X}\pair{\exists i\in I}\pair{x\in A_i\Rightarrow P(x)}\label{eq:step4}\\
\Leftrightarrow&\pair{\forall x\in X}\sqbra{\pair{\pair{\forall i\in I}x\in A_i}\Rightarrow P(x)}\nonumber\\
\Leftrightarrow&\pair{\forall x\in X}\pair{x\in\bigcap_{i\in I}A_i\Rightarrow P(x)}\nonumber\\
\Leftrightarrow&\pair{\forall x\in\bigcap_{i\in I}A_i}P(x),\nonumber
\end{align}
which proves half of Equation \ref{eq:all_cap}. The converse may not be true because Equation \ref{eq:step3} is strictly stronger than Equation \ref{eq:step4}. More concretely, let $X=\nat$ and $P(x)\Leftrightarrow x$ is even. Let $I=\set{1,2}$ such that $A_1=:A$ and $A_2=:B$. Let $A:=\set{1,4}$ and $B:=\set{3,4}$. Clearly $A\cap B=\set{4}$ and hence $\pair{\forall x\in A\cap B}P(x)$ is true. However, neither $\pair{\forall x\in A}P(x)$ nor $\pair{\forall x\in A\cap B}P(x)$ is true. In the first case, this is because $1\in A$. In the second case, this is because $3\in B$.

This proves all four equations and shows that the converses of the latter two are not true in general.
\end{proof}

\noindent We now apply Theorem \ref{thm:bq_cup_cap} (page \pageref{thm:bq_cup_cap}) to digraphs.

\begin{corollary}\label{cor:cup_cap_plus_minus}
We have that
\begin{align}
&\pair{\bigcup_{i\in I}S_i}^{\pm}=\bigcup_{i\in I}S_i^{\pm}\label{eq:plus_cup}
\\
&\pair{\bigcap_{i\in I}S_i}^{\pm}\subseteq\bigcap_{i\in I}S_i^{\pm},\label{eq:plus_cap}
\end{align}
and the converse is not true in general for Equation \ref{eq:plus_cap}.
\end{corollary}

\noindent Notice that for $I=\es$, both results reduce to $\es=\es$ and $A^\pm\subseteq A$, which are trivially true.

\begin{proof}
For the $+$ case in Equation \ref{eq:plus_cup}, we apply Equations \ref{eq:S_plus} and \ref{eq:some_cup}.
\begin{align*}
a\in\pair{\bigcup_{i\in I}S_i}^+\Leftrightarrow&\pair{\exists b\in\bigcup_{i\in I}S_i}R(b,a)\\
\Leftrightarrow&\pair{\exists i\in I}\pair{\exists b\in S_i}R(b,a)\\
\Leftrightarrow&\pair{\exists i\in I} a\in S_i^+\\
\Leftrightarrow& a\in\bigcup_{i\in I}S_i^+.
\end{align*}
Notice that for the $-$ case, the proof is the same but with $+$ replaced by $-$ and $R(b,a)$ replaced by $R(a,b)$. Therefore, Equation \ref{eq:plus_cup} follows. 

For the $+$ case in Equation \ref{eq:plus_cap}, we apply Equations \ref{eq:S_plus} and \ref{eq:some_cap}.
\begin{align*}
a\in\pair{\bigcap_{i\in I}S_i}^+\Leftrightarrow&\pair{\exists b\in\bigcap_{i\in I}S_i}R(b,a)\\
\Rightarrow&\pair{\forall i\in I}\pair{\exists b\in S_i}R(b,a)\\
\Leftrightarrow&\pair{\forall i\in I} a\in S_i^+\\
\Leftrightarrow& a\in\bigcap_{i\in I}S_i^+.
\end{align*}
Notice that for the $-$ case, the proof is the same but with $+$ replaced by $-$ and $R(b,a)$ replaced by $R(a,b)$. Therefore, Equation \ref{eq:plus_cap} follows.

For the converse to the $+$ case in Equation \ref{eq:plus_cap}, consider the digraph $$\ang{\set{a,b,c,x},\set{\pair{a,x},\pair{c,x}}}.$$ This digraph is depicted in Figure \ref{fig:cap_plus}.

\begin{figure}[H]
\begin{center}
\begin{tikzpicture}[>=stealth',shorten >=1pt,node distance=2cm,on grid,initial/.style    ={}]
\tikzset{mystyle/.style={->,relative=false,in=0,out=0}};
\node[state] (a) at (0,0) {$ a $};
\node[state] (b) at (2,0) {$ b $};
\node[state] (c) at (4,0) {$ c $};
\node[state] (x) at (2,-2) {$ x $};
\draw [->] (a) to (x);
\draw [->] (c) to (x);
\end{tikzpicture}
\caption{The digraph mentioned in Corollary \ref{cor:cup_cap_plus_minus}.}\label{fig:cap_plus}
\end{center}
\end{figure}
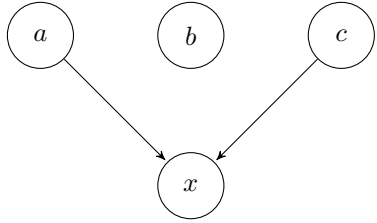

\noindent Let $S_1:=\set{a,b}$ and $S_2:=\set{b,c}$. We have $S_1^+=S_2^+=\set{x}$ and $S_1^+\cap S_2^+=\set{x}$. However, $S_1\cap S_2=\set{b}$ and $\pair{S_1\cap S_2}^+=\es$. It is not the case that $\set{x}\subseteq\es$.

For the converse in the $-$ case, take the dual of the digraph in Figure \ref{fig:cap_plus} with the same definitions of $S_1$ and $S_2$ to conclude that $\set{x}\not\subseteq\es$. Therefore, the converse to Equation \ref{eq:plus_cap} does not hold in general.
\end{proof}

\newpage

\section{A Recap of Order and Lattices}\label{app:order}

Many of these results can be found in textbooks such as \cite{Lattices} or \cite{subrahmanyam2018elementary}.

\subsection{Partially Ordered Sets}

In abstract argumentation, we will be concerned with partially ordered sets of the form $\ang{\mathcal{S},\subseteq}$ and lattices of the form $\ang{\mathcal{S},\cap,\cup}$, where $\mathcal{S}\subseteq\pow\pair{A}$ for a set $A$. We therefore recap the necessary elements of order and lattice theory.

\begin{definition}
A binary relation $R$ on a set $A$ is a \textbf{partial order} iff it is reflexive, antisymmetric and transitive.
\end{definition}

\begin{definition}
A digraph $\ang{A,R}$ is a \textbf{partially ordered set} (poset) iff $R$ is a partial order on $A$. We denote $R$ with $\leq$ and the underlying set will be denoted as $P$.
\end{definition}

\begin{example}
For any set $X$, the structure $\ang{\pow\pair{X},\subseteq}$ is a poset.
\end{example}

\begin{definition}
A \textbf{sub-poset} of $\ang{P,\leq}$ is a poset $\ang{Q,\leq'}$ where $Q\subseteq P$ and $\leq':=Q^2\cap\leq$.
\end{definition}

\begin{example}
For any set $X$ and family of subsets $\mathcal{S}\subseteq\pow\pair{X}$, the structure $\ang{\mathcal{S},\subseteq}$ is a poset.
\end{example}

\begin{definition}
Let $\ang{P,\leq}$ be a poset. We say $m\in U\subseteq P$ is a \textbf{maximal element} of $U$ iff for all $x\in U$, if $m\leq x$ then $m=x$. We let $\max_\leq U\subseteq U$ denote the set of all maximal elements of $U$.
\end{definition}

\begin{definition}
Let $\ang{P,\leq}$ be a poset. We say $m\in U\subseteq P$ is the \textbf{greatest element} of $U$ iff for all $x\in U$, $m\geq x$.
\end{definition}

\noindent Note that greatest elements are necessarily unique and maximal, but maximal elements do not have to be unique, and a maximal element do not have to be the greatest.

\begin{definition}
A poset $\ang{P,\leq}$ is an \textbf{antichain} iff $\leq$ is the diagonal relation on $P$.
\end{definition}

\noindent Clearly for any $\ang{P,\leq}$ and $U\subseteq P$, $\leq$ restricted onto $\max_\leq U$ is an antichain.


\begin{definition}
A poset $\ang{P,\leq}$ is a \textbf{totally ordered set} (toset) a.k.a. \textbf{chain} iff for all $x,y\in P$, either $x\leq y$ or $y\leq x$.
\end{definition}

\begin{definition}
Let $\ang{P,\leq}$ be a poset. A \textbf{chain} in $\ang{P,\leq}$ is a sub-poset that is also a toset.
\end{definition}



\begin{definition}
Let $\ang{P,\leq}$ be a poset. An \textbf{$\omega$-chain} is a $\geq$-ascending sequence\footnote{Recall for any set $X$, an $X$-sequence is a function $f:\nat\to X$. If $X$ has a partial order $\leq$, then $f$ is $\geq$-ascending iff $f(i)\leq f(i+1)$.} in $\ang{P,\leq}$.
\end{definition}

\begin{definition}
Let $\ang{P,\leq}$ be a poset and $Q\subseteq P$. We say $u\in P$ is an \textbf{upper bound for $Q$} iff $\pair{\forall x\in Q}x\leq u$. Similarly, $l\in P$ is a \textbf{lower bound for $Q$} iff $\pair{\forall x\in Q}l\leq x$.
\end{definition}

\begin{definition}
A poset $\ang{P,\leq}$ is \textbf{directed} iff every (not necessarily distinct) pair of elements in $P$ has an upper bound in $P$.
\end{definition}

\begin{definition}
We say $D$ is a \textbf{directed subset} of $\ang{P,\leq}$ iff $\ang{D,\leq\cap D^2}$ is a poset and is directed.\footnote{N.B. The upper bound of every pair of elements now has to be in $D$.}
\end{definition}

\begin{corollary}
Let $\ang{D,\leq}$ be an infinite directed poset and $x\in D$. Then there exists an $\omega$-chain in $\ang{D,\leq}$ starting from $x$.
\end{corollary}
\begin{proof}
We construct the chain as follows. Let $x_0:=x$ and $x_1:=x$. Trivially, $x_0\leq x_1$. Now let $x_2\in D$ be the upper bound of $x_0$ and $x_1$, which is well-defined because $D$ is directed and trivially $x_0\leq x_1\leq x_2$. Inductively, for $i\in\nat$, $x_{i+2}$ is the upper bound of the set $\set{x_{i+1},x_i}$, which always exists in $D$, and may not be equal to $x_i$ or $x_{i+1}$ because $D$ is infinite. Therefore, the $D$-sequence $\set{x_i}_{i\in\nat}$ is well-defined, has cardinality $\aleph_0$, and satisfies $\pair{\forall i\in\nat}x_i\leq x_{i+1}$. This is our $\omega$-chain in $\ang{D,\leq}$, starting from $x_0=x$.
\end{proof}

In what follows let $\ang{P,\leq}$ be an arbitrary poset.

\begin{definition}
For $Q\subseteq P$, let $Q^{\text{up}}:=\set{u\in P\:\vline\:\pair{\forall x\in Q}x\leq u}$ denote the set of upper bounds for $Q$. The \textbf{least upper bound of $Q$}, $\sup Q$, is $\min_\leq Q^{\text{up}}$.
\end{definition}

\begin{definition}
For $Q\subseteq P$, let $Q^{\text{low}}:=\set{l\in P\:\vline\:\pair{\forall x\in Q}l\leq x}$ denote the set of lower bounds for $Q$. The \textbf{greatest lower bound of $Q$}, $\inf Q$, is $\max_\leq Q^{\text{low}}$.
\end{definition}

\noindent Depending on $Q$, $\sup Q$ and $\inf Q$ may not exist, but if each does exist then each is unique.

\begin{definition}
Let $\ang{P,\leq}$ and $\ang{Q,\leq'}$ be two posets. A function $f:P\to Q$ is \textbf{monotone} iff $\pair{\forall x,y\in P}\sqbra{x\leq y\Rightarrow f(x)\leq' f(y)}$.
\end{definition}

\begin{corollary}
Let $\ang{P,\leq}$ and $\ang{Q,\leq'}$ be two posets and $f:P\to Q$ be a monotone function. If $\set{x_i}_{i\in\nat}$ is a $\leq$-chain in $P$, then $\set{f\pair{x_i}}_{i\in\nat}$ is a $\leq'$-chain in $Q$.
\end{corollary}
\begin{proof}
Immediate from the fact that $f$ is monotone.
\end{proof}

\begin{definition}\label{def:prefix_postfix_fp}
Let $\ang{P,\leq}$ be a posets and let $f:P\to P$ be monotone. We say $x\in P$ is a:
\begin{enumerate}
\item \textbf{prefixed point} iff $f(x)\leq x$.
\item \textbf{postfixed point} iff $x\leq f(x)$.
\item \textbf{fixed point} iff $f(x)=x$.
\end{enumerate}
\end{definition}

\subsection{Lattices}

\begin{definition}
A poset $\ang{P,\leq}$ is a \textbf{lattice} iff for every pair $x,y\in P$, $\sup\set{x,y}$ and $\inf\set{x,y}$ always exist. We formalise this as two binary operations:
\begin{align}
\wedge:P^2\to& P\nonumber\\
\pair{x,y}\mapsto& x\wedge y:=\inf\set{x,y}.\\
\vee:P^2\to& P\nonumber\\
\pair{x,y}\mapsto& x\vee y:=\sup\set{x,y}.
\end{align}
We call $\wedge$ \textbf{meet} and $\vee$ \textbf{join}.
\end{definition}

\noindent Clearly $x\wedge x=x$ and $x\vee x=x$ for any lattice $\ang{L,\wedge,\vee}$ and $x\in L$.

\begin{example}
For any set $X$, $\ang{\pow\pair{X},\cap,\cup}$ is a lattice where $\cap$ is meet and $\cup$ is join. In abstract argumentation we will be mainly concerned with lattices of the form $\ang{\mathcal{S},\cap,\cup}$ where $\mathcal{S}\subseteq\pow\pair{X}$.
\end{example}

\noindent Every lattice is a poset where $x\leq y\Leftrightarrow x\wedge y=x\Leftrightarrow x\vee y=y$. Not every poset is a lattice.

\begin{example}
The antichain of length 2, i.e. $P=\set{0,1}$ and $0=1$ is not a lattice because $\sup\set{0,1}$ and $\inf\set{0,1}$ do not exist. However, this is a well-defined poset.
\end{example}

\begin{definition}
As a poset, the least element (if it exists) of a lattice is called the \textbf{bottom element}, and the greatest element (if it exists) of a lattice is called the \textbf{top element}. A lattice with both top and bottom elements is a \textbf{bounded lattice}.
\end{definition}

\begin{corollary}\label{cor:bot_id_join}
Let $\bot$ denote the least element of a lattice $L$ if it exists, then for any $x\in L$, $x\vee\bot=x$.
\end{corollary}
\begin{proof}
Let $x$ be as given, then by definition $\bot\leq x$. We have that $x\vee\bot=\sup\set{\bot,x}=\max\set{\bot,x}=x$.
\end{proof}

\begin{definition}
A poset is a \textbf{complete lattice} iff \textit{all} subsets $Q\subseteq P$ has a least upper bound $\bigwedge Q$ and a greatest lower bound $\bigvee Q$ in the poset.
\end{definition}

\begin{corollary}\label{cor:CL_not_empty}
A complete lattice is never empty.
\end{corollary}
\begin{proof}
Let $\ang{L,\leq}$ be a complete lattice. $\es\subseteq L$ has a least upper bound and a greatest lower bound, both in $L$. Therefore, $L\neq\es$.
\end{proof}

\begin{definition}
For a complete lattice, we define its greatest element to be $\top:=\bigwedge\es$ and least element to be $\bot:=\bigvee\es$.
\end{definition}

\begin{corollary}\label{cor:interval_CL}
Let $\ang{L,\leq}$ be a complete lattice and let $a,b\in L$. If $a\leq b$, then $\ang{\sqbra{a,b},\leq}$ is also a complete lattice, where $\sqbra{a,b}:=\set{x\in L\:\vline\:a\leq x\leq b}$.
\end{corollary}
\begin{proof}
Let $a\leq b$ so $\sqbra{a,b}\neq\es$ (else it cannot be a complete lattice by Corollary \ref{cor:CL_not_empty}). Let $S\subseteq\sqbra{a,b}$ be arbitrary. If $S=\es$ then $\bigvee S=a$ and $\bigwedge S=b$. If $S\neq\es$, then for all $s\in S$, $a\leq s\leq b$ so $a$ is a lower bound of $S$ and $b$ is an upper bound of $S$. As $L$ is a complete lattice and $S\subseteq\sqbra{a,b}\subseteq L$, we have $p:=\bigwedge S,\:q:=\bigvee S\in L$. By definition, $a\leq p$ and $q\leq b$. Now let $x\in S$ be arbitrary, then $a\leq p\leq x\leq q\leq b$. This means $p,q\in\sqbra{a,b}$ and hence $\sqbra{a,b}$ contains the supremum and infimum of $S$. As $S$ is arbitrary, $\sqbra{a,b}$ is also a complete lattice.
\end{proof}

\begin{example}
For any set $X$, $\ang{\pow\pair{X},\subseteq}$ is a complete lattice. For every family of subsets of $X$, their collective union is the greatest lower bound and their collective intersection is the least upper bound.
\end{example}

\begin{lemma}\label{lem:KT}
(\textbf{Knaster-Tarski lemma}) Let $\ang{L,\leq}$ be a complete lattice and $f:L\to L$ be a monotone function. The set $F:=\set{x\in L\:\vline\:f(x)=x}$ is a bounded lattice with respect to $\leq$.
\end{lemma}
\begin{proof}
This is equivalent to showing that $F$ has a greatest element (the greatest fixed point of $f$) and a least element (the least fixed point of $f$). Let $D$ be the set of all postfixed points of $L$, i.e. $D=\set{x\in L\:\vline\:x\leq f(x)}$ (Definition \ref{def:prefix_postfix_fp}). As $L$ is a complete lattice, $0_L:=\bigwedge L\in L$ and for all $x\in L$, $f(x)\in L$ so $\pair{\forall x\in L}0_L\leq f(x)$. Thus, $0_L\in D$ and hence $D\neq\es$. Therefore, there is some $x\in D$, iff $x\leq f(x)$, implies that $f(x)\leq f^2(x)$ because $f$ is monotone, iff $f(x)\in D$.

As $D\subseteq L$ and $L$ is a complete lattice, let $u:=\bigvee D\in L$. By definition, $\pair{\forall x\in D}x\leq u$ hence $f(x)\leq f(u)$. But as $x\in D$, $x\leq f(x)\leq f(u)$ so $f(u)$ is an upper bound of $D$. As $u$ is the supremum of $D$, have $u\leq f(u)$. Therefore, $u\in D$. Further, $u\leq f(u)$ implies $f(u)\leq f^2(u)$ iff $f(u)\in D$. But as $f(u)\in D$ and $u$ is the supremum of $D$, $f(u)\leq u$. Therefore, $f(u)=u$. As $P\subseteq D$, $D$ contains all fixed points, and $u\in P$ is the greatest fixed point of $f$.

Dually, as $f$ also has a least fixed point of $f$ by arguing as above on the dual lattice of $L$. Therefore, $P$ is a bounded lattice.
\end{proof}

\begin{theorem}\label{thm:KT}
(\textbf{Knaster-Tarski theorem}) Let $\ang{L,\leq}$ be a complete lattice and $f:L\to L$ be a monotone function. The set $F:=\set{x\in L\:\vline\:f(x)=x}$ is a complete lattice.
\end{theorem}
\begin{proof}
Let $S\subseteq F$ be arbitrary. Let $s:=\bigvee S \in L$ as $L$ is a complete lattice. We show that there is an element of $F$ that is greater than all elements in $S$, and this is the smallest such element of $F$. It is sufficient to show the stronger result that this element of $F$ is greater than $s$.

Consider the interval $\sqbra{s,1_L}\subseteq L$ where $1_L:=\bigvee L\in L$. We show that $f:\sqbra{s,1_L}\to\sqbra{s,1_L}$. Clearly, for all $a\in S$, $a\leq s$. As $S\subseteq F$, $f(a)=a$. Therefore, $a=f(a)\leq f(s)$ as $f$ is monotone, which means $a\leq f(s)$ and as $a\in S$ is arbitrary, $f(s)$ is an upper bound for $S$. As $s$ is the supremum for $S$, we must have $s\leq f(s)$. Let $x\in\sqbra{s,1_L}$, then $s\leq x$, and hence $f(s)\leq f(x)$ and hence $s\leq f(x)$, so $f(x)\in \sqbra{s,1_L}$. Therefore, $f:\sqbra{s,1_L}\to\sqbra{s,1_L}$ is well-defined.

As $\sqbra{s,1_L}\subseteq L$ is a complete lattice by Corollary \ref{cor:interval_CL}, and $f$ is a monotonic function, then by Lemma \ref{lem:KT}, $f$ has a least fixed point which is the supremum of $S$. By definition, this is in $F$.

Dually, the infimum of $S$ is also in $F$ as the greatest fixed point of $f$ as a function on the complete lattice $\sqbra{0_L,\bigwedge S}$. Therefore, $S\subseteq F$ has a supremum and infimum both in $F$. As $S\subseteq F$ is arbitrary, $F$ is a complete lattice.
\end{proof}

\subsection{Complete Partial Orders}

\begin{definition}
The \textbf{limit} of a chain $C$ in a poset that is also a lattice is $\sup C$.
\end{definition}

\begin{corollary}
Finite chains always have a limit in the chain.
\end{corollary}
\begin{proof}
Let $\ang{P,\leq}$ be a poset and $C\subseteq P$ be a chain. WLOG let $C=\set{c_1,\ldots,c_n}$ as it is finite. Clearly $\bigvee C=\max_{\leq} C=:c_{\max}$, which exists and is in $C$. Therefore, $C$ contains its own limit.
\end{proof}

Clearly not every chain has a well-defined limit.

\begin{example}
The poset $\ang{\nat,\leq}$ contains itself as a chain, and $\sup\nat$ is not contained in $\nat$. Therefore, $\nat$ itself does not have a limit.
\end{example}

\begin{definition}
A poset is \textbf{chain-complete} iff the limit of every chain is in the poset.
\end{definition}

\begin{definition}
A poset is \textbf{$\omega$-complete} iff every $\omega$-chain has a least upper bound in the poset.
\end{definition}

\noindent Clearly every chain-complete poset is $\omega$-complete.

\begin{example}
$\ang{\pow\pair{X},\subseteq}$ is chain complete, where for every chain $\set{S_i}_{i\in I}$, its limit which is the union of all such elements of the chain is clearly in $\pow\pair{X}$.
\end{example}

\begin{definition}\label{def:complete_SL}
\cite[Footnote 5]{Dung:95} A poset $\ang{P,\leq}$ is a \textbf{complete semilattice} iff every non-empty subset of $P$ has an infimum, and $P$ is chain complete.
\end{definition}

\noindent There is no universally accepted and consistent definition of a ``complete semilattice''.\footnote{See \url{https://en.wikipedia.org/wiki/Semilattice}, last accessed 5/10/2017.} Definition \ref{def:complete_SL} comes from \cite[Page 330, Footnote 5]{Dung:95}. Further, the infimum is with respect to $\subseteq$ and does not have to be set-theoretic intersection.

\begin{corollary}
Let $\ang{P,\leq}$ and $\ang{Q,\leq'}$ be $\omega$-complete posets and $f:P\to Q$ be a monotone function. Let $\set{x_i}_{i\in\nat}$ be a $P$-chain with limit $x:=\bigvee_{i\in\nat}x_i\in P$.\footnote{This abuses notation as we use $\bigvee$ to refer to the iterated meet of elements in both $P$ and $Q$.} Then
\begin{align}\label{eq:limit_higher}
\bigvee_{i\in\nat}f\pair{x_i}\leq' f\pair{\bigvee_{i\in\nat}x_i}=f(x).
\end{align}
\end{corollary}
\begin{proof}
By definition, for all $i\in\nat$, $x_i\leq x$. By monotonicity, $f\pair{x_i}\leq' f(x)$. Therefore, $f(x)$ is an upper bound for $\set{f\pair{x_i}}_{i\in\nat}$, with limit $\bigvee_{i\in\nat}f\pair{x_i}$. As the limit is the greatest lower bound, $\bigvee_{i\in\nat}f\pair{x_i}\leq' f(x)$, and the result follows.
\end{proof}

\begin{definition}\label{def:omega_cts}
Let $\ang{P,\leq}$ and $\ang{Q,\leq'}$ be $\omega$-complete posets and $f:P\to Q$ be a monotone function. We say $f$ is \textbf{$\omega$-continuous} iff for every chain $\set{x_i}_{i\in\nat}$ in $P$ with limit $x:=\bigvee_{i\in\nat}x_i\in P$,
\begin{align}\label{eq:limit_lower}
f\pair{\bigvee_{i\in\nat}x_i}=f(x)\leq'\bigvee_{i\in\nat}f\pair{x_i}.
\end{align}
\end{definition}

\begin{corollary}
$\omega$-continuous functions preserve limits of chains, i.e.
\begin{align}
\bigvee_{i\in\nat} f\pair{x_i}=f\pair{\bigvee_{i\in\nat} x_i}.
\end{align}
\end{corollary}
\begin{proof}
Immediate from Equations \ref{eq:limit_higher} and \ref{eq:limit_lower}, and that $\leq'$ is antisymmetric.
\end{proof}

\noindent Notice that this definition is analogous to the continuity of real functions, where $\bigwedge$ is replaced with $\lim$.

\begin{example}
Not all monotone functions between posets are $\omega$-continuous. Let $\ang{\nat\cup\set{\infty},\leq}$ be the poset of extended natural numbers with the usual order relation on $\infty$. Let $\ang{\set{0',\infty'},\leq'}$ be another poset such that $0'<'\infty'$. Let $f$ be the function
\begin{align}
f:\nat\cup\set{\infty}&\to\set{0',\infty'}\nonumber\\
n&\mapsto 0'\nonumber\\
\infty&\mapsto\infty'.
\end{align}
Clearly, $f$ is monotonic, because if $n\leq m$ then $f(n)=0'\leq' f(m)=0'$ and $n\leq\infty$ means $0'\leq'\infty'$. However, $f$ is not $\omega$-continuous. Consider the chain $\set{n}_{n\in\nat}$ with limit $\infty$. This chain is mapped to the constant sequence $\set{0'}_{n\in\nat}$ with limit $0'$. Therefore,
\begin{align}
f\pair{\bigvee_{n\in\nat}\set{n}}=f\pair{\infty}=\infty'\leq'\bigvee_{n\in\nat}f(n)=0'
\end{align}
is false.
\end{example}

\begin{definition}\label{def:dcpo}
A poset is \textbf{directed-complete} (dcpo) iff every directed subset has its least upper bound in the poset.
\end{definition}

\begin{definition}
A poset is \textbf{pointed directed-complete} (cppo) iff it is a directed-complete poset with a least element.
\end{definition}

We now recap a fixed point theorem in lattice theory: iterating an $\omega$-continuous function $f$ from a cppo to itself, starting from the bottom element, will eventually yield the least fixed point of $f$.

\begin{theorem}\label{thm:fp_thm}
(\textbf{Kleene's fixed point theorem}) Let $\ang{D,\leq,\bot}$ be a cppo and the function $f:\ang{D,\leq,\bot}\to\ang{D,\leq,\bot}$ be $\omega$-continuous. Let $F=\bigvee_{n\in\nat}f^n\pair{\bot}$ be the supremum of the ($\leq$-increasing) chain $\set{f^n\pair{\bot}}_{n\in\nat}$. Then $F$ is the least fixed point of $f$.
\end{theorem}
\begin{proof}
First we show that $F$ is a fixed point of $f$. Recall that $f$ is $\omega$-continuous.
\begin{align}
f\pair{F}=&f\pair{\bigvee_{n\in\nat}f^n\pair{\bot}}=\bigvee_{n\in\nat}f^{n+1}\pair{\bot}\nonumber\\
=&\bigvee_{n\in\nat^+}f^n\pair{\bot}=\bot\vee\bigvee_{n\in\nat^+}f^n\pair{\bot}=F\:,
\end{align}
as $\bot$ is the identity of $\vee$. Therefore, $F$ is a fixed point of $f$.

Then we show that $F$ is the $\leq$-least fixed point. Let $F'$ be any fixed point of $f$, so $f\pair{F'}=F'$. We show by induction that for all $n\in\nat$, $f^n\pair{\bot}\leq F'$.
\begin{enumerate}
\item (Base) For $n=0$, $f^0\pair{\bot}=\bot\leq F'$ by definition of it being a bottom element.
\item (Inductive) Assume $f^k\pair{\bot}\leq F'$. Apply $f$ to both sides of the inequality to get $f^{k+1}\pair{\bot}\leq f\pair{F'}=F'$ as $f$ is monotone and $F'$ is a fixed point. Therefore, $f^{k+1}\pair{\bot}\leq F'$ as well.
\end{enumerate}
By induction, this shows that $F'$ is an upper bound of the chain $\set{f^n\pair{\bot}}_{n\in\nat}$. As $F$ is the supremum of this chain, by definition $F\leq F'$. Therefore, $F$ is the $\leq$-least fixed point of $f$.
\end{proof}

\newpage

\section{The Axiom of Choice in Abstract Argumentation Theory}\label{app:AC}

In this appendix, we show that the axiom of choice is \textit{necessary} to assert the existence of naive and preferred extensions for \textit{all} AFs \cite{Spanring:14}. 

\begin{definition}
Let $\mathcal{F}$ be \textit{any} family of non-empty sets, i.e. $\pair{\forall X\in\mathcal{F}}X\neq\es$.\footnote{Note ``any'' means $\mathcal{F}$ can have a finite or any infinite number of these non-empty sets.} A \textbf{choice function} is a function
\begin{align}
f:\mathcal{F}\to\bigcup\mathcal{F}
\end{align}
that satisfies
\begin{align}
\pair{\forall X\in\mathcal{F}}f(X)\in X.
\end{align}
\end{definition}

\begin{definition}
The \textbf{axiom of choice (AC)} states the following: any family $\mathcal{F}$ of non-empty sets has a choice function.
\end{definition}

\noindent Intuitively, in a possibly infinite collection of non-empty sets, the axiom of choice permits one to choose exactly one element from each such set. We take as given that AC is equivalent to Zorn's lemma (ZL) \cite{kuratowski1922methode,zorn1935remark}:

\begin{lemma}\label{lem:ZL}
\textbf{Zorn's lemma (ZL):} Let $\ang{P,\leq}$ be a poset. If every chain in $P$ has an upper bound in $P$, then $P$ has an $\leq$-maximal element.
\end{lemma}

\noindent Recall that AC (as ZL) is sufficient to show that naive extensions exist for all AFs (Theorem \ref{thm:existence_of_naive}, page \pageref{thm:existence_of_naive}), and that preferred extensions exist for all AFs (Corollary \ref{cor:pref_exist}, page \pageref{cor:pref_exist}). We now show that AC is also \textit{necessary} for the existence of naive and preferred extensions. This was not stated explicitly in, e.g. \cite[Corollary 12]{Dung:95}

\begin{theorem}\label{thm:NAI_exist_implies_AC}
\cite[Theorem 7]{Spanring:14} If all AFs $\ang{A,R}$ have a naive extension, then AC is true.
\end{theorem}
\begin{proof}
Let $\mathcal{F}$ be any family of non-empty sets. We need to construct a choice function $f$. Recall that any function is a set of ordered pairs. The idea is to construct an AF from $\mathcal{F}$ such that its naive extensions correspond to choice functions on $\mathcal{F}$, hence verifying AC.

Let $A:=\set{(X,x)\:\vline\:X\in\mathcal{F},\:x\in X}$ be the arguments under consideration. Each argument picks a set $X\in\mathcal{F}$ and also an element $x\in X$. Let $R\subseteq A^2$ be defined as follows:
\begin{align}
R\pair{\pair{X_1,x_1},\pair{X_2,x_2}}\Leftrightarrow \sqbra{X_1=X_2\text{ and }x_1\neq x_2}.
\end{align}
Notice $R$ is a symmetric relation. The idea is that two such arguments attack each other if they disagree on which element of $X_1$ to choose. $\ang{A,R}$ is a well-defined AF.

Let $N\subseteq A$ be a naive extension of $\ang{A,R}$, which exists our hypothesis. We show this corresponds to a choice function on $\mathcal{F}$. Given this $N$, suppose there is some $X\in\mathcal{F}$ such that $\pair{\forall x\in X}\pair{X,x}\notin N$. Then $N\cup\set{\pair{X,x}}\supset N$ will be a conflict-free set - contradicting that $N$ is $\subseteq$-maximal conflict-free. Therefore, for every $X\in\mathcal{F}$, there is some $x\in X$ such that $(X,x)\in N$. Now suppose there is some $X\in\mathcal{F}$ such that for $x,y\in X$ distinct, $(X,x),\:(X,y)\in N$. But as $x\neq y$, the arguments $(X,x)$ and $(X,y)$ must be attacking, contradicting the conflict-freeness of $N$. It therefore follows that for each $X\in\mathcal{F}$, there is exactly one $x\in X$ such that $(X,x)\in N$. We can use this $N$, which exists, to define a choice function
\begin{align}
f_N:\mathcal{F}&\to\bigcup\mathcal{F}\nonumber\\
X&\mapsto f_N(X)=x
\end{align}
where $f_N(X)=x$ iff $(X,x)\in N$. Clearly, $\pair{\forall X\in\mathcal{F}}f_N(X)\in X$, and $f_N$ is total on $\mathcal{F}$ and single-valued. As $\mathcal{F}$ is \textit{any} family of non-empty sets, we have shown that this has a choice function. This verifies AC.
\end{proof}

\begin{theorem}\label{thm:PREF_exist_implies_AC}
\cite[Theorem 11]{Spanring:14} If all AFs $\ang{A,R}$ have a preferred extension, then AC is true.
\end{theorem}
\begin{proof}
Let $\mathcal{F}$ be any family of non-empty sets. We use the same construction as in Theorem \ref{thm:NAI_exist_implies_AC}. Notice that the resulting AF is symmetric. By our hypothesis, this AF will have some preferred extension $P$, which is also naive by Corollary \ref{cor:sym_PREF_NAI_equal} (page \pageref{cor:sym_PREF_NAI_equal}). Therefore, the set of nodes in $P$ defines a choice function on $\mathcal{F}$ in the same manner as in the proof of Theorem \ref{thm:NAI_exist_implies_AC}, thereby verifying AC.
\end{proof}

We can summarise the preceding two theorems as follows:

\begin{corollary}
TFAE:
\begin{enumerate}
\item The axiom of choice
\item Every AF has a naive extension.
\item Every AF has a preferred extension.
\end{enumerate}
\end{corollary}
\begin{proof}
Theorem \ref{thm:NAI_exist_implies_AC} shows the equivalence of statements 1 and 2. Theorem \ref{thm:PREF_exist_implies_AC} shows the equivalence of statements 1 and 3. The result follows.
\end{proof}

Note that (as usual) AC is not needed if we only consider finite AFs, e.g. in the context of modelling and implementing real argumentation.

\begin{theorem}
All finite AFs $\ang{A,R}$ satisfy $NAI\neq\es$ and $PREF\neq\es$.
\end{theorem}
\begin{proof}
If $\ang{A,R}$ is finite, then $\ang{\pow\pair{A},\subseteq}$ is a finite poset. Hence $CF,\:ADM\subseteq\pow\pair{A}$ are also finite posets by inheriting $\subseteq$ on $\pow\pair{A}$. All finite posets have at least one maximal element (e.g. \cite[page 16]{Lattices}). Therefore, $\max_{\subseteq}CF=:NAI\neq\es$ and $\max_{\subseteq}ADM=:PREF\neq\es$.
\end{proof}


\newpage

\bibliographystyle{abbrv}
\bibliography{AT_notes}

\end{document}